\documentclass{article}

\usepackage{microtype}
\usepackage{graphicx}
\usepackage{subfigure}
\usepackage{booktabs} % for professional tables

\usepackage{hyperref}

\usepackage[accepted]{icml2024}

\usepackage{amsmath}
\usepackage{amssymb}
\usepackage{mathtools}
\usepackage{amsthm}

\usepackage[capitalize,noabbrev]{cleveref}

\renewcommand{\algorithmiccomment}[1]{\hfill{\textcolor{gray}{$\triangleright$ #1}}} %new
\theoremstyle{plain}
\newtheorem{theorem}{Theorem}[section]
\newtheorem{proposition}[theorem]{Proposition}
\newtheorem{lemma}[theorem]{Lemma}
\newtheorem{corollary}[theorem]{Corollary}
\theoremstyle{definition}
\newtheorem{definition}[theorem]{Definition}

\newtheorem{conjecture}[theorem]{Conjecture} %new
\theoremstyle{remark}
\newtheorem{remark}[theorem]{Remark}

\usepackage{natbib} %new
\usepackage{multirow} %new
\usepackage{adjustbox} %new

\DeclareMathOperator*{\argmin}{arg\,min} %new
\newcommand{\fset}{N} %new
\newcommand{\fnum}{n} %new
\usepackage{enumitem}

\newcommand{\shapiq}{\href{https://github.com/mmschlk/shapiq}{\url{github.com/mmschlk/shapiq}}} % new
\newcommand{\repository}{\href{https://github.com/FFmgll/KernelSHAP_IQ_Supplementary_Material}{\url{https://github.com/FFmgll/KernelSHAP_IQ_Supplementary_Material}}} %new anonymous.4open.science/r/KernelSHAP_IQ-C2D9/README.md
\usepackage{color,colortbl} %new
\definecolor{tablegray}{gray}{0.9} %new
\usepackage{xcolor} %new
\definecolor{shapiq}{rgb}{0.937,0.153,0.651} %new
\definecolor{svarmiq}{rgb}{0,0.706,0.847} %new
\definecolor{permutation}{rgb}{0.49,0.325,0.871} %new
\definecolor{kernelshapiq}{rgb}{1,0.435,0.} %new
\definecolor{biaskernelshapiq}{rgb}{1, 0.729, 0.031} %new

\icmltitlerunning{KernelSHAP-IQ: Weighted Least Square Optimization for Shapley Interactions}

\begin{document}

\twocolumn[
\icmltitle{KernelSHAP-IQ: Weighted Least Square Optimization for Shapley Interactions}

% It is OKAY to include author information, even for blind
% submissions: the style file will automatically remove it for you
% unless you've provided the [accepted] option to the icml2024
% package.

% List of affiliations: The first argument should be a (short)
% identifier you will use later to specify author affiliations
% Academic affiliations should list Department, University, City, Region, Country
% Industry affiliations should list Company, City, Region, Country

% You can specify symbols, otherwise they are numbered in order.
% Ideally, you should not use this facility. Affiliations will be numbered
% in order of appearance and this is the preferred way.
\icmlsetsymbol{equal}{*}

\begin{icmlauthorlist}
\icmlauthor{Fabian Fumagalli}{bi}
\icmlauthor{Maximilian Muschalik}{lmu,mcml}
\icmlauthor{Patrick Kolpaczki}{upb}
\icmlauthor{Eyke Hüllermeier}{lmu,mcml}
\icmlauthor{Barbara Hammer}{bi}
\end{icmlauthorlist}

\icmlaffiliation{bi}{Bielefeld University, CITEC, D-33619 Bielefeld, Germany}
\icmlaffiliation{lmu}{LMU Munich, D-80539 Munich, Germany}
\icmlaffiliation{mcml}{MCML, Munich}
\icmlaffiliation{upb}{Paderborn University, D-33098, Paderborn, Germany}

\icmlcorrespondingauthor{Fabian Fumagalli}{ffumagalli@techfak.uni-bielefeld.de}
%\icmlcorrespondingauthor{Maximilian Muschalik}{Maximilian.Muschalik@lmu.de}

% You may provide any keywords that you
% find helpful for describing your paper; these are used to populate
% the "keywords" metadata in the PDF but will not be shown in the document
\icmlkeywords{Shapley Interactions, Shapley Value, Explainable Artificial Intelligence, Game Theory}

\vskip 0.3in
]

% this must go after the closing bracket ] following \twocolumn[ ...

% This command actually creates the footnote in the first column
% listing the affiliations and the copyright notice.
% The command takes one argument, which is text to display at the start of the footnote.
% The \icmlEqualContribution command is standard text for equal contribution.
% Remove it (just {}) if you do not need this facility.

\printAffiliationsAndNotice{}  % leave blank if no need to mention equal contribution
%\printAffiliationsAndNotice{\icmlEqualContribution} % otherwise use the standard text.

\begin{abstract}
The Shapley value (SV) is a prevalent approach of allocating credit to machine learning (ML) entities to understand black box ML models.
Enriching such interpretations with higher-order interactions is inevitable for complex systems, where the Shapley Interaction Index (SII) is a direct axiomatic extension of the SV.
While it is well-known that the SV yields an optimal approximation of any game via a weighted least square (WLS) objective, an extension of this result to SII has been a long-standing open problem, which even led to the proposal of an alternative index.
In this work, we characterize higher-order SII as a solution to a WLS problem, which constructs an optimal approximation via SII and $k$-Shapley values ($k$-SII).
We prove this representation for the SV and pairwise SII and give empirically validated conjectures for higher orders.
As a result, we propose KernelSHAP-IQ, a direct extension of KernelSHAP for SII, and demonstrate state-of-the-art performance for feature interactions.
\end{abstract}

\section{Introduction}
\label{sec_introduction}

The Shapley value (SV) \cite{Shapley.1953} is a commonly used theoretical framework from cooperative game theory to allocate credit among entities in machine learning (ML) settings.
Prevalent ML applications of the SV include \emph{feature attribution} \cite{Lundberg.2017} for local interpretability, \emph{feature importance} \cite{Covert.2021} for global interpretability, or \emph{data valuation} to quantify the worth of data points \cite{Ghorbani.2019}.
However, in many application domains, the SV accounting solely for individual entities such as features or data points does not suffice \cite{Kumar.2020, Kumar.2021, Slack.2020, Sundararajan.2020, Wright.2016}.
As illustrated with \cref{fig_language_illustration,fig_intro_illustration}, in complex domains like language modeling \cite{Murdoch.2018}, image classification \cite{Tsang.2018}, or bioinformatics \cite{Winham.2012, Wright.2016}, enriching contributions of individuals with \emph{interactions} among entities is required.

\begin{figure}[t]
    \centering
    \includegraphics[width=0.85\columnwidth]{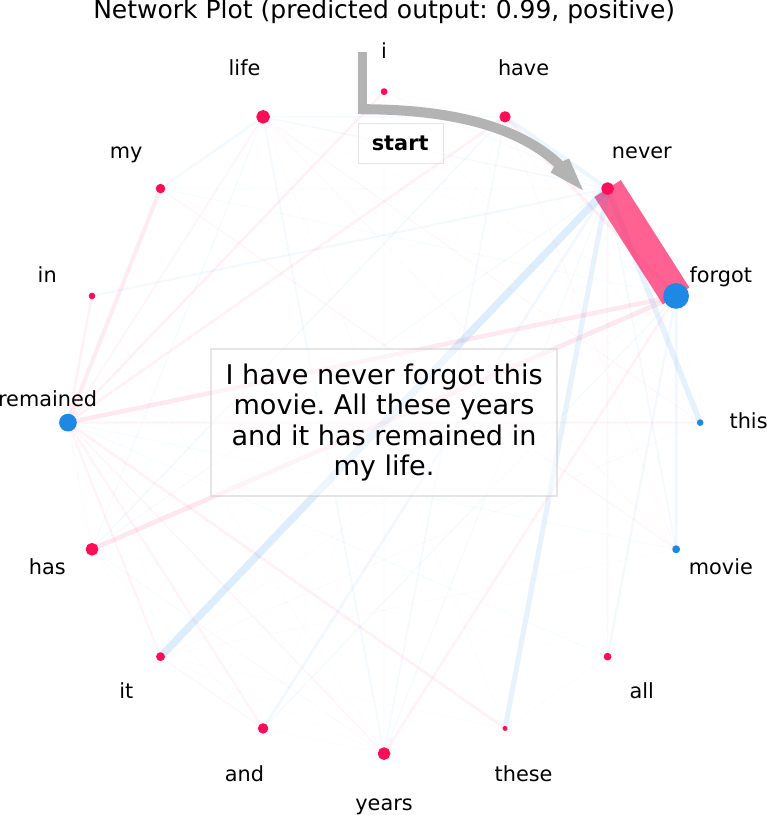}
    \caption{Positive (red) and negative (blue) feature attributions (vertices) and interactions (edges) for a movie review excerpt provided to a sentiment analysis language model. The interaction of ``never'' and ``forget'' highly contributes to the positive sentiment.}
    \label{fig_language_illustration}
\end{figure}

\begin{figure*}[t]
    \centering
    \textbf{Potential Question:} ``Does the \emph{location} of my property affect its price $\hat{y} = 4.54$ (in 100k\$)?''\\[1.5em]
    \begin{minipage}[c]{0.15\textwidth}
        \centering
        \vspace{-1.5em}
        \textbf{$1$-additive\\Explanation\\(SV, 1-SII)}
        \\[3em]
        \textbf{$2$-additive\\Explanation\\(2-SII)}
    \end{minipage}
    \hfill
    \begin{minipage}[c]{0.84\textwidth}
        \includegraphics[width=\textwidth]{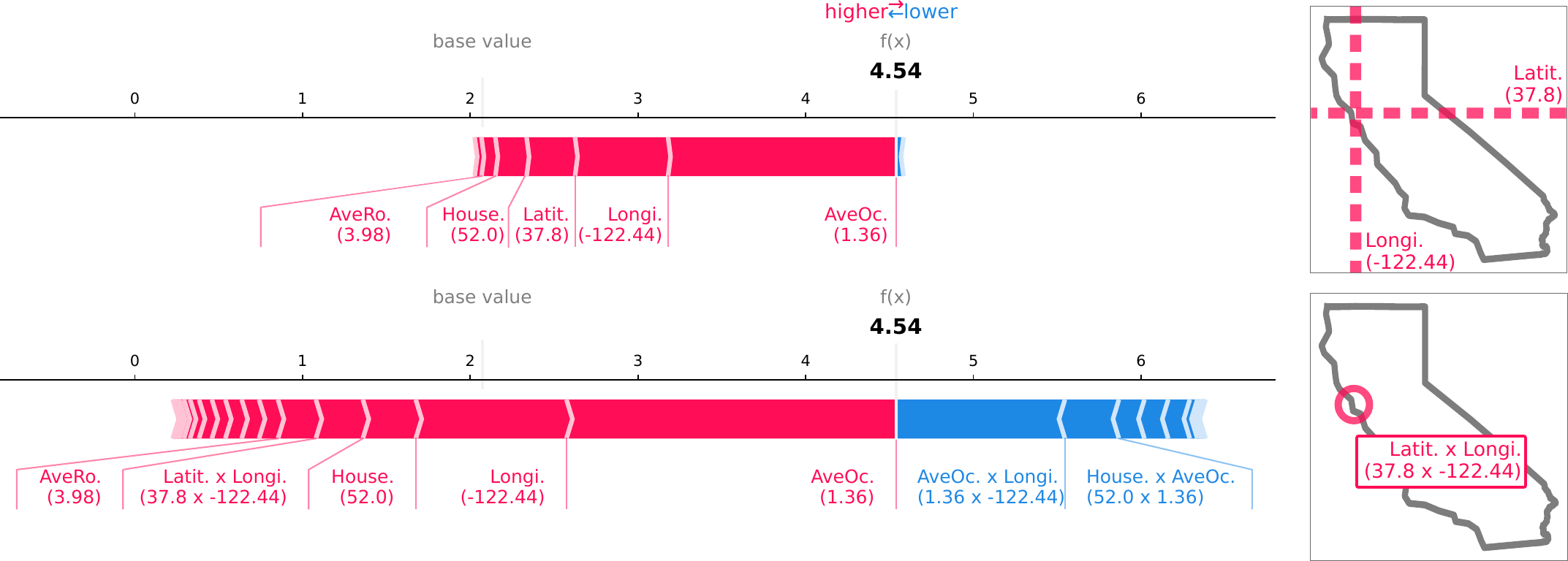}
    \end{minipage}
    \caption{Force plots as first and second order explanations showing positive (red) and negative (blue) feature attributions for a data point of the \emph{California Housing} regression dataset \cite{Kelley.1997}. The SVs show that a \emph{longitude} of $-122.44$ and a \emph{latitude} of $37.8$, positively influences the predicted property's price. Considering 2-SII feature interactions reveals that the positive influence of \emph{latit.} vanishes and only the combination of \emph{longi.} and \emph{latit.} pointing to the exact location, San Francisco, impacts the property's price.}
    \label{fig_intro_illustration}
\end{figure*}

The SV's allocation scheme can be generalized to Shapley interactions and quantify the synergy of groups of entities. 
The Shapley Interaction Index (SII) \cite{Grabisch.1999} is a natural axiomatic extension of the SV to any-order interactions, which yields an exact additive decomposition of any cooperative game \cite{Grabisch.2000}.
However, this decomposition includes an exponential number of components and is not suited for interpretability.
As a remedy, fixing a maximum interaction order $k$, $k$-Shapley values ($k$-SII) \cite{Bord.2023} aggregate SII up to order $k$ uniquely and yield an \emph{interpretable} and \emph{efficient} $k$-additive interaction index.
This aggregation \emph{fairly} distributes the joint payout, e.g.\ the model's prediction, to all individuals \emph{and} groups of entities up to order $k$.

Both, SII and $1$-SII are the SV for individuals.
In this case, besides the classical representation, it is well-known that the SV yields an \emph{optimal} $1$-additive approximation of the cooperative game, which minimizes the weighted least square (WLS) loss over all subsets constrained on the efficiency property \cite{Charnes.1988}.
Extensions of such representations for interactions exist for constant weights, which yields the Banzhaf interaction index \cite{Hammer.1992,Grabisch.2000}.
For Shapley interactions, the Faithful Shapley Interaction Index (FSI) \cite{Tsai.2022} extends on the axiomatic structure of the least square family \cite{Ruiz.1998} by imposing Shapley-like axioms.
However, while FSI yields a unique index, it neither yields SII nor $k$-SII and the dualism between both representations is not as clear as for the SV.
In light of the exponential complexity of the SV and SII, which necessitates efficient approximation techniques, a WLS representation of SII is particularly desirable.
In fact, KernelSHAP \cite{Lundberg.2017} directly utilizes this representation of the SV and yields state-of-the-art (SOTA) approximations in ML applications \cite{Chen.2023,Kolpaczki.2024a}.
For FSI, where a direct extension of KernelSHAP is possible, it has been shown that it is similarly superior over sampling-based approaches \cite{Fumagalli.2023,Kolpaczki.2024b}.

\textbf{Contribution.}
In this work, we show that SII is the solution to a WLS problem, which yields the \emph{optimal} $k$-additive approximation of any cooperative game via $k$-SII and SII.
We then present two variants of KernelSHAP Interaction Quantification (KernelSHAP-IQ), a direct extension of KernelSHAP for higher-order SII.
We empirically demonstrate that KernelSHAP-IQ outperforms existing baselines and yields SOTA performance across various datasets and model classes.
Therein, our main contributions include:
\begin{enumerate}
    \item \textbf{Iterative approximations via SII}: We introduce $k$-additive approximations via $k$-SII and prove an iterative link to SII (Proposition~\ref{prop_iterative_approx}).
    \item \textbf{Optimal $k$-additive approximation via SII}: We prove that pairwise SII are the solution to a WLS problem (Theorem~\ref{thm_kernel_sii}) with empirically validated conjectures for higher orders (Conjecture~\ref{conj_sii}).
    \item \textbf{KernelSHAP-IQ}\footnote{KernelSHAP-IQ is implemented in the open-source \emph{shapiq} explanation library \shapiq.}: We propose an inconsistent and consistent variant of KernelSHAP for SII that yields SOTA performance for local feature interactions.
\end{enumerate}

\textbf{Related Work.}
In cooperative game theory, the studied set functions are known as transferable utility (TU) games, and $k$-additive TU-games, if $k$ is the maximum order \cite{Grabisch.2016}.
Extensions of the SV to interactions were proposed with SII \cite{Grabisch.1999}, FSI \cite{Tsai.2022} among others \cite{Marichal.1999,Sundararajan.2020}.
Recent work was largely motivated by the lack of efficiency of SII, which was recently resolved by the SII-based aggregation $k$-SII \cite{Bord.2023} that extends on (pairwise) Shapley interactions \cite{Lundberg.2020}.
The dualism between semivalues, such as the SV or Banzhaf value \cite{Banzhaf.1964}, and optimal WLS approximations is well known \cite{Ruiz.1998}.
For interactions, these results extend to Banzhaf interactions \cite{Grabisch.1999,Grabisch.2000}.
FSI provides a novel index, which directly yields this dualism for Shapley interactions.
However, FSI differs from SII and $k$-SII, and the classical representation, as well as the link between orders remains mostly unknown \cite{Tsai.2022}.
Lastly, our approach is linked to the \emph{Shapley residuals} \cite{Kumar.2021}.

Approximations of Shapley interactions, were introduced for particular interaction indices by \citet{Sundararajan.2020,Tsai.2022} and arbitrary cardinal interaction indices \cite{Fujimoto.2006} by \citet{Fumagalli.2023,Kolpaczki.2024b}. 
They directly extend on methods for the SV \cite{Castro.2009,Covert.2021,Kolpaczki.2024a}.
KernelSHAP \cite{Lundberg.2017} was introduced for the SV and extended to FSI \cite{Tsai.2022}.
Including SII in KernelSHAP also improved approximations of the SV \cite{Pelegrina.2023}.
Above approximations do not make any assumption about the game and were used for model-agnostic interpretability.
For tree-based models, it was shown that SII can be computed much more efficiently \cite{Muschalik.2024,Zern.2023}, which extends TreeSHAP \cite{Lundberg.2020,Yu.2022} to SII.

Apart from cooperative game theory, interactions were also studied in statistics with functional decomposition \cite{Hooker.2004,Hooker.2007,Lengerich.2020}, where some are linked to game theory \cite{Herbinger.2023,Herren.2022,Hiabu.2023}.
In ML, model-specific variants for deep learning architectures were proposed \cite{Deng.2024,Harris.2022,Janizek.2021,Tsang.2018,Tsang.2020b,Zhang.2021}, as well as model-agnostic methods \cite{Tsang.2020} without axiomatic structures.

\section{From Shapley Values to Interactions}

\paragraph{Notations.}
We use lower-case letters to denote subset sizes, i.e.\ $s := \vert S \vert$, and write $i$ for $\{i\}$ and $ij$ for $\{i,j\}$.
We write constraints in the superscript of a sum, e.g.\ $\sum_{T\subseteq \fset}^{k \leq t \leq \fnum-k}$ indicates a sum over subsets $T \subseteq \fset$ of size $k \leq \vert T \vert \leq \fnum-k$.
We index matrices using subsets, e.g.\ $(\mathbf{X})_{TS}$ refers to the row indexed by $T$ and column indexed by $S$.
Lastly, the indicator function $\mathbf{1}_{A}$ is one if $A$ is fulfilled and $T\sim p$ is a subset sampled according to a probability distribution $p$.

We consider a set of players $\fset := \{1,\dots,\fnum\}$ and the payout of a cooperative game $\nu:\mathcal P(\fset) \to \mathbb{R}$, where $\nu$ models the payout given a subset of players $T \subseteq \fset$.
In the context of local interpretability, $\fset$ might be chosen as the set of features and $\nu$ the prediction of the model given only a subset of features.
For simplicity, we consider games with $\nu(\emptyset)=0$.
In the following, we are interested in quantifying the contribution of a set of players $S\subseteq \fset$ to the payout $\nu$.
For local interpretability, this is known as \emph{feature attribution} for individual players and \emph{feature interaction} for multiple players.
For individual players, the \emph{fair} attribution is the well-known Shapley value (SV) \cite{Shapley.1953}
\begin{equation*}
    \mathcal \phi^{\text{SV}}(i) = \sum_{T \subseteq \fset \setminus i} \tfrac{(\fnum-1-t)!\cdot t!}{\fnum!}\underbrace{\left[\nu(T \cup i) - \nu(T)\right]}_{=: \Delta_i(T)}.
\end{equation*}
The SV is the unique attribution measure that fulfills the linearity, symmetry, dummy, and in particular the efficiency axiom, which distributes the payout $\nu(\fset)-\nu(\emptyset)$ among the players, e.g.\ the prediction of a ML model.
The SV of player $i$ is the weighted average over all marginal contributions $\Delta_i(T) := \nu(T~\cup~i)-\nu(T)$.
For two players $ij$ the marginal contribution to $T \subseteq \fset \setminus ij$ can be extended using the recursion
\begin{align*}\label{eq_delta_recursion}
    \Delta_{ij}(T) = \underbrace{\nu(T \cup ij) -\nu(T)}_{\text{contribution of $ij$ jointly}} - \underbrace{\Delta_i(T)}_{\text{contribution of $i$}} -\underbrace{\Delta_j(T)}_{{\text{contribution of $j$}}},
\end{align*}
i.e.\ the effect of adding both players $ij$, minus their individual contributions.
For feature interactions, a positive value indicates meaningful joint information, e.g.\ exact position versus latitude and longitude separately, whereas a negative value indicates redundancy.
Using above intuition, the marginal contributions can be extended to any subset, which is known as the \emph{discrete derivative} $\Delta_S(T) := \sum_{L\subseteq S}(-1)^{s-l}\nu(T \cup L)$ of a set of players $S \subseteq \fset$ to a set of participating players $T \subseteq \fset \setminus S$.
Similar to the SV, an interaction index is then defined as a weighted average over discrete derivatives.
\begin{definition}[Shapley Interaction Index \cite{Grabisch.1999}]\label{def_sii}
The Shapley Interaction Index (SII) is
\begin{align*}
    \phi^{\text{SII}}(S) := \sum_{T \subseteq \fset \setminus S} \frac{(\fnum-s-t)! \cdot t!}{(\fnum-s+1)!} \Delta_S(T).
\end{align*}
We refer to all SIIs of order $s$ as $\phi^{\text{SII}}_s := \{\phi^{\text{SII}}(S) : \vert S \vert = s \}$.
\end{definition}
The SII includes the SV as $\phi^{\text{SII}}_1 \equiv \phi^{\text{SV}}$.
It was shown by \citet{Grabisch.1999} that the SII satisfies the (generalized) linearity, symmetry and dummy axiom.
The SII is further the unique interaction index that additionally satisfies the \emph{recursive} axiom, which naturally extends the recursion of discrete derivatives and links higher- to lower-order interactions.
For pairwise interaction, it requires
\begin{equation*}
    \phi(ij) = \underbrace{\phi^{[ij]}([ij])}_{\text{SV of merged player}} - \underbrace{\phi^{-j}(i)}_{\text{SV of $i$, $j$ excluded}} - \underbrace{\phi^{-i}(j)}_{\text{SV of $j$, $i$ excluded}}.
\end{equation*}
Here, $\phi^{[ij]},\phi^{-j},\phi^{-i}$ correspond the SV of games, where $[ij]$ are merged, $j$ is absent, and $i$ is absent, respectively.
This recursion is similar to $\Delta_{ij}$, where the right-hand side is the marginal contribution of the merged player $[ij]$ minus the marginal contributions of the individual players, where the other player is absent in $T$.

\subsection{$k$-Additive Interactions and $k$-Shapley Values}
While the SII is an intuitive extension of the SV to arbitrary interactions, it is not suitable for interpretability as it consists of $2^\fnum$ non-trivial components.
We thus introduce $k$-additive interaction indices $\Phi_k(S)$, which admit a lower complexity with non-zero interactions up to order $1\leq \vert S \vert \leq k \leq \fnum$, i.e.\ $\Phi_k(S) := 0$ for $\vert S \vert > k$, which we omit for simplicity.

\begin{definition}[Efficiency]
    A $k$-additive interaction index $\Phi_k(S)$ is efficient if and only if $\sum_{S\subseteq \fset}^{1\leq \vert S \vert \leq k} \Phi_k(S) = \nu(\fset)$.
\end{definition}

The SV and the SII subsume all higher order effects, which has been shown by \citet{Grabisch.1997} and was recently re-discovered and linked to feature attributions \cite{Bord.2023,Hiabu.2023}.
When constructing an efficient $k$-additive interaction index $\Phi_k$ based on SII, it is required to account for double counting of effects with order larger than $k$.
For $k=2$, this was introduced as the (pairwise) Shapley interactions \cite{Lundberg.2020} and extended to higher-order $k$-Shapley values \cite{Bord.2023}.
\begin{definition}[$k$-Shapley Values \cite{Bord.2023}]\label{def_k_sii}
The $k$-Shapley value ($k$-SII) is defined as $\Phi_{k}(S) :=$
 \begin{align*}     
        \begin{cases}
            \phi^{\text{SII}}(S) &\text{if } \vert S \vert = k
            \\
            \Phi_{k-1}(S) + B_{k-\vert S \vert} \sum_{\tilde S \subseteq \fset \setminus S}^{\vert S \vert + \tilde S = k} \phi^{\text{SII}}(S \cup \tilde S) &\text{if } \vert S \vert < k
        \end{cases}
    \end{align*}
    with $1\leq \vert S \vert \leq k \leq \fnum$, $\Phi_0 \equiv 0$ and Bernoulli numbers $B_n$.
\end{definition}
Note that 1-SII is again the SV, i.e.\ $\Phi_1 \equiv \phi^{\text{SV}}$.
Furthermore, $k$-SII requires SII values up to order $k$ and the highest order is equal to SII.
It was shown that $\Phi_k$ is efficient and that the Bernoulli numbers are the unique choice of weighting to achieve this property \cite{Bord.2023}.
For $k=\fnum$ the explicit form is drastically simplified and corresponds to the well-known Möbius transform \cite{Bord.2023}.
The $k$-SII therefore defines a flexible low-complexity $k$-additive interaction index, which is a straight-forward aggregation of SII and can be directly used for an interpretable representation.
In Section~\ref{sec_main}, we will show and leverage that $k$-SII is directly linked to an \emph{optimal} $k$-additive approximation of $\nu$.

\subsection{Approximations of Shapley Interactions}
In applications, the limiting factor of SII and the SV are the amount of evaluations of $\nu$, where both require $2^\fnum$ evaluations for exact computation.
Therefore, multiple approximation techniques have been proposed, as illustrated in \cref{tab_approx_landscape}.
In the following, we introduce the main approaches for SII, which are based on Monte Carlo sampling of different representations, and extend on techniques for the SV.
These methods will be used as baseline methods in our experiments.
\emph{Permutation Sampling} \cite{Castro.2009,Tsai.2022} samples random permutations $\pi$ of $\fset$ and is based on the representation $\phi^{\text{SII}}(S) = \mathbb{E}_\pi[\mathbf{1}_{\pi \in S} \Delta_S(u_S(\pi))]$, where $u_S(\pi)$ is the set of elements that appear in $\pi$ before any element of $S$, and $\mathbf{1}_{\pi \in S}$ is one, if all elements of $S$ appear consecutively in $\pi$, and zero otherwise.
For permutation sampling only a small fraction of interaction estimates are updated with each evaluation.
An alternative method is \emph{SHAP-IQ} \cite{Fumagalli.2023} that samples subsets $T \subseteq \fset$ directly and utilizes a representation $\phi^{\text{SII}}(S) = \mathbb{E}_T[\nu(T)\omega_s(t,\vert T\cap S\vert)]$, where $\omega_s$ are SII-specific weights adjusted to the sampling distribution.
SHAP-IQ can therefore utilize every evaluation to update all interaction estimates.
\emph{SVARM-IQ} \cite{Kolpaczki.2024b} further improves approximation by stratifying over subset sizes $t$ and intersections $T \cap S = L$ with
\begin{equation*}
    \phi^{\text{SII}}(S) = \sum_{t=0}^\fnum \sum_{L\subseteq S}\mathbb{E}[\nu(T)\tilde\omega_s(t,L) \mid \vert T \vert = t, T \cap S = L],
\end{equation*}
where $\tilde \omega_s(t,L) := \omega_s(t,\ell) \cdot p(\vert T \vert = t,T\cap S = L)$, i.e. adjusted by the strata probabilities.
The core idea is to compute strata estimates by sampling $T\subseteq \fset$ and treat it as a random sample for one specific stratum of every interaction.
This allows again to utilize every evaluation for all interaction estimates.
It was shown empirically that SVARM-IQ is the most powerful estimator for SII \cite{Kolpaczki.2024b}.

\subsection{Approximating Shapley Values with KernelSHAP}
While recent research extended techniques of the SV to the SII, one particularly powerful method has not been considered, so far.
KernelSHAP \cite{Lundberg.2017} is based on a representation of the SV as a solution to a constrained WLS problem \cite{Charnes.1988,Covert.2021}
\begin{align}\label{align_kernelshap}
\begin{aligned}
    \phi^{\text{SV}} &= 
        \argmin_{\phi \in \mathbb{R}^\fnum} \sum_{T \subseteq \fset}^{0 < \vert T \vert < \fnum} \mu(t) \left[\nu(T) - \sum_{i\in T} \phi(i)\right]^2,
    %\argmin_{\phi \in \mathbb{R}^\fnum} \mathbb{E}_{T \sim p}\left[\left(\nu(T) - \sum_{i\in T} \phi(i)\right)^2\right],
    \\
    &\text{such that } \sum_{i\in \fset} \phi(i) = \nu(\fset),
\end{aligned}
\end{align}
where $p(T) \propto \mu(t) := \binom{\fnum-2}{t-1}^{-1}$ for $1\leq t\leq \fnum-1$.

\begin{table}[t]
\centering
    \caption{Summary of approximation techniques extended to SII}\label{tab_approx_landscape}
    \vspace{0.5em}
    \adjustbox{max width=\columnwidth}{%
    %\resizebox{\columnwidth}{!}{
    %\centering
    %\begin{tabular}{p{1cm}p{2.8cm}p{3cm}p{3.5cm}p{3.5cm}}
    {
    \renewcommand{\tabcolsep}{3pt}
    \begin{tabular}{ccc}
    \toprule
     \textbf{Method}&\textbf{SV}&\textbf{SII} ($k$-SII)\\\midrule
     \textbf{Permutation}&\cite{Castro.2009}&\cite{Tsai.2022}\\
     \textbf{SHAP-IQ}&\cite{Covert.2021}&\cite{Fumagalli.2023}\\
     \textbf{SVARM}&\cite{Kolpaczki.2024a}&\cite{Kolpaczki.2024b} \\
     \textbf{KernelSHAP}&\cite{Lundberg.2017}&\textbf{KernelSHAP-IQ}\\ \bottomrule
    \end{tabular}
    }
    }
\end{table}

KernelSHAP considers above objective as an expectation
\begin{align*}
    \sum_{T \subseteq \fset}^{0<t<\fnum} \mu(t) [\nu(T) - \sum_{i\in T} \phi(i)]^2 
    = \mathbb{E}_{T}[(\nu(T) - \sum_{i\in T} \phi(i))^2]
\end{align*}
over sampled subsets $T \sim p(T) \propto \mu(t)$, and approximates the expectation using Monte Carlo sampling.
Then, the WLS problem is solved with the approximated objective.
The constraint is thereby included with sets $T \in \{\emptyset,\fset\}$ and a large weight $\mu_\infty \gg 1$, where we prove in \cref{thm_kernel_sv} that this yields the SV.
While the theoretical analysis of KernelSHAP remains difficult, it was shown that the approach works especially well for ML applications \cite{Chen.2023,Kolpaczki.2024a,Lundberg.2017}.
In the following, we rigorously describe this approach and discover such a representation for SII, which we use for efficient approximation.

\section{Weighted Least Square Optimization for Shapley Interactions}\label{sec_main}
The SII is an axiomatic extension of the SV and constructs $k$-SII as an efficient $k$-additive interaction index used for interpretability.
In this section, we give an alternative characterization of SII, as the solution to a WLS problem, which extends on the representation of the SV given in \cref{align_kernelshap}.
Based on this representation, we introduce KernelSHAP-IQ, an extension of KernelSHAP for Shapley interactions.
We rigorously prove its convergence for pairwise interactions and conclude with empirically validated conjectures for higher orders.
All proofs are deferred to \cref{appx_proofs}.

The representation of the SV in \cref{align_kernelshap}, utilized by KernelSHAP, shows that the approximation $\sum_{i \in T} \phi(i) \approx \nu(T)$ of the game is \emph{optimal} in terms of the WLS objective, if and only if $\phi(i)$ are the SVs.
This approximation can be viewed as a $1$-additive approximation of the game, which we now directly extend to $k$-additive approximations.

\begin{definition}[$k$-additive Approximation]\label{def_k_order}
    The $k$-additive approximation induced by the $k$-SII $\Phi_k$ is given by
     \begin{equation*}
         \hat\nu_k(T) := \sum_{S \subseteq T}^{1 \leq \vert S \vert\leq k} \Phi_{k}(S) \text{ for } T \subseteq \fset.
     \end{equation*}
\end{definition}

Clearly, $\hat\nu_1$ is the approximation given in \cref{align_kernelshap}.
Linking higher orders of the SII to such an optimal approximation of $\nu$ has remained unknown.
In order to extend KernelSHAP to interactions, recent research introduced the FSI as an alternative index \cite{Tsai.2022}.
FSI constructs an optimal $k$-additive approximation of $\nu$ by extending on the KernelSHAP representation with interactions up to order $k$.
One disadvantage of FSI is that the solution directly depends on $k$, which complicates the relationship between representations of different order.
Furthermore, the exact representation of lower-order FSI with discrete derivatives remains unknown \cite{Tsai.2022}.
In contrast, we show that $k$-SII constructed from SII values yield an appealing iterative approximation of $\nu$, which additively extends on the previous approximation by including a weighted average of SII of order $k$.
We then show that SII via $k$-SII is linked to an \emph{optimal} $k$-additive approximation of $\nu$.
This yields a suitable framework, where SII is fixed independent of $k$ and $k$-SII yields an interpretable $k$-additive interaction index, as summarized in \cref{fig_link_SII_kSII_approx}.

\subsection{Optimal Approximations via Shapley Interactions}

%KernelSHAP utilizes that the SV is the solution to a WLS problem yielding an \emph{optimal} $1$-additive approximation of $\nu$ with $\nu(T) \approx \sum_{i \in T} \phi^{\text{SV}}(i)$, which is exact for $\nu(\fset) = \sum_{i \in \fset} \phi^{\text{SV}}(i)$ due to the efficiency axiom.

\begin{figure}[t]
    \centering
    \includegraphics[width=\columnwidth]{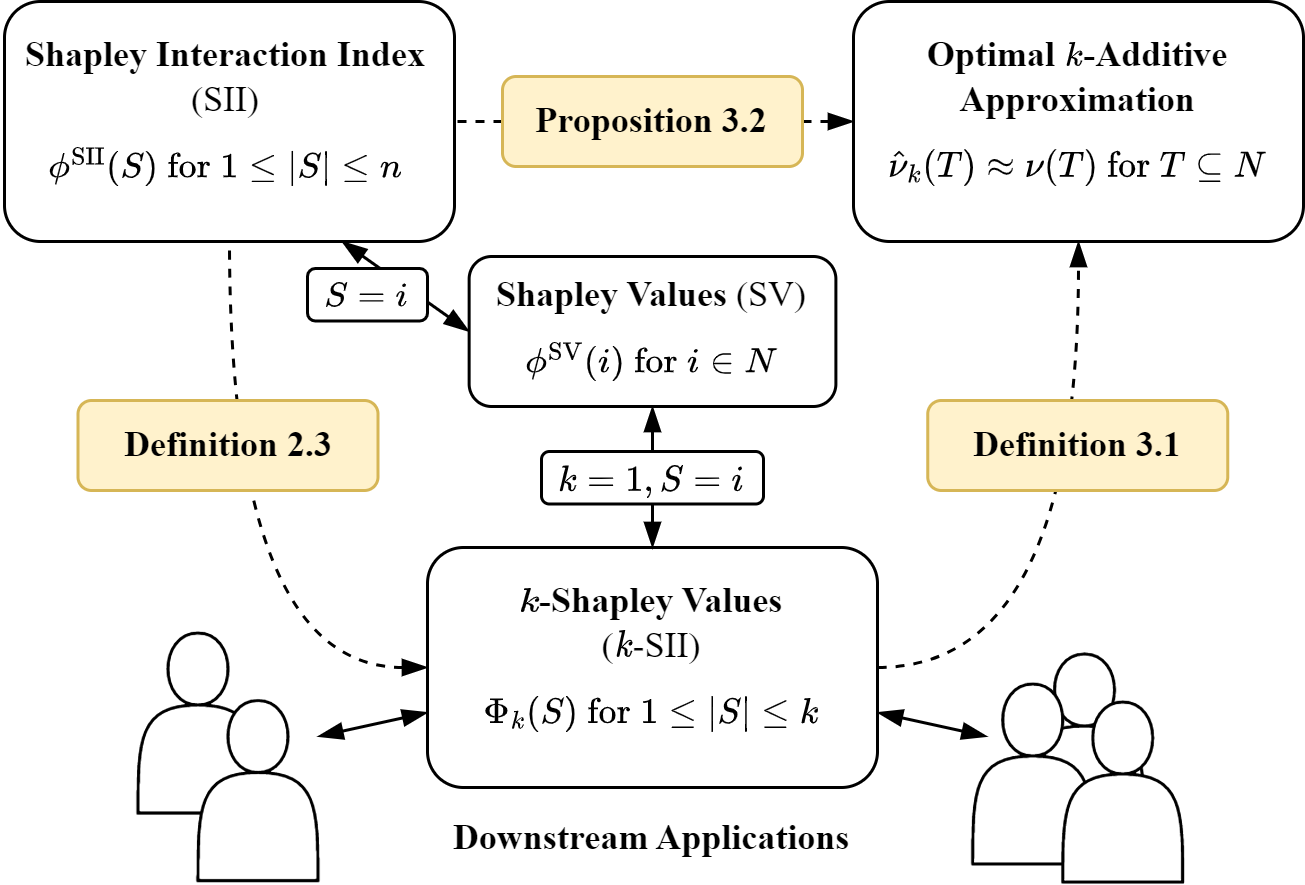}
    \caption{Links between SII, $k$-SII and the $k$-additive approximation $\hat\nu_k$. The SII captures the average contribution of $S$ to $\nu$, which constructs the $k$-additive interaction index $k$-SII, which is used for interpretation. Both are linked to $\hat\nu_k$, where SII yields an \emph{optimal} approximation, which iteratively constructs $\hat\nu_k$.}
    \label{fig_link_SII_kSII_approx}
\end{figure}

%Similar to the SV, we construct a $k$-additive approximation of $\nu(T)$ as the sum of $k$-SII values contained in $T$.

%Clearly, $\hat\nu_k$ includes $k$-SII and thereby SII up to order $k$.
We first link $\hat\nu_k$ and SII, which reveals a surprising recursion. 

\begin{proposition}[Iterative Approximation]\label{prop_iterative_approx}
If $k\geq 2$, then
\begin{equation*}
    \hat\nu_k(T)  = \hat\nu_{k-1}(T) + \sum_{S \subseteq \fset}^{\vert S \vert = k} \phi^{\text{SII}}(S) \lambda(\vert S \vert,\vert S \cap T \vert)
\end{equation*}
with weights $\lambda(k,\ell) := \sum_{r=1}^{\ell}\binom{\ell}{r} B_{k-r}$ and $\lambda(k,0) := 0$.
\end{proposition}
\begin{corollary}\label{cor_nuhat}
    It holds $\lambda(k,k)=0$ for $k>1$, $\hat\nu_1(T) = \sum_{i \in T} \phi^{\text{SV}}(i)$, and $\hat\nu_2(T) =\hat\nu_1(T) - \frac 1 2 \sum_{{ij} \subseteq \fset}^{\vert ij \cap T\vert = 1}\phi^{\text{SII}}(ij)$.
\end{corollary}
Increasing $k$ successively builds on the SV-based approximation $\hat\nu_1(T)$ from \cref{align_kernelshap} by including all SII of order $k$, i.e.\ $\binom{\fnum}{k}$ components.
With increasing $k$, the approximation $\hat\nu_k$ includes more interaction terms, which yields a complexity-accuracy trade-off.
In this context, $\hat\nu_1$ based on the SVs is the least complex, and $\hat\nu_n$ is the most complex and most faithful approximation, where $\hat\nu_\fnum(T) = \nu(T)$ follows from a game-theoretic result \cite{Grabisch.2000}.

\begin{corollary}\label{cor_d_sii}
    For any $T\subseteq \fset$, it holds $\nu(T)=\hat\nu_\fnum(T)$.
\end{corollary}

In this case, $\fnum$-SII values ($k = n$ for $k$-SII), which construct $\hat\nu_\fnum$, are the Möbius transform \cite{Bord.2023}, which is related to the well-known functional ANOVA decomposition \cite{Hooker.2004} from statistics, if $\nu$ are conditional expectations \cite{Herren.2022}.

\paragraph{Optimality of the SII.}
We now link SII of order $k$, $\phi^{\text{SII}}_k$, to an  \emph{optimal} $k$-additive approximation of $\nu$ in terms of a WLS objective.
More concretely, we show that for all interaction indices $\phi_k$ of order $k$ with induced approximation via \cref{prop_iterative_approx}, the SII $\phi_k^{\text{SII}}$ yields the optimal approximation given a specific WLS objective.
In the following, it is convenient to use matrix notations, where the rows correspond to subsets $T \subseteq \fset$ and the columns to interactions $S \subseteq \fset$ of order $\vert S \vert = k$.
We first introduce the residual vector $\mathbf{y}_k \in \mathbb{R}^{2^\fnum}$ as $(\mathbf{y}_1)_T := \nu(T)$ and for $k\geq 2$
\begin{equation*}
    (\mathbf{y}_k)_T := \nu(T)-\hat\nu_{k-1}(T) \text{ for } T \subseteq N.
\end{equation*}
Next, we introduce the coefficient matrix $\mathbf{X}_k \in \mathbb{R}^{2^\fnum \times \binom{\fnum}{k}}$ as
\begin{align*}
    (\mathbf{X}_k)_{TS} := \lambda(\vert S \vert,\vert T \cap S \vert) \text{ for } T,S \subseteq N \text{ with } \vert S \vert = k.
\end{align*}
With SII of order $k$, $\phi^{\text{SII}}_k \in \mathbb{R}^{\binom{\fnum}{k}}$, and Proposition~\ref{prop_iterative_approx} the approximation error of $\nu$ and $\hat\nu_k$ is given by $\mathbf{y}_k -\mathbf{X}_k \phi^{\text{SII}}_k$.
Introducing a diagonal weight matrix $\mathbf{W}_k \in \mathbb{R}^{2^\fnum\times 2^\fnum}$, we aim to characterize $\phi^{\text{SII}}_k$ as the solution to a WLS problem
\begin{equation}\label{eq_WLS}
    \phi^* = \argmin_{\phi_k \in \mathbb{R}^{\binom{\fnum}{k}}} \left\Vert \sqrt{\mathbf{W}_k}  \left(\mathbf{y}_k - \mathbf{X}_k  \phi_k\right)\right\Vert_2^2,
\end{equation}
where $\Vert \cdot \Vert_2$ is the Euclidean norm.
Note that this WLS problem for $k=1$ is similar to \cref{align_kernelshap} but includes all subsets $T \subseteq \fset$.
The solution is explicitly given as
\begin{equation}\label{eq_WLS_solution}
\phi^* = (\mathbf{X}^T_k \mathbf{W}_k \mathbf{X}_k)^{-1} \mathbf{X}^T_k \mathbf{W}_k \cdot \mathbf{y}_k.
\end{equation}

We now define the weights of $\mathbf{W}_k$ with $\mu_\infty \gg 1$ as
\begin{equation*}
   (\mathbf{W}_k)_{TT} := \mu_k(t) := \begin{cases}
         \binom{\fnum-2 \cdot k}{t-k}^{-1} &\text{if } k\leq t\leq \fnum-k
        \\
        \mu_\infty&\text{else.}
    \end{cases}
\end{equation*}
\begin{remark}
    The weights $\mu_k$ appear naturally from $\phi^{\text{SII}}(S) = \sum_{T \subseteq \fset} \nu(T) \omega_s(t,\vert T \cap S\vert)$ \cite{Fumagalli.2023} as the common factors $\mu_k(t) \propto \omega_s(t,0),\dots\omega_s(t,k)$ for $k \leq t \leq \fnum-k$.
    Since $(\mathbf{X}^T_k \mathbf{W}_k \mathbf{X}_k)^{-1} \in \mathbb{R}^{\binom{\fnum}{k} \times \binom{\fnum}{k}}$ does not depend on $T$, it holds that $\mu_k$ controls the final weight in $\mathbf{X}_k^T \mathbf{W}_k$ for all interactions $S$. 
    For a more detailed discussion, see \cref{sec_appendix_intuition}.
\end{remark}

For $k=1$, $\mu_1$ are the KernelSHAP weights and the SV is the solution of the WLS problem for $\mu_\infty \to \infty$, which justifies KernelSHAP.
\begin{theorem}[KernelSHAP]\label{thm_kernel_sv}
Let $\fnum\geq 2$ and $(\mathbf{W}_1)_{TT} := \mu_1(t)$. Then the SV is represented as
    \begin{align*}
       \phi^{\text{SV}}  = \lim_{\mu_\infty \to \infty} \argmin_{\phi_1 \in \mathbb{R}^\fnum} \left\Vert \sqrt{\mathbf{W}_1} \left(\mathbf{y}_1 - \mathbf{X}_1 \phi_1\right)\right\Vert_2^2.
    \end{align*}
\end{theorem}
A large weight $\mu_{\infty}$ requires the solution to achieve lower approximation error for this subset, which is a soft constraint.
The subsets with weights $\mu_\infty$ are the empty set and $N$.
For the empty set, the approximation error is constant, and is thus not influenced by the solution of the WLS problem.
For $\fset$, the approximation error is $ (\mathbf{y}_1)_\fset-(\mathbf{X}_1 \phi)_{\fset} = \nu(\fset)-\sum_{i\in\fset}\phi(i)$, which is zero if efficiency holds.
Intuitively, with $\mu_\infty \to \infty$, the soft constraint becomes a hard constraint, which requires zero approximation error, i.e. efficiency with $\nu(\fset) = \sum_{i\in \fset} \phi(i)$.
This was argued in KernelSHAP \cite{Lundberg.2017} without formal proof.
\cref{thm_kernel_sv} is a formal proof that this intuition holds, and indeed yields the SV satisfying the efficiency constraint from \cref{align_kernelshap}. 

For pairwise SII, we now present a novel representation akin to \cref{thm_kernel_sv}.
\begin{theorem}[KernelSHAP-IQ, $k=2$]\label{thm_kernel_sii}
Let $\fnum\geq4$ and $(\mathbf{W}_2)_{TT} := \mu_2(t)$. Then the pairwise SII is represented as
    \begin{align*}
       \phi^{\text{SII}}_2  = \lim_{\mu_\infty \to \infty} \argmin_{\phi_2 \in \mathbb{R}^{\binom{\fnum}{2}}} \left\Vert \sqrt{\mathbf{W}_2}  \left(\mathbf{y}_2 - \mathbf{X}_2  \phi_2\right)\right\Vert_2^2.
    \end{align*}
\end{theorem}

While $\mu_\infty \to \infty$ corresponds to a constraint for $k=1$ that ensures efficiency, its behavior for $k=2$ is less clear.
In fact, $\mu_2(1) = \mu_2(n-1) = \mu_{\infty}$, but for two such subsets $T \in \{\ell,N \setminus \ell\}$ with $\ell \in N$ the approximation is for both the same value as $(X_2 \phi_2)_{T}= -\frac{1}{2}\sum_{ij \subseteq \fset}^{\vert T \cap ij\vert =1} \phi(ij) = -\frac 1 2 \sum_{j\in \fset}^{j \neq \ell} \phi(j\ell)$, which implies that the corresponding hard constraint of zero approximation error cannot be satisfied. 
Consequently, contrary to $k=1$, the soft constraints with $\mu_\infty$ cannot be re-written to a hard-constrained WLS problem, since it would not have a solution in general.

\subsection{Optimal Higher-Order Approximations}
For $k>2$ we were unable to find closed-form solutions and we suspect that our proof may not be suited for finding these in general.
However, we empirically validated that $\mu_k$ can also be used for higher-order approximations, which we summarize in the following conjectures.
The first one concerns the structure of the inverse in \cref{eq_WLS_solution}.
\begin{conjecture}\label{conj_inverse}
Let $\fnum\geq 2k$ and define the \emph{precision matrix} 
\begin{align*}
    \mathbf{A}_k := \lim_{\mu_\infty \to \infty} (\mathbf{X}^T_k \mathbf{W}_k \mathbf{X}_k)^{-1},
\end{align*}
Then, for $S_1,S_2 \subseteq \fset$ with $\vert S_1 \vert=\vert S_2 \vert = k$, it holds
\begin{equation*}
    (\mathbf{A}_k)_{S_1 S_2} = \frac{(-1)^{k-\vert S_1 \cap S_2\vert}}{\fnum -k +1}\binom{\fnum-k}{k-\vert S_1 \cap S_2\vert}^{-1}.
\end{equation*}
\end{conjecture}
Referring to $\mathbf{A}_k$ as the precision matrix is due to the interpretation of $\mathbf{A}_1^{-1}$ as the \emph{covariance matrix} for the SV by \citet{Covert.2021}.
The second conjecture poses a higher-order representation of SII.
In fact, if Conjecture~\ref{conj_inverse} holds, then the subsets $T \subseteq \fset$ with finite weights, i.e.\ $k \leq \vert T \vert \leq \fnum-k$ yield the correct weights for SII.

\begin{conjecture}\label{conj_sii}
Let $\fnum\geq 2k$ and $\mathcal T_k := \{T \subseteq \fset : k \leq \vert T \vert \leq \fnum-k\}$, which splits for $T \subseteq \fset$ the vector $\mathbf{y}_k$ into
\begin{align*}
    &(\mathbf{y}^{+}_k)_T := \mathbf{1}_{T \in \mathcal T_k} \cdot \mathbf{y}_k(T),    &(\mathbf{y}^{-}_k)_T := \mathbf{1}_{T \notin \mathcal T_k} \cdot \mathbf{y}_k(T).
\end{align*}
Let further for $T \subseteq \fset$ and interaction $S \subseteq \fset$ with $\vert S \vert = k$
\begin{align*}
    (\mathbf{Q}_k)_{ST} := \begin{cases}
        (-1)^{s-\vert S\cap T\vert} m_s(t-\vert S \cap T \vert), &\text{if } T \notin \mathcal{T}_k,
        \\
        0, &\text{if } T \in \mathcal{T}_k,
    \end{cases}
\end{align*}
where $m_s(t) := \frac{(\fnum-s-t)!\cdot t!}{(\fnum-s+1)!}$ from \cref{def_sii}. Then,
\begin{equation*}
       \phi^{\text{SII}}_k  = \mathbf{Q}_k \mathbf{y}^{-}_k + \lim_{\mu_\infty \to \infty} \argmin_{\phi_k \in \mathbb{R}^{\binom{\fnum}{k}}} \left\Vert \sqrt{\mathbf{W}_k} \left(\mathbf{y}^+_k - \mathbf{X}_k \phi_k\right)\right\Vert_2^2.
\end{equation*}
\end{conjecture}

\cref{conj_sii} states that higher-order SII are again represented as the solution to a WLS problem.
However, in contrast to $k\leq 2$, the representation does only hold partially for $T \in \mathcal{T}_k$.
For $T \notin \mathcal T_k$, $\mathbf{Q}_k$ contains all SII weights \cite{Fumagalli.2023}, which are required for the correct weighting.
This is necessary as, in contrast to $k\leq 2$, the weights do not converge to the SII weights for $\mu_\infty \to \infty$ and $T \notin \mathcal T_k$, which requires to adjust the representation using $\mathbf{y}^+_k$ and $\mathbf{y}^-_k$.
Presumably, different $\mu_\infty$ could resolve this issue, which we leave to future research.

\subsection{KernelSHAP-IQ for Shapley Interactions}
We introduce two variants of KernelSHAP-IQ for SII based on distinct optimization problems.
Inconsistent KernelSHAP-IQ solves a single WLS problem and does not converge to SII.
In contrast, (consistent) KernelSHAP-IQ solves iteratively a WLS problem for every order and converges to SII.
Similar to KernelSHAP, we solve an \emph{approximated objective} by sampling subsets from a distribution $p^*(T) \propto q(t)$ with \emph{sampling weights} $q(t) \geq 0$ for $t=0,\dots,\fnum$.
We re-write \cref{eq_WLS} as
\begin{equation*}
    \phi^* = \argmin_{\phi_k \in \mathbb{R}^{\binom{\fnum}{k}}}\mathbb{E}_{T \sim p^*}\left[\left\Vert \sqrt{\mathbf{W}^*_k}  \left(\mathbf{y}_k - \mathbf{X}_k  \phi_k\right)\right\Vert_2^2\right],
\end{equation*}
where we adjusted $(\mathbf{W}^*_k)_{TT} := (\mathbf{W}_k)_{TT}/p^*(T)$.
According to our previous results, the right-hand side has to be solved for $\mu_\infty \to \infty$, and we thus set $\mu_\infty \gg 1$ to account for the limit.
KernelSHAP-IQ estimates are computed by solving the \emph{approximated} WLS objective with $\hat{\mathbf{W}}_k,\hat{\mathbf{X}}_k$ and $\hat{\mathbf{y}}_k$, containing sampled subsets $T \subseteq \fset$ at each row.
More details and pseudocode can be found in \cref{appx_sec_algos}.

\textbf{Inconsistent KernelSHAP-IQ.}
Our baseline method, \emph{inconsistent KernelSHAP}, solves a single WLS problem using the KernelSHAP weights.
We include all interactions up to order $k$ via the weighting in Proposition~\ref{prop_iterative_approx}, i.e.\ in total $\sum_{\ell=1}^k \binom{\fnum}{\ell}$ interactions.
More formally, introducing the stacked matrices $\mathbf{X}_{\leq k} := (\mathbf{X}_1,\dots,\mathbf{X}_k)$ and $\phi_{\leq k} := [\phi_1,\dots,\phi_k]^T$, this approach solves
    \begin{align}\label{eq_inconsistent}
       \phi^* = \lim_{\mu_\infty \to \infty} \argmin_{\phi_{\leq k}} \left\Vert \sqrt{\mathbf{W}_1} \left(\mathbf{y}_1 - \mathbf{X}_{\leq k} \phi_{\leq k}\right)\right\Vert_2^2.
    \end{align}
Inconsistent KernelSHAP solves the WLS problem with the approximated objective to obtain estimates for SII.

\begin{remark}
Interestingly, we observe empirically that solving \cref{eq_inconsistent} yields the exact SVs.
\end{remark}

\begin{remark}
Inconsistent KernelSHAP-IQ is related, but not equivalent, to $k_\text{ADD}$-SHAP \cite{Pelegrina.2023} and FSI.
$k_\text{ADD}$-SHAP includes $\phi^{\text{SII}}(\emptyset)$ and a sum starting from zero in $\lambda$ in $\mathbf{X}_{\leq k}$.
FSI uses binary weights $\mathbf{X}_{\leq k}$, and effectively solves for $k$-SII $\Phi_k$ instead of SII $\phi^{\text{SII}}_{\leq k}$.
\end{remark}

While higher-order estimates do not converge to SII, inconsistent KernelSHAP-IQ surprisingly yields high-quality estimations in low-budget settings.
Utilizing our novel representations we formalize (consistent) KernelSHAP-IQ, which yields SOTA approximation \emph{and} converges to SII.

\begin{algorithm}[t]
    \caption{KernelSHAP-IQ}
    \label{alg_kernelshap_iq}
    \begin{algorithmic}[1]
    \REQUIRE order $k$, sampling weights $q$, budget $b$.
    \STATE $\{T_i\}_{i=1,\dots,b}, \{w_{T}\}_{T=T_1,\dots,T_b} \gets \textsc{\texttt{Sample}}(q,b)$ 
    \STATE $\hat{\mathbf{y}}_1 \gets [\nu(T_1),\dots,\nu(T_b)]^T$
    \FOR[iterative approximation]{$\ell=1,\dots,k$}
        \FOR{$T \in \{T_i\}$ and $\vert S \vert = \ell$}
            \STATE $(\hat{\mathbf{X}}_{\ell})_{TS} \gets \lambda(\vert S \vert,\vert T \cap S \vert)$ \COMMENT{Bernoulli weighting}
            \STATE $(\hat{\mathbf{W}}^*_{\ell})_{TT} \gets \mu_{\ell}(t) \cdot w_{T}$ 
        \COMMENT{weight adjustment}
        \ENDFOR
        \IF[higher order split]{$\ell \leq 2$}
            \STATE $\hat{\phi}_\ell \gets \textsc{\texttt{SolveWLS}}(\hat{\mathbf{X}}_\ell,\hat{\mathbf{y}}_\ell,\hat{\mathbf{W}}^*_\ell)$
        \ELSE
        \FOR{$T \in \{T_i\}$ and $\vert S \vert = \ell$} 
        \STATE $(\mathbf{Q}_\ell)_{ST} \gets \mathbf{1}_{T \notin \mathcal T_\ell}\cdot \textsc{\texttt{SIIWeight}}(T,S)$
        \STATE $(\hat{\mathbf{y}}^+_\ell)_T \gets \mathbf{1}_{T \in \mathcal T_\ell}(\hat{\mathbf{y}}_\ell)_T$
        \STATE $(\hat{\mathbf{y}}^-_\ell)_T \gets \mathbf{1}_{T \notin \mathcal T_\ell}(\hat{\mathbf{y}}_\ell)_T$
        \ENDFOR
        \STATE $\hat{\phi}_\ell \gets \mathbf{Q}_\ell \hat{\mathbf{y}}^-_\ell +  \textsc{\texttt{SolveWLS}}(\hat{\mathbf{X}}_\ell,\hat{\mathbf{y}}^+_\ell,\hat{\mathbf{W}}^*_\ell)$
        \ENDIF
        \STATE $\hat{\mathbf{y}}_{\ell+1} \gets \hat{\mathbf{y}}_{\ell} - \hat{\mathbf{X}}_\ell \cdot \hat\phi_{\ell}$ \COMMENT{compute residuals}
    \ENDFOR
    \STATE $\hat\Phi_k \gets \textsc{\texttt{AggregateSII}}(\hat\phi_1,\dots,\hat\phi_k)$ \COMMENT{compute $k$-SII} 
    \STATE \textbf{return} $k$-SII estimates $\hat\Phi_k$, SII estimates $\hat{\phi}_{1},\dots,\hat\phi_k$
    \end{algorithmic}
\end{algorithm}

\textbf{KernelSHAP-IQ.}
We now introduce KernelSHAP-IQ based on the novel representation in \cref{thm_kernel_sii}, which ensures that KernelSHAP-IQ converges to SII for $k=2$.
Further, with Conjecture~\ref{conj_sii}, KernelSHAP-IQ can be extended to $k>2$, where we empirically validate its convergence.
KernelSHAP-IQ once samples subsets given the sampling weights $q$, where the weights $w$ account for the sampling probabilities and the number of Monte Carlo samples (line~1).
As a default, we propose the KernelSHAP weights for sampling, i.e $p(T) \propto \mu_1(t)$, and apply the \emph{border-trick} \cite{Fumagalli.2023}.
For further details regarding the sampling and the weights $w$, we refer to \cref{appx_sec_sampling}.
The game is then evaluated on all sampled subsets, stored in $\hat{\mathbf{y}}_1$, which determines the computational complexity (line~2).
Then, $\hat\nu_\ell$ is iteratively constructed for $\ell=1,\dots,k$ by computing the SII estimates (line~3-19).
Starting from $\ell=1$, where KernelSHAP-IQ reduces to KernelSHAP, $\hat\phi_\ell$ is estimated and the residual $\hat{\mathbf{y}}_{\ell+1} := \hat{\mathbf{y}}_{\ell} - \mathbf{X}_\ell \cdot \hat\phi_\ell$ is computed for the next iteration (line~18).
We repeat this process until $\ell=k$, summarized in \cref{alg_kernelshap_iq}.
At each step, we set the matrices for the WLS problem (line~4-7).
The coefficient matrix $\hat{\mathbf{X}}_\ell$ is set for every sampled subset and every interaction of the current order $\ell$ (line~5).
The diagonal weight matrix $\hat{\mathbf{W}}^*_\ell$ is set for every sampled subset, and contains $\mu_\ell$ adjusted by $w_T$, which accounts for the sampling probabilities $p^*(T)$ and the number of Monte Carlo samples (line~6).
If $\ell \leq 2$, then the WLS problem is solved directly with $\hat{\mathbf{X}}_\ell,\hat{\mathbf{y}}_\ell$, and $\hat{\mathbf{W}}^*_\ell$ (line~9).
If $\ell > 2$, then Conjecture~\ref{conj_sii} applies, and the WLS problem is split by $\mathcal{T}_\ell$ (line~11-15).
For every sampled subset and every interaction of the current order $\ell$, the SII weights are assigned to $\mathbf{Q}_{\ell}$ (line~12), and the residuals $\hat{\mathbf{y}}_{\ell}$ are split by $\mathcal{T}_k$ into $\hat{\mathbf{y}}^+_{\ell}$ (line~13) and $\hat{\mathbf{y}}^-_{\ell}$ (line~14).
The SII estimates $\hat\phi_\ell$ are then computed with the SII weights $\mathbf{Q}_\ell$ for $\hat{\mathbf{y}}_\ell^-$ and by solving the WLS problem with $\hat{\mathbf{X}}_\ell,\hat{\mathbf{y}}^+_\ell$ and $\hat{\mathbf{W}}^*_\ell$ (line~16). 
After reaching $\ell=k$, the SII estimates are aggregated via the non-recursive formula (cf.  \cref{appx_sec_ksii_weights}) to $k$-SII values, and used for final interpretation (line 20-21).

\section{Experiments}
\label{sec_experiments}

We conduct multiple experiments to compare KernelSHAP-IQ with existing baseline methods for estimating SII and $k$-SII values.
For each method, we assess estimation quality with \emph{mean-squared error} (MSE; lower is better) and \emph{precision at ten} (Prec@10; higher is better) compared to ground truth (GT) SII and $k$-SII.
We compute confidence bands with the standard error of the mean (SEM).
Prec@10 measures the accuracy of correctly identifying the ten highest interaction scores in terms of absolute values.
The GT values are calculated once via brute force by evaluating $2^{n}$ coalitions for each game of $n$ players (i.e.\ features).

\textbf{Baselines.}
We compare KernelSHAP-IQ and inconsistent KernelSHAP-IQ with all available baseline algorithms; permutation sampling \cite{Castro.2009,Tsai.2022}, SHAP-IQ \cite{Fumagalli.2023}, and SVARM-IQ \cite{Kolpaczki.2024b}, as shown in \cref{tab_approx_landscape}.

\begin{table}[t]
    \caption{Summary of the benchmark datasets and models used.}\label{tab_setup}
    \vspace{0.5em}
    \adjustbox{max width=\columnwidth}{%
    \begin{tabular}{ccccc}
    \toprule
    \textbf{ID} & \textbf{Model} & \textbf{Removal Strategy} & $\fnum$ & $\mathcal{Y}$ \\ \midrule
    \rowcolor{tablegray}\emph{SOUM} & Synthetic & -- & 20, 40 & $[0, 1]$  \\
    \emph{LM} & DistilBert & Token Removal & 14 & $[-1, 1]$  \\
    \rowcolor{tablegray}\emph{CH} & Neural Net & Mean & 8 & $\mathbb{R}$ \\
    \emph{BR} & XGBoost & Mean/Mode & 12 & $\mathbb{R}$ \\
    \rowcolor{tablegray}\emph{ViT} & ViT-32-384 & Token Removal & 16 & $[0, 1]$ \\
    \emph{CNN} & ResNet18 & Superpixel & 14 & $[0, 1]$ \\
    \rowcolor{tablegray}\emph{AC} & RF & Mean/Mode & 14 & $[0, 1]$ \\
    \bottomrule
    \end{tabular}}
\end{table}

\begin{figure*}[t]
    \centering
    \includegraphics[width=0.75\textwidth]{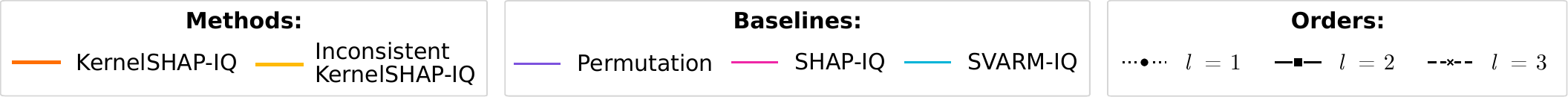}
    \\
    \begin{minipage}[c]{0.24\textwidth}
    \includegraphics[width=\textwidth]{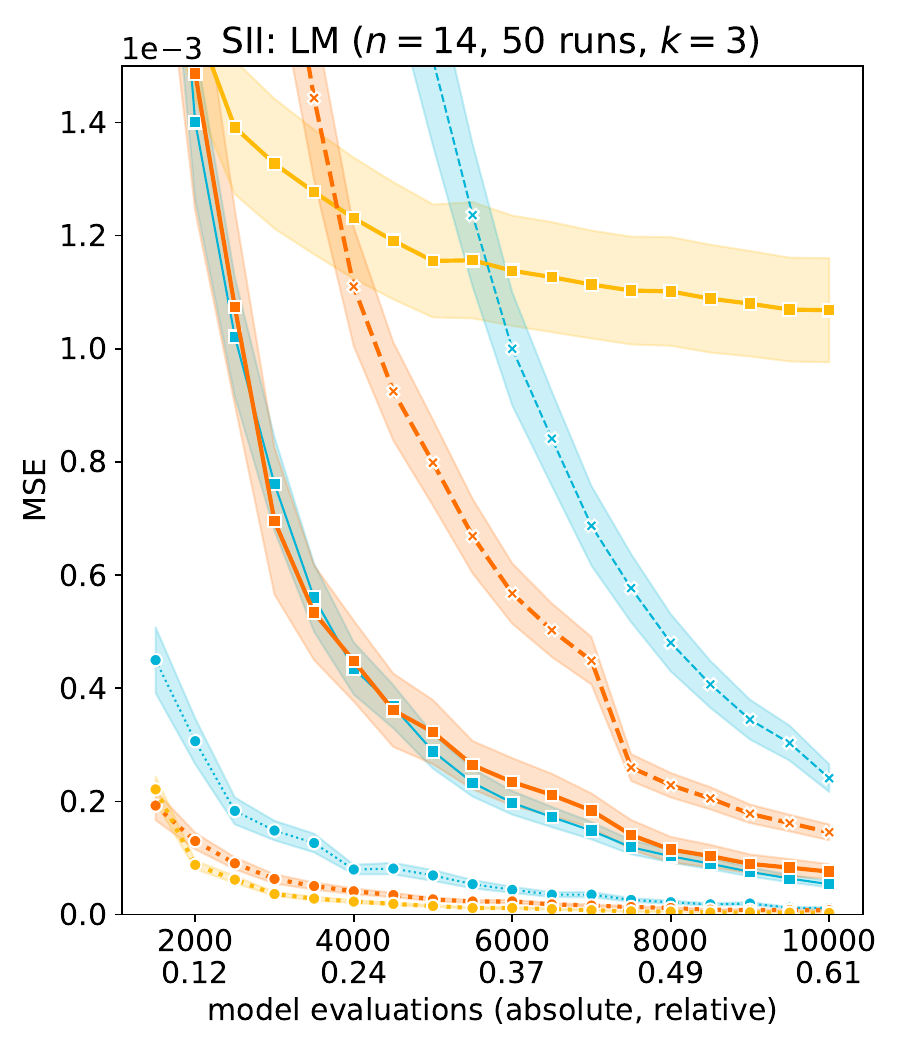}
    \end{minipage}
    \hfill
    \begin{minipage}[c]{0.24\textwidth}
    \includegraphics[width=\textwidth]{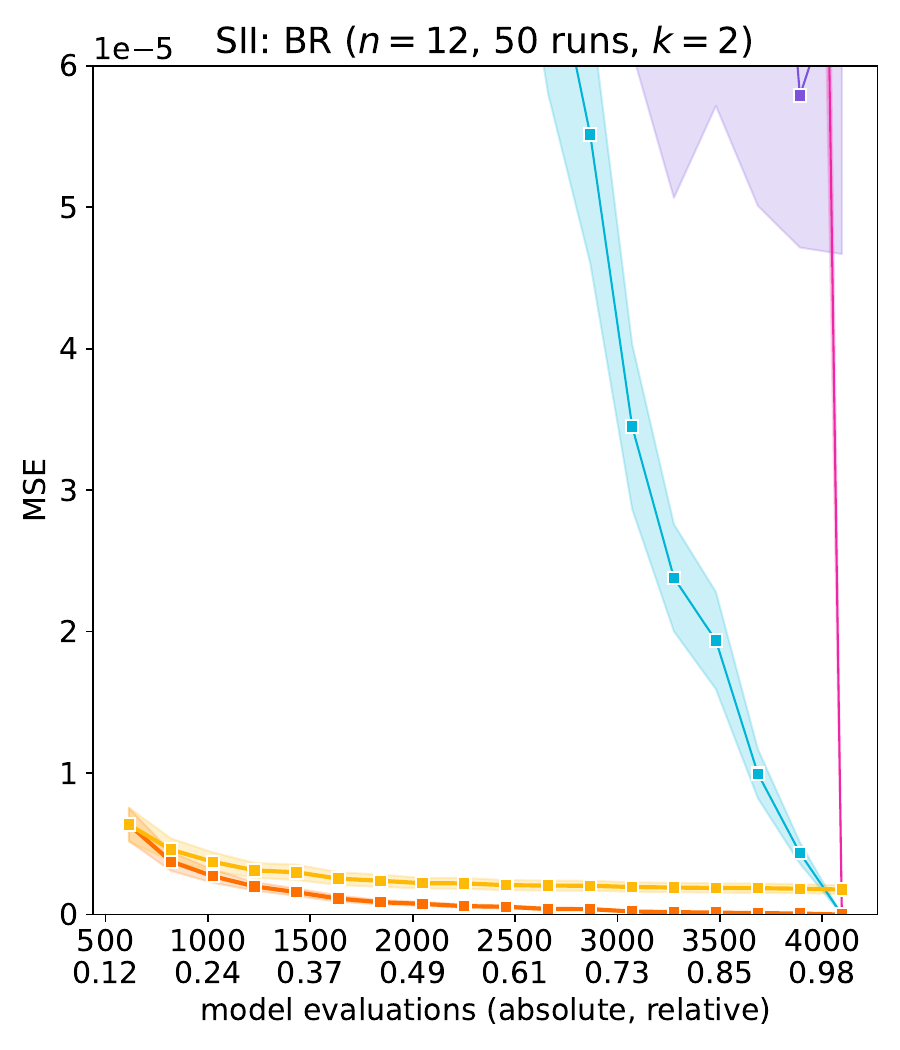}
    \end{minipage}
    \hfill
    \begin{minipage}[c]{0.24\textwidth}
    \includegraphics[width=\textwidth]{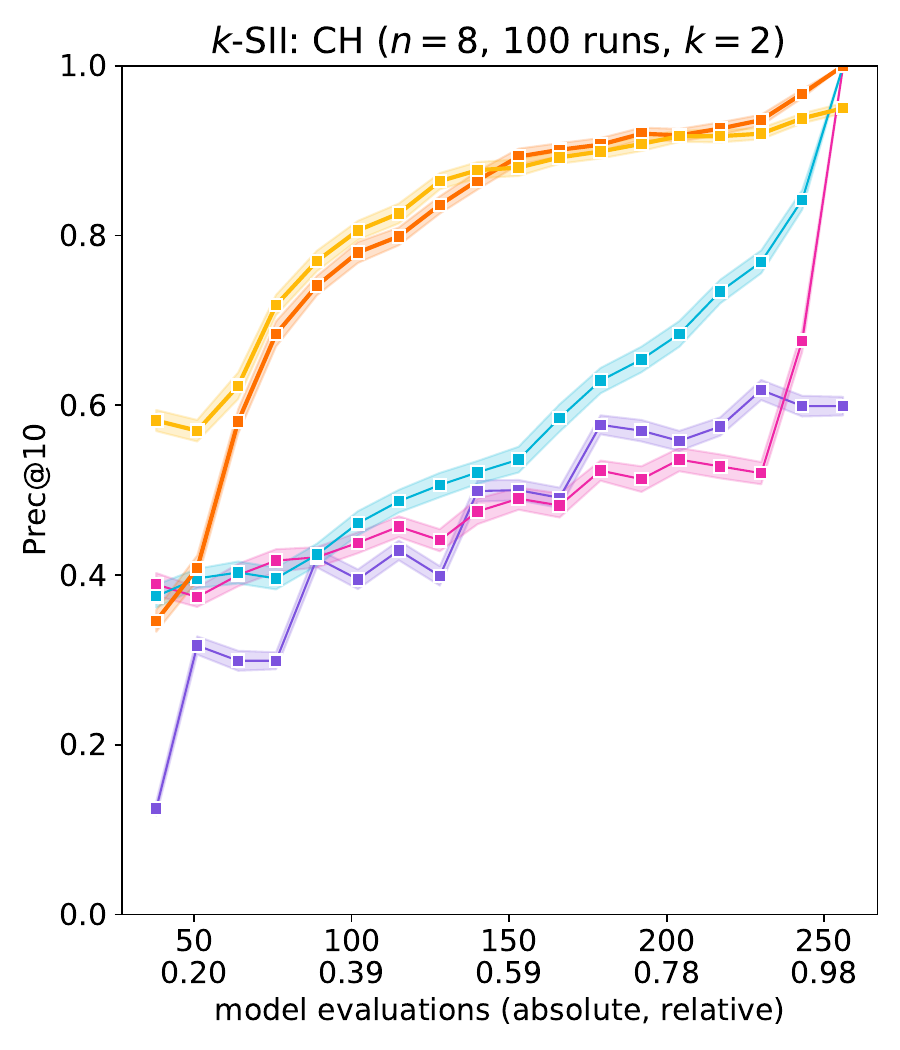}
    \end{minipage}
    \hfill
    \begin{minipage}[c]{0.24\textwidth}
    \includegraphics[width=\textwidth]{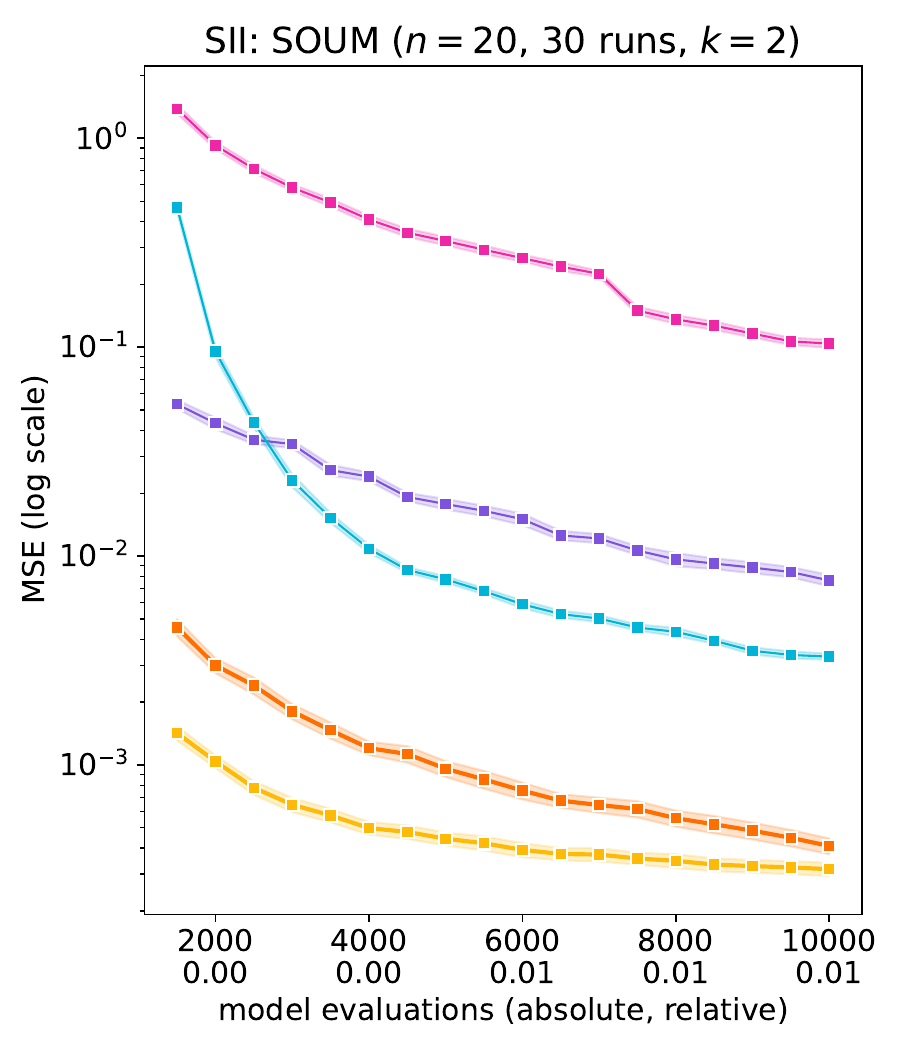}
    \end{minipage}
    \caption{Approximation quality of KernelSHAP-IQ (\textcolor{kernelshapiq}{orange}) and inconsistent KernelSHAP-IQ (\textcolor{biaskernelshapiq}{yellow}) compared to the permutation sampling (\textcolor{permutation}{purple}), SHAP-IQ (\textcolor{shapiq}{pink}) and SVARM-IQ (\textcolor{svarmiq}{blue}) baselines for estimating SII values for the LM (left; $\fnum=14,l\in\{1,2,3\}$) the \emph{bike rental} dataset (center left, $\fnum=12,l=2$), the \emph{california housing} dataset (center right; $\fnum=8,l=2$), and the SOUM (right; $\fnum=20,l=2$). The shaded bands represent the standard error of the mean (SEM).}
    \label{fig_approx_quality}
\end{figure*}

\textbf{Benchmark Datasets and Models.}
Based on recent work by \citet{Fumagalli.2023,Kolpaczki.2024b,Tsai.2022}, we create different benchmark scenarios\footnote{{\repository}}.
\cref{tab_setup} summarizes the scenarios and corresponding removal approaches after \citet{Covert.2021b}.
For a detailed description regarding the experimental setup and model training we refer to \cref{sec_appendix_experimental_setup,sec_appendix_additional_results} including a runtime analysis. 
First, we create synthetic \emph{sum of unanimity models} (SOUMs), also known as sum of unanimity games and induced subgraph game \cite{Deng.1994}, where GT SII can be computed, cf. \cref{appx_sii_gt}.
Second, we compute feature (i.e.\ token) interaction values for a sentiment analysis \emph{language model} (LM).
The LM is a fine-tuned version of \texttt{DistilBert} \cite{Sanh.2019} on movie review excerpts from the IMDB dataset \cite{Maas.2011}.
Third and forth, we explain local instances of the \emph{bike rental} (BR) \cite{FanaeeT.2014} and \emph{california housing} (CH) \cite{Kelley.1997} regression datasets.
On BR we train an XGBoost regressor and on CH a small neural network. 
The target variables are logarithmized.
Fifth and sixth, we explain a \texttt{ResNet18} \emph{convolutional neural network} (CNN) \cite{He.2018} and a \emph{vision transformer} (ViT) image classifiers, which were pre-trained on ImageNet \cite{Deng.2009}. 
The ViT considers patches of $32 \times 32$ pixels.
In the CNN, individual pixels are grouped together into super-pixels.
Seventh, we compute feature interactions for instances of the \emph{adult census} (AC) \cite{Kohavi.1996} classification dataset and a random forest (RF) classifier. 

\textbf{Approximation Quality of SII and $k$-SII.} 
\cref{fig_approx_quality} depicts the approximation quality of KernelSHAP-IQ and inconsistent KernelSHAP-IQ compared to the baseline estimators on a selection of benchmark tasks. 
Further results are provided in \cref{sec_appendix_additional_results}.
KernelSHAP-IQ and inconsistent KernelSHAP-IQ clearly outperform all three sampling-based baselines on both regression benchmarks (BR and CH).
Inconsistent KernelSHAP-IQ achieves high estimation qualities in low-budget scenarios.
This trend also materializes for second order SII values, in particular in high-player tasks (e.g.\ SOUM in \cref{fig_approx_quality}).
Yet, unlike KernelSHAP-IQ, the inconsistent version does not converge to the GT SII values. 
Interestingly, we observe that the inconsistent version does converge to the GT SV independent of the approximation order $k \geq 1$.
Notably, in benchmark tasks like the LM, KernelSHAP-IQ and SVARM-IQ rapidly outperform all baselines as well as inconsistent KernelSHAP-IQ. 
Both SVARM-IQ and KernelSHAP-IQ achieve SOTA results for the LM task and second order SII values.
By relying on \cref{conj_sii}, KernelSHAP-IQ yields high-quality estimations for SII of higher orders ($k > 2$) and outperforms SVARM-IQ for order $3$ on the LM.

\textbf{Example Use Case: Feature Attribution.}
As illustrated in \cref{fig_language_illustration,fig_intro_illustration}, 2-SII values can be used to enhance current feature attribution techniques.
Similar to \citet{Fumagalli.2023,Sundararajan.2020,Tsai.2022}, we show that feature interactions are relevant for understanding intricate LMs. 
In \cref{fig_language_illustration}, 2-SII scores reveal that the predicted positive sentiment largely stems from the interaction of the two words ``never'' and ``forget'', while ``forget'' individually points towards a negative sentiment. 
In the CH example illustrated in \cref{fig_intro_illustration}, the \emph{longit.} and \emph{latit.} features are both contributing positively to the prediction considering 1-SII (SV). 
The 2-SII explanation, however, reveals that the positive contribution of \emph{latit.} vanishes (very low \emph{latit.} score) and can be mostly attributed to the \emph{latit. x longit.} interaction. 
Hence, the \emph{exact location} of the property is meaningful. 

\section{Limitations}
\label{sec_limitations}
We linked the SII to a solution of a WSL optimization problem.
As previously discussed, our theoretical results extend to higher orders, provided that $2k \leq \fnum$, although we were unable to give a rigorous proof.
We expect that other proof techniques are required to further understand these coherences, which could resolve improper weighting in Conjecture~\ref{conj_sii}.
In practice, due to iterative computation of KernelSHAP-IQ, higher-order estimates are negatively affected by previous low-quality estimates, where interpretation may be flawed.
Lastly, the exponentially growing number of interactions requires human-centered post processing, and we give modest suggestions for visualizations.

\section{Conclusion and Future Work}
\label{sec_conclusion}
In this work, we clarified the link between SII, an axiomatic interaction index, and $k$-SII, an aggregation of SII that yields a $k$-additive interaction index used for interpretability.
We demonstrated that the approximation of any game induced by $k$-SII is iteratively constructed via SII.
Similar to the SV, we then established that SII of order $k$ is represented as the solution to a WLS problem, where SII yields an \emph{optimal} $k$-additive approximation.
We rigorously prove our results for the SV and pairwise SII and give empirically validated conjectures for higher orders.
Consequently, we introduce KernelSHAP-IQ, a direct extension of KernelSHAP, which efficiently approximates SII for higher orders and yields SOTA performance.
We apply KernelSHAP-IQ for local interpretability and demonstrate benefits of enriching feature attributions with interactions.

In future work, we suspect that a rigorous proof of our conjectures would reveal novel techniques and further insights.
Another interesting link are orthogonal projections, where we suspect a link of pairwise SII to the \emph{Shapley residual} \cite{Kumar.2021}.
Our novel representation of SII opens up additional possibilities, including the investigation of \emph{amortized} Shapley interactions through methods akin to FastSHAP \cite{Jethani.2022}.
Apart from local interpretability with feature interactions, KernelSHAP-IQ is applicable in any game-theoretic setting, which includes global interpretability or data valuation.

\clearpage

\section*{Acknowledgment}
We gratefully thank the anonymous reviewer for their valuable feedback for improving this work.
Fabian Fumagalli and Maximilian Muschalik gratefully acknowledge funding by the Deutsche Forschungsgemeinschaft (DFG, German Research Foundation): TRR 318/1 2021 – 438445824. Patrick Kolpaczki was supported supported by the research training group Dataninja (Trustworthy AI for Seamless Problem Solving: Next Generation Intelligence Joins Robust Data Analysis) funded by the German federal state of North Rhine-Westphalia.

\section*{Impact Statement}
This paper presents work whose goal is to advance the field of Machine Learning, specifically the field of Explainable Artificial Intelligence.
There are many potential societal consequences of our work.
Our work can positively impact Machine Learning adoption and potentially reveal biases or unwanted behavior in Machine Learning systems.
However, it may also enable explainability-based white-washing in organizations, firms, or policing.

\bibliography{references}
\bibliographystyle{icml2024}

%%%%%%%%%%%%%%%%%%%%%%%%%%%%%%%%%%%%%%%%%%%%%%%%%%%%%%%%%%%%%%%%%%%%%%%%%%%%%%%
%%%%%%%%%%%%%%%%%%%%%%%%%%%%%%%%%%%%%%%%%%%%%%%%%%%%%%%%%%%%%%%%%%%%%%%%%%%%%%%
% APPENDIX
%%%%%%%%%%%%%%%%%%%%%%%%%%%%%%%%%%%%%%%%%%%%%%%%%%%%%%%%%%%%%%%%%%%%%%%%%%%%%%%
%%%%%%%%%%%%%%%%%%%%%%%%%%%%%%%%%%%%%%%%%%%%%%%%%%%%%%%%%%%%%%%%%%%%%%%%%%%%%%%

\newpage
\appendix
\onecolumn

\section*{Organisation of the Supplementary Material}

The supplementary material is organized as follows. \cref{appx_proofs} contains all proofs. \cref{appx_sec_algos} contains algorithmic details and ground truth computation procedures. \cref{sec_appendix_experimental_setup} contains details regarding the experimental setup and reproducibility. Lastly, \cref{sec_appendix_additional_results} contains additional experimental results. 

% A
\contentsline {section}{\numberline{A}Proofs}{15}{}
\contentsline {subsection}{\numberline{A.1}Proof of Proposition 3.2}{15}{}
\contentsline {subsection}{\numberline{A.2}Proof of Corollary 3.3}{15}{}
\contentsline {subsection}{\numberline{A.3}Proof of Corollary 3.4}{15}{}
\contentsline {subsection}{\numberline{A.4}Proof of Theorem 3.6}{16}{}
\contentsline {subsection}{\numberline{A.5}Proof of Theorem 3.7}{17}{}
% B
\contentsline {section}{\numberline{B}Algorithms and Ground Truth Calculations}{23}{}
\contentsline {subsection}{\numberline{B.1}KernelSHAP-IQ Default Parameters}{23}{}
\contentsline {subsection}{\numberline{B.2}Inconsistent KernelSHAP-IQ}{23}{}
\contentsline {subsection}{\numberline{B.3}Sampling Algorithm}{23}{}
\contentsline {subsection}{\numberline{B.4}Solving WLS}{24}{}
\contentsline {subsection}{\numberline{B.5}Subset Weights SII}{25}{}
\contentsline {subsection}{\numberline{B.6}Aggregate SII to k-SII}{25}{}
\contentsline {subsection}{\numberline{B.7}Baseline Methods: SHAP-IQ, Permutation Sampling and SVARM-IQ}{26}{}
\contentsline {subsection}{\numberline{B.8}Analytic Solution for SII of SOUMs}{26}{}
\contentsline {subsection}{\numberline{B.9}Intuition about the KernelSHAP-IQ Weights}{27}{}
% C
\contentsline {section}{\numberline{C}Experimental Setup and Reproducibility}{28}{}
\contentsline {subsection}{\numberline{C.1}Model, Datasets and Task Descriptions}{28}{}
\contentsline {subsection}{\numberline{C.2}Computational Effort}{29}{}
% D
\contentsline {section}{\numberline{D} Additional Empirical Results}{30}{}
\contentsline {subsection}{\numberline{D.1}Runtime Analysis}{30}{}
\contentsline {subsection}{\numberline{D.2}Validations of Higher-Order Conjecture}{30}{}
\contentsline {subsection}{\numberline{D.3}Additional Approximation Results}{31}{}

\clearpage

\section{Proofs}\label{appx_proofs}

\subsection{Proof of Proposition 3.2}
\begin{proof}
By Definition~\ref{def_k_order}, we have $\hat\nu_k(T) := \sum_{S \subseteq T}^{1\leq \vert S \vert \leq k}\Phi_{k}(S)$.
According to Appendix A.2 in \cite{Bord.2023}, $\Phi_k(S)$ is explicitly computed as
\begin{equation*}
    \Phi_k(S) = \sum_{S \subseteq \tilde S \subseteq \fset}^{\vert \tilde S \vert \leq k} B_{\vert \tilde S \vert- \vert S \vert} \phi^{\text{SII}}(\tilde S) \text{ for } 1 \leq \vert S \vert \leq k.
\end{equation*}
Hence, it follows by re-arranging terms
\begin{align*}
    \hat\nu_k(T) &= \sum_{S \subseteq T}^{1 \leq \vert S \vert \leq k} \Phi_{k}(S)
    = \sum_{S \subseteq T}^{1\leq \vert S \vert \leq k} \sum_{S \subseteq \tilde S \subseteq \fset}^{\vert \tilde S \vert \leq k} B_{\vert \tilde S \vert- \vert S \vert} \phi^{\text{SII}}(\tilde S)
    = \sum_{S \subseteq T}^{1\leq \vert S \vert}  \sum_{\tilde S \subseteq \fset}^{\vert \tilde S \vert \leq k} B_{\vert \tilde S \vert- \vert S \vert} \phi^{\text{SII}}(\tilde S) \mathbf{1}(S \subseteq \tilde S)
    \\
    &= \sum_{\tilde S \subseteq \fset}^{\vert \tilde S \vert \leq k} \phi^{\text{SII}}(\tilde S)  \sum_{S \subseteq T}^{1\leq \vert S \vert}  B_{\vert \tilde S \vert- \vert S \vert} \mathbf{1}(S \subseteq \tilde S) 
    = \sum_{\tilde S \subseteq \fset}^{\vert \tilde S \vert \leq k} \phi^{\text{SII}}(\tilde S)  \sum_{r = 1}^{\vert \tilde S \vert} B_{\vert \tilde S \vert- r} \sum_{S \subseteq \fset}^{\vert S \vert = r}  \mathbf{1}(S \subseteq \tilde S \cap T)  
    \\
    &= \sum_{\tilde S \subseteq \fset}^{\vert \tilde S \vert \leq k} \phi^{\text{SII}}(\tilde S)  \sum_{r = 1}^{\vert \tilde S \cap T \vert} B_{\vert \tilde S \vert-r} \binom{\vert \tilde S \cap T\vert}{r} =  \sum_{\tilde S \subseteq \fset}^{\vert \tilde S \vert \leq k} \phi^{\text{SII}}(\tilde S)  \lambda(\vert \tilde S\vert,\vert \tilde S \cap T \vert),
\end{align*}
where we introduced $\binom{0}{r} := 0$ for $r>0$ and $\lambda(k,\ell) := \sum_{r=1}^{\ell}\binom{\ell}{r} B_{k-r}$.
The result follows then immediately by observing that the terms do not depend on $k$ and separating above sum into terms of order $k$ and order less than $k$ as

\begin{align*}
    \hat\nu_k(T) = \hat\nu_{k-1}(T) + \sum_{\tilde S \subseteq \fset}^{\vert \tilde S \vert = k} \phi^{\text{SII}}(\tilde S)  \lambda(\vert \tilde S \vert,\vert \tilde S \cap T \vert). 
\end{align*}
\end{proof}

\subsection{Proof of Corollary 3.3}
\begin{proof}
    From Proposition~\ref{prop_iterative_approx} it suffices to compute $\lambda(k,\ell) := \sum_{r=1}^{\ell}\binom{\ell}{r} B_{k-r}$.
    Clearly, $\lambda(1,1)=B_0=1$, which yields $\hat\nu_1(T) = \sum_{i\in T} \phi^{\text{SII}}(i) = \sum_{i\in T} \phi^{\text{SV}}(i)$. 
    For $k>1$, we have by the symmetry of the binomial coefficient $\lambda(k,k) = \sum_{r=1}^{k}\binom{k}{r} B_{k-r} =  \sum_{r=0}^{k-1}\binom{k}{r} B_{k} = 0$, which is the defining property of the Bernoulli numbers.
    Hence, for $k=2$, $\lambda(2,0)=\lambda(2,2)=0$.
    Lastly, $\lambda(2,1)=B_1=-1/2$, which yields
\begin{align*}
    \hat\nu_2(T) = \hat\nu_{1}(T) -\frac 1 2 \sum_{ij \subseteq \fset}^{\vert ij \cap T \vert = 1}\phi^{\text{SII}}(ij). 
\end{align*}
    
\end{proof}

\subsection{Proof of Corollary 3.4}
\begin{proof}
    This proof follows immediately from the fact that $\hat\nu_\fnum(T)$ is constructed from $k$-SII values $\Phi_k$, which are the Möbius transform for $k=\fnum$, cf. Theorem 4 by \citet{Bord.2023}.
    A defining property of the Möbius transform \cite{Grabisch.2016} is
    \begin{equation*}
        \nu(T) = \sum_{S \subseteq T} \Phi_\fnum(S),
    \end{equation*}
    and thus, by \cref{def_k_order}, $\nu(T) = \hat\nu_n(T)$.
\end{proof}

\subsection{Proof of Theorem 3.6}
\begin{proof}
The solution of the optimization problem is explicitly given as
\begin{equation}
\phi^* = (\mathbf{X}^T_1 \mathbf{W}_1 \mathbf{X}_1)^{-1} \mathbf{X}^T_1 \mathbf{W}_1 \cdot \mathbf{y}_1.
\end{equation}
Our goal is show that for $\mu_\infty \to \infty$ we obtain the correct weight of the SV for $\left((\mathbf{X}^T_1 \mathbf{W}_1 \mathbf{X}_1)^{-1} \mathbf{X}^T_1 \mathbf{W}_1\right)_{iT}$.
We will show that those weights are equal to the representation given in Theorem 4.4 by \citet{Fumagalli.2023} as
\begin{equation*}
\phi^{\text{SV}}(i) = \frac{1}{\fnum}\nu(\fset) +  \sum_{T \subseteq \fset}^{1\leq \vert T \vert \leq \fnum-1}\nu(T)\frac{\mu_1(t)}{\fnum-1}\left[ \mathbf{1}(i \in T) - \frac{t}{\fnum}\right].
\end{equation*}
Note that the values $\nu(T)$ are encoded in $\mathbf{y}_1$.
The proof is structured as follows:
\begin{itemize}
    \item Find structure of $\mathbf{X}_1^T \mathbf{W}_1 \mathbf{X}_1$.
    \item Find inverse $(\mathbf{X}_1^T \mathbf{W}_1 \mathbf{X}_1)^{-1}$.
    \item Compute $\lim_{\mu_\infty\to\infty}(\mathbf{X}_1^T \mathbf{W}_1 \mathbf{X}_1)^{-1} \mathbf{X}_1^T \mathbf{W}_1$ 
    \begin{itemize} 
        \item for $T$ with finite weight, i.e.\ $1\leq \vert T \vert \leq \fnum-1$.
        \item for $T$ with infinite weight $\mu_1(t) = \mu_\infty$, i.e.\ $T \in \{\emptyset,\fset\}$.
    \end{itemize}
\end{itemize}
We proceed by computing the terms in $\phi^*$ separately.
First, with $(\mathbf{X}_1)_{Ti} = \mathbf{1}(i \in T)$
\begin{align*}
    (\mathbf{X}^T_1 \mathbf{W}_1 \mathbf{X}_1)_{i,j} = \sum_{T \subseteq \fset} \mu_1(\vert T \vert)\mathbf{1}(i \in T) \mathbf{1}(j \in T) = \mathbf{1}(i=j) \underbrace{\sum_{t=1}^{\fnum} \mu_1(t)\binom{\fnum-1}{t-1}}_{=: \mu_{1,1}} + \mathbf{1}(i \neq j)\underbrace{\sum_{t=2}^\fnum \mu_1(t) \binom{\fnum-2}{t-2}}_{=: \mu_{1,0}}.
\end{align*}
Introducing a matrix $(\mathbf{J})_{i,j} \equiv 1$ of all ones and the identity $\mathbf{I}$, we have
$(\mathbf{X}^T_1 \mathbf{W}_1 \mathbf{X}_1) = \mu_{1,0}\mathbf{J}+(\mu_{1,1}-\mu_{1,0})\mathbf{I}$.
Due to this simplistic structure, we can compute the inverse explicitly using the following Lemma.
\begin{lemma}[\citet{Fumagalli.2023}]\label{lem_inverse_sv}
Let $\eta_0,\eta_1 > 0$ with $\eta_0\neq \eta_1$. Then, $(\eta_0 \mathbf{J} + (\eta_1-\eta_0)\mathbf{I})^{-1} = \tilde \eta_0 \mathbf{J} + (\tilde\eta_1-\tilde\eta_0)\mathbf{I}$
with
\begin{align*}
    &\tilde \eta_0  = \frac{-\eta_0}{(\eta_1-\eta_0)(\eta_1+(\fnum-1)\eta_0)} 
    & \tilde\eta_1= \frac{\eta_1+ (\fnum-2) \eta_0}{(\eta_1-\eta_0)(\eta_1+(\fnum-1)\eta_0)}.
\end{align*}
\end{lemma}
Hence,
\begin{align*}
(\mathbf{X}^T_1 \mathbf{W}_1 \mathbf{X}_1)^{-1} = (\mu_{1,0}\mathbf{J}+(\mu_{1,1}-\mu_{1,0})\mathbf{I})^{-1} = \tilde \mu_{1,0} \mathbf{J} + (\tilde\mu_{1,1}-\tilde\mu_{1,0})\mathbf{I}.
\end{align*}
By definition of $\mu_1$, we have
\begin{align*}
    \mu_{1,1} = \mu_\infty + \sum_{t=1}^{\fnum-1}\binom{\fnum-2}{t-1}^{-1} \binom{\fnum-1}{t-1} = \mu_\infty + \sum_{t=1}^{\fnum-1} \frac{\fnum-1}{\fnum-t} = \mu_\infty + (\fnum-1)h_{\fnum-1} = \mu_\infty + (\fnum-1)h_{\fnum-2} + 1,
\end{align*}
where $h_n := \sum_{k=1}^{n} 1/k$ is the harmonic number.
Further,
\begin{align*}
    \mu_{1,0} = \mu_\infty + \sum_{t=2}^{\fnum-1}\binom{\fnum-2}{t-1}^{-1} \binom{\fnum-2}{t-2} = \mu_\infty + \sum_{t=2}^{\fnum-1} \frac{t-1}{\fnum-t} =  \mu_\infty + \sum_{t=1}^{\fnum-2} \frac{\fnum-t-1}{t} = \mu_\infty + (\fnum-1) h_{\fnum-2}-(\fnum-2).
\end{align*}
Hence, $\mu_{1,1}-\mu_{1,0} = \fnum-1$.
Note that if $\mu_\infty \to \infty$, then $\tilde\mu_{1,0} \overset{\mu_\infty\to\infty}{\longrightarrow} \frac{1}{\fnum(\fnum-1)}$ and $\tilde\mu_{1,1} \overset{\mu_\infty\to\infty}{\longrightarrow} \frac 1 \fnum$, which proves the special case for $k=1$ of Conjecture~\ref{conj_inverse}.
We are now able to compute the weights
\begin{align*}
    (\mathbf{X}^T_1 \mathbf{W}_1 \mathbf{X}_1)^{-1}\mathbf{X}_1^T\mathbf{W}_1 = \tilde \mu_{1,0} \mathbf{J}\cdot\mathbf{X}_1^T\mathbf{W}_1 + (\tilde\mu_{1,1}-\tilde\mu_{1,0})\mathbf{I}\cdot \mathbf{X}_1^T\mathbf{W}_1.
\end{align*}
By Lemma~\ref{lem_inverse_sv}, $\tilde\mu_{1,1}-\tilde\mu_{1,0} = \frac{1}{\mu_{1,1}-\mu_{1,0}} = \frac{1}{\fnum-1}$.
With $(\mathbf{J} \mathbf{X}_1^T \mathbf{W}_1)_{ST} = \vert T \vert \mu_1(t)$ it follows
\begin{align*}
    &\left(\tilde \mu_{1,0} \mathbf{J}\cdot\mathbf{X}_1^T\mathbf{W}_1 + (\tilde\mu_{1,1}-\tilde\mu_{1,0})\mathbf{I}\cdot \mathbf{X}_1^T\mathbf{W}_1\right)_{iT} = \tilde\mu_{1,0} \vert T \vert \mu_1(t) + \frac{1}{\fnum-1}\mathbf{1}(i \in T) \mu_1(t)
    \\
    &= \frac{1}{\fnum-1} \mu_1(t) \left(\frac{-\mu_{1,0}\vert T \vert+ \mathbf{1}(i\in T)(\mu_{1,1}+(\fnum-1)\mu_{1,0})}{\mu_{1,1}+(\fnum-1)\mu_{1,0}} \right).
\end{align*}
Now, for $T$ with $ \mu_1(t) \neq \mu_\infty$, i.e.\ $1\leq \vert T \vert \leq \fnum-1$, we have that $\mu_1(t) \overset{\mu_\infty\to\infty}{\longrightarrow} \mu_1(t)$ and thus compute
\begin{align*}
     \left((\mathbf{X}^T_1 \mathbf{W}_1 \mathbf{X}_1)^{-1}\mathbf{X}_1^T\mathbf{W}_1\right)_{iT} \overset{\mu_\infty \to \infty}{\longrightarrow} \frac{1}{\fnum-1}\mu_1(t) \left(\mathbf{1}(i \in T)-\frac{\vert T \vert}{\fnum}\right),
\end{align*}
where we used that $\mu_\infty$ appears in $\mu_{1,0}$ and $\mu_{1,1}$.
Clearly, this directly yields the weight in the representation of Theorem 4.4 in \cite{Fumagalli.2023}.
For the cases, where $\mu_1(t)=\mu_\infty$, i.e.\ $T \in \{\emptyset,\fset\}$, we have 
zero weight for $\emptyset$, as $\vert T \vert =0$ and $\mathbf{1}(i\in T)=0$ for all $i \in \fset$.
For $T=\fset$, we have $\vert T \vert = \fnum$ and $\mathbf{1}(i\in T) =1$ for all $i \in \fset$.
Hence,
\begin{align*}
     \left((\mathbf{X}^T_1 \mathbf{W}_1 \mathbf{X}_1)^{-1}\mathbf{X}_1^T\mathbf{W}_1\right)_{i\fset} 
     = \frac{1}{\fnum-1}\mu_1(t) \left(\frac{\mu_{1,1}-\mu_{1,0}}{\mu_{1,1}+(\fnum-1)\mu_{1,0}} \right) 
    = \mu_1(t) \left(\frac{1}{\mu_{1,1}+(\fnum-1)\mu_{1,0}} \right) 
     \overset{\mu_\infty\to\infty}{\longrightarrow} \frac 1 \fnum,
\end{align*}
which is again the weight in Theorem 4.4 in \cite{Fumagalli.2023}.
\end{proof}

\subsection{Proof of Theorem 3.7}
\begin{proof}
Again, the solution of the optimization problem is explicitly given as
\begin{equation}
\phi^* = (\mathbf{X}^T_2 \mathbf{W}_2 \mathbf{X}_2)^{-1} \mathbf{X}^T_2 \mathbf{W}_2 \cdot \mathbf{y}_2.
\end{equation}
We show that for $\mu_\infty \to \infty$, we obtain the correct weight of the SII in $\left((\mathbf{X}^T_2 \mathbf{W}_2 \mathbf{X}_2)^{-1} \mathbf{X}^T_2 \mathbf{W}_2\right)_{ST}$ with $S=ij$.
We will show that those weights are equal to the representation given in Theorem 4.1 in \cite{Fumagalli.2023} with $\vert S \vert = 2$ as
\begin{equation}\label{eq_sii_weight}
    \phi^{\text{SII}}(S) =  \sum_{T \subseteq \fset}\nu(T) (-1)^{\vert S\vert-\vert T \cap S\vert} \frac{1}{\fnum-\vert S \vert +1}\binom{\fnum-\vert S \vert}{\vert T \vert -\vert T \cap S\vert}^{-1} =  \frac{1}{\fnum-1}\sum_{T \subseteq \fset}\nu(T) (-1)^{\vert T \cap S\vert}\binom{\fnum-2}{\vert T \vert-\vert T \cap S\vert}^{-1}.
\end{equation}
Note that for $T \in \{\emptyset,\fset\}$, we have $(\mathbf{y}_2)_T = \nu(T) - \hat\nu_1(T) = \nu(T)-\sum_{i \in T} \phi^{\text{SV}}(i) = 0$ and thus we do not have to consider these cases.
The proof is again structured as follows:
\begin{itemize}
    \item Find structure of $\mathbf{X}_2^T \mathbf{W}_2 \mathbf{X}_2$.
    \item Find conditions for inverse $(\mathbf{X}_2^T \mathbf{W}_2 \mathbf{X}_2)^{-1}$.
    \item Compute $\lim_{\mu_\infty\to\infty} (\mathbf{X}_2^T \mathbf{W}_2 \mathbf{X}_2)^{-1} \mathbf{X}_2^T \mathbf{W}_2$ 
    \begin{itemize} 
        \item for $T$ with finite weight, i.e.\ $2\leq \vert T \vert \leq \fnum-2$.
        \item for $T$ with infinite weight $\mu_2(t) = \mu_\infty$, i.e.\ $\vert T \vert = 1$ or $\vert T \vert = \fnum-1$.
    \end{itemize}
\end{itemize}

    Recall from Corollary~\ref{cor_nuhat} that $\lambda(2,0)=\lambda(2,2)=0$ and $\lambda(2,1)=-1/2$.
    Hence, 
    \begin{align*}
        \left(\mathbf{X}_2^T \mathbf{W}_2 \mathbf{X}_2\right)_{S_1S_2} = \frac 1 4\sum_{T\subseteq \fset}\mu_2(t) \mathbf{1}(\vert T\cap S_1\vert=1)\mathbf{1}(\vert T\cap S_2\vert = 1) =: \mu_{2,\vert S_1\cap S_2\vert}.
    \end{align*}
    For each $T$, the element may either be from $S_1\cap S_2$ or from $T\setminus S_1$ and $T\setminus S_2$, and thus
    \begin{align*}
        \mu_{2,0} &= \frac 1 4 \sum_{t=2}^{\fnum-2} \binom{\fnum-4}{t-2}^{-1} \binom{\fnum-4}{t-2}4 = \fnum-3 
        \\
        \mu_{2,1} &= \frac{1}{2} \mu_\infty + \frac{1}{4} \sum_{t=2}^{\fnum-2}\binom{\fnum-4}{t-2}^{-1}\left[\binom{\fnum-3}{t-1}+\binom{\fnum-3}{t-2}\right] 
        = \frac{1}{2} \mu_\infty + \frac{1}{4} \sum_{t=2}^{\fnum-2} \frac{(\fnum-2)(\fnum-3)}{(\fnum-t-1)(t-1)} 
        \\
        &= \frac{1}{2} \mu_\infty + \frac{\fnum-3}{4} \sum_{t=2}^{\fnum-2}\left(\frac{1}{t-1} + \frac{1}{\fnum-t-1}\right)
        = \frac{1}{2} \mu_\infty + \frac{\fnum-3}{2}h_{\fnum-3}.
        \\
        \mu_{2,2} &= \mu_\infty + \frac{1}{4} \sum_{t=2}^{\fnum-2}\binom{\fnum-4}{t-2}^{-1}\binom{\fnum-2}{t-1}2 = \mu_\infty +   (\fnum-3)h_{\fnum-3} = 2 \mu_{2,1},
    \end{align*}
    where for $\vert T \vert =1$ or $\vert T \vert = \fnum-1$ with weight $\mu_\infty$ no combination is found for $\vert S_1\cap S_2 \vert = \emptyset$, two for $\vert S_1\cap S_2\vert = 1$, i.e.\ $T= S_1\cap S_2$ and $T=\fset\setminus (S_1\cap S_2)$ and four for $S_1=S_2$, i.e.\ $T=i$ and $T=\fset\setminus i$ with $i \in S_1$.
    Due to this simplistic structure of $\mathbf{X}_2^T \mathbf{W}_2 \mathbf{X}_2$, we now introduce three square matrices indexed with all subsets of size $2$, i.e.\ $\binom{\fnum}{2}$ many.
    We introduce the matrix with all ones, $\mathbf{J}$, the identity $\mathbf{I}$ and the \emph{intersection matrix} $\mathbf{Q} := \mathbf{1}(S_1 \cap S_2 = \emptyset)$.
    Using these matrices, we can rewrite
    \begin{equation*}
        \mathbf{X}_2^T \mathbf{W}_2 \mathbf{X}_2 = \mu_{2,1} \mathbf{J} + (\mu_{2,0}-\mu_{2,1}) \mathbf{Q} + (\mu_{2,2}-\mu_{2,1}) \mathbf{I}.
    \end{equation*}
    We now first prove the following lemma.
    \begin{lemma}\label{lem_inverse_sii}
        Let $\eta_0,\eta_1,\eta_2>0$, $q_k := \binom{\fnum-4+k}{2}$ for $k=0,\dots,2$ and let square matrices $\mathbf{J},\mathbf{Q},\mathbf{I}$ indexed by all subsets of size $2$ of $\fset$. Then, $\left(\eta_1 \mathbf{J} + (\eta_0-\eta_1) \mathbf{Q} + (\eta_2-\eta_1) \mathbf{I}\right)^{-1} =  \tilde\eta_1 \mathbf{J} + (\tilde\eta_0-\tilde\eta_1) \mathbf{Q} + (\tilde\eta_2-\tilde\eta_1) \mathbf{I}$,  where
        \begin{align*}
            \tilde\eta_2 -\tilde\eta_1 &= \frac{(\eta_2-\eta_1)+(q_0-q_1)(\eta_0-\eta_1)}{(\eta_2-\eta_1)^2+(q_0-q_1)(\eta_2-\eta_1)(\eta_0-\eta_1)-(q_2-q_1)(\eta_0-\eta_1)^2},
            \\
            \tilde\eta_0 - \tilde\eta_1 &= - \frac{(\eta_0-\eta_1)(\tilde\eta_2-\tilde\eta_1)}{\eta_2-\eta_1 + (q_0-q_1)(\eta_0-\eta_1)}
            \\
            \tilde\eta_1 &= -\frac{(q_2\eta_1+q_1(\eta_0-\eta_1))(\tilde\eta_0-\tilde\eta_1)+\eta_1(\tilde\eta_2-\tilde\eta_1)}{(\binom{\fnum}{2}-1-q_2)\eta_1 + q_2\eta_0 + \eta_2}
        \end{align*}
        provided that the inverse exists and all denominators are unequal zero.

    \end{lemma}
        \begin{remark}
        Note that this system of equations can be directly solved by computing the first equation.
        Then, inserting this solution into the second equation and solving it explicitly.
        Finally inserting both results into the third equation to compute $\tilde\eta_1$, and consequently, $\tilde\eta_0$ and $\tilde\eta_2$.
        However, we do not need this explicit structure and will rely on the above conditions to prove our result.
        \end{remark}
    \begin{proof}[Proof of Lemma A.2.]
        We explicitly compute the product
        \begin{align*}
        \left(\eta_1 \mathbf{J} + (\eta_0-\eta_1) \mathbf{Q} + (\eta_2-\eta_1) \mathbf{I}\right) \left(\tilde\eta_1 \mathbf{J} + (\tilde\eta_0-\tilde\eta_1) \mathbf{Q} + (\tilde\eta_2-\tilde\eta_1) \mathbf{I}\right).
        \end{align*}
        In this computation the products $\mathbf{J}^2=\binom{\fnum}{2} \mathbf{J}$, $\mathbf{J} \mathbf{I} = \mathbf{I}\mathbf{J}= \mathbf{J}$, $\mathbf{I}^2 = \mathbf{I}$, $\mathbf{Q}\mathbf{I} = \mathbf{I}\mathbf{Q} = \mathbf{Q}$ are trivial to compute.
        Additionally,
        \begin{equation*}
            (\mathbf{Q}^2)_{S_1S_2} = \sum_{S\subseteq \fset}^{\vert S \vert = 2} \mathbf{1}(S_1\cap S = \emptyset) \mathbf{1}( S \cap S_2 = \emptyset) = \sum_{S\subseteq \fset}^{\vert S \vert = 2} \mathbf{1}((S_1\cup S_2) \cap S = \emptyset) = \binom{\fnum-\vert S_1 \cup S_2\vert}{2} = q_{\vert S_1\cap S_2\vert},
        \end{equation*}
        since $\vert S_1\cup S_2\vert = \vert S_1\vert + \vert S_2\vert - \vert S_1\cap S_2\vert = 4-\vert S_1\cap S_2\vert$ and $\mathbf{J}\mathbf{Q}=\mathbf{Q}\mathbf{J} = q_2 \mathbf{J}$.
        Hence, we can write $\mathbf{Q}^2 = q_1 \mathbf{J} + (q_0-q_1)\mathbf{Q} + (q_2-q_1) \mathbf{I}$.
        We then collect all coefficients  of $\mathbf{J}, \mathbf{Q},\mathbf{I}$ in the above product separately as
        \begin{align*}
            &\left(\eta_1 \mathbf{J} + (\eta_0-\eta_1) \mathbf{Q} + (\eta_2-\eta_1) \mathbf{I}\right) \left(\tilde\eta_1 \mathbf{J} + (\tilde\eta_0-\tilde\eta_1) \mathbf{Q} + (\tilde\eta_2-\tilde\eta_1) \mathbf{I}\right) 
            \\
            &=\mathbf{J}\left(\underbrace{\eta_1 \tilde\eta_1 \binom{\fnum}{2}}_{\text{from } \mathbf{J}^2}+\underbrace{q_2(\eta_0-\eta_1)\tilde\eta_1}_{\text{from } \mathbf{QJ}} + \underbrace{q_2\eta_1(\tilde\eta_0-\tilde\eta_1)}_{\text{from } \mathbf{JQ}} + \underbrace{\eta_1(\tilde\eta_2-\tilde\eta_1)}_{\text{from } \mathbf{JI}}+ \underbrace{\tilde\eta_1(\eta_2-\eta_1)}_{\text{from } \mathbf{IJ}} + \underbrace{q_1 (\eta_0-\eta_1)(\tilde\eta_0-\tilde\eta_1)}_{\text{from } \mathbf{Q}^2}\right)
            \\
            &+ \mathbf{Q}\left((\underbrace{\eta_2-\eta_1)(\tilde\eta_0-\tilde\eta_1)}_{\text{from }\mathbf{IQ}}+\underbrace{(\eta_0-\eta_1)(\tilde\eta_2-\tilde\eta_1)}_{\text{from } \mathbf{QI}} + \underbrace{(q_0-q_1)(\eta_0-\eta_1)(\tilde\eta_0-\tilde\eta_1)}_{\text{from } \mathbf{Q}^2}\right)
            \\
            &+ \mathbf{I}\left(\underbrace{(\eta_2-\eta_1)(\tilde\eta_2-\tilde\eta_1)}_{\text{from } \mathbf{I}^2} + \underbrace{(q_2-q_1)(\eta_0-\eta_1)(\tilde\eta_0-\tilde\eta_1)}_{\text{from } \mathbf{Q}^2} \right)
        \end{align*}

        Clearly, the coefficient of $\mathbf{J}$ and $\mathbf{Q}$ should be zero, whereas the coefficient of $\mathbf{I}$ should be equal to one to yield the identity matrix. 
        We thus obtain the following system of equations for the coefficients of $\mathbf{J}, \mathbf{Q}$ and $\mathbf{I}$, respectively:
        \begin{align*}
            0 &= (\tilde\eta_0-\tilde\eta_1) \cdot (q_2\eta_1 + q_1(\eta_0-\eta_1)) &&+ (\tilde\eta_2-\tilde\eta_1)\cdot \eta_1  &&+\tilde\eta_1 \cdot ((\binom{\fnum}{2}-1-q_2)\eta_1+q_2\eta_0+\eta_2)
            \\
            0 &= (\tilde\eta_0-\tilde\eta_1) \cdot ((\eta_2-\eta_1) + (q_0-q_1)(\eta_0-\eta_1)) &&+ (\tilde\eta_2-\tilde\eta_1) \cdot (\eta_0-\eta_1)
            \\
            1 &= (\tilde\eta_0-\tilde\eta_1) \cdot (q_2-q_1)(\eta_0-\eta_1) &&+ (\tilde\eta_2-\tilde\eta_1) \cdot (\eta_2-\eta_1).
        \end{align*}
        Solving the second equation for $\tilde\eta_0-\tilde\eta_1$ directly yields the second condition in Lemma~\ref{lem_inverse_sii}.
        Inserting this result into the third equation yields
        \begin{align*}
            1 &= (\tilde\eta_2-\tilde\eta_1) \cdot \left((\eta_2-\eta_1) - \frac{(q_2-q_1)(\eta_0-\eta_1)^2}{\eta_2-\eta_1+(q_0-q_1)(\eta_0-\eta_1)}\right) 
            \\
            &= (\tilde\eta_2-\tilde\eta_1) \cdot \left(\frac{(\eta_2-\eta_1)^2 + (q_0-q_1)(\eta_0-\eta_1)(\eta_2-\eta_1)-(q_2-q_1)(\eta_0-\eta_1)^2}{\eta_2-\eta_1+(q_0-q_1)(\eta_0-\eta_1)}\right),
        \end{align*}
        where solving for $\tilde\eta_2-\tilde\eta_1$ yields the first condition in Lemma~\ref{lem_inverse_sii}.
        Note that the value of $\tilde\eta_2-\tilde\eta_1$ can be explicitly computed, which implies and explicit representation of $\tilde\eta_0-\tilde\eta_1$.
        Lastly, treating $\tilde\eta_2-\tilde\eta_1$ and $\tilde\eta_0-\tilde\eta_1$ as known and solving the first equation for $\tilde\eta_1$ yields the third condition in Lemma~\ref{lem_inverse_sii}.
    \end{proof}
    Having established the conditions for the inverse, we apply Lemma~\ref{lem_inverse_sii} to $\mathbf{X}_2^T \mathbf{W}_2 \mathbf{X}_2 = \mu_{2,1} \mathbf{J} + (\mu_{2,0}-\mu_{2,1}) \mathbf{Q} + (\mu_{2,2}-\mu_{2,1}) \mathbf{I}$ to obtain $(\mathbf{X}_2^T \mathbf{W}_2 \mathbf{X}_2)^{-1} = \tilde\mu_{2,1} \mathbf{J} + (\tilde\mu_{2,0}-\tilde\mu_{2,1}) \mathbf{Q} + (\tilde\mu_{2,2}-\tilde\mu_{2,1}) \mathbf{I}$.
    For the following calculations the explicit form of $\mu_{2,\ell}$ and the following relations will be used
    \begin{align*}
       &q_2 - q_1 = \fnum-3, &&q_0-q_1=-(\fnum-4), &&q_2+q_1 = (\fnum-3)^2, &&\mu_{2,2}-\mu_{2,1}=\mu_{2,1}
    \end{align*}
    We now further simplify the terms using the structure of $\mu$ and $q$ as
    \begin{align}
        \tilde\mu_{2,2}-\tilde\mu_{2,1} &=  \frac{(\mu_{2,2}-\mu_{2,1})+(q_0-q_1)(\mu_{2,0}-\mu_{2,1})}{(\mu_{2,2}-\mu_{2,1})^2+(q_0-q_1)(\mu_{2,2}-\mu_{2,1})(\mu_{2,0}-\mu_{2,1})-(q_2-q_1)(\mu_{2,0}-\mu_{2,1})^2},\notag
        \\
        &=\frac{\mu_{2,1}-(\fnum-4)(\mu_{2,0}-\mu_{2,1})}{\mu_{2,1}^2-(\fnum-4)\mu_{2,1}(\mu_{2,0}-\mu_{2,1})-(\fnum-3)(\mu_{2,0}-\mu_{2,1})^2}\notag
        \\
        &=\frac{(\fnum-3)\mu_{2,1}-(\fnum-4)\mu_{2,0}}{\mu_{2,1}^2+(\mu_{2,1}-(\fnum-3)\mu_{2,0})(\mu_{2,0}-\mu_{2,1})} \notag
        \\
        &= \frac{(\fnum-3)\mu_{2,1}-(\fnum-4)\mu_{2,0}}{(\fnum-2)\mu_{2,1}\mu_{2,0}-(\fnum-3)\mu_{2,0}^2}.\label{eq_mu21}
    \end{align}
    We continue with the second equation and the previous result as
    \begin{align}
        \tilde\mu_{2,0}-\tilde\mu_{2,1} &= - \frac{(\mu_{2,0}-\mu_{2,1})(\tilde\mu_{2,2}-\tilde\mu_{2,1})}{\mu_{2,2}-\mu_{2,1} + (q_0-q_1)(\mu_{2,0}-\mu_{2,1})} \notag
        \\
        &=-\frac{(\mu_{2,0}-\mu_{2,1})}{\mu_{2,1} -(\fnum-4)(\mu_{2,0}-\mu_{2,1})} \cdot (\tilde\mu_{2,2}-\tilde\mu_{2,1}) \notag
        \\
        &= -\frac{\mu_{2,0}-\mu_{2,1}}{(\fnum-3)\mu_{2,1}-(\fnum-4)\mu_{2,0}} \cdot (\tilde\mu_{2,2}-\tilde\mu_{2,1}) \label{eq_mu01}
    \end{align}
    Lastly, the third equation yields
    \begin{align}
        \tilde\mu_{2,1} &= -\frac{(q_2\mu_{2,1}+q_1(\mu_{2,0}-\mu_{2,1}))(\tilde\mu_{2,0}-\tilde\mu_{2,1})+\mu_{2,1}(\tilde\mu_{2,2}-\tilde\mu_{2,1})}{(\binom{\fnum}{2}-1-q_2)\mu_{2,1} + q_2\mu_{2,0} + \mu_{2,2}} \notag
        \\
        &= -\frac{((\fnum-3)\mu_{2,1}+\binom{\fnum-3}{2}\mu_{2,0})(\tilde\mu_{2,0}-\tilde\mu_{2,1})+\mu_{2,1}(\tilde\mu_{2,2}-\tilde\mu_{2,1})}{(\binom{\fnum}{2}+1-\binom{\fnum-2}{2})\mu_{2,1} + \binom{\fnum-2}{2}\mu_{2,0}}\notag
        \\
        &= -\frac{(\fnum-3)\mu_{2,1}+\binom{\fnum-3}{2}\mu_{2,0}}{2(\fnum-1)\mu_{2,1} + \binom{\fnum-2}{2}\mu_{2,0}} \cdot (\tilde\mu_{2,0}-\tilde\mu_{2,1}) -\frac{\mu_{2,1}}{2(\fnum-1)\mu_{2,1} + \binom{\fnum-2}{2}\mu_{2,0}} \cdot (\tilde\mu_{2,2}-\tilde\mu_{2,1}) \label{eq_mu1}.
    \end{align}
    Since our goal is to compute $\lim_{\mu_\infty\to\infty} (\mathbf{X}_2^T \mathbf{W}_2 \mathbf{X}_2)^{-1} \mathbf{X}_2^T \mathbf{W}_2$ we distinguish between the finite subsets, where $\mu_2(t) \neq \mu_\infty$, i.e.\ $2 \leq \vert T \vert \leq \fnum-2$ and the infinite subsets, where $\mu_2(t) = \mu_\infty$, i.e.\ $\vert T \vert = 1$ or $\vert T \vert = \fnum-1$.
    \paragraph{$T$ with finite weight.}
    Clearly, for the $T$ with $\mu_2(t) \neq \mu_\infty$, we have that $\lim_{\mu_\infty\to\infty} (\mathbf{X}_2^T \mathbf{W}_2)_{TS} = (\mathbf{X}_2^T \mathbf{W}_2)_{TS}$ and thus we can compute the limit of $(\mathbf{X}_2^T\mathbf{W}_2\mathbf{X}_2)^{-1}$ separately.
    Recall that $\mu_1 \propto \mu_\infty$ and $\mu_0=\fnum-3$ does not depend on $\mu_\infty$.
    Hence, by \cref{eq_mu21}, \cref{eq_mu01} and \cref{eq_mu1}, it follows for $\mu_\infty \to \infty$
    \begin{align*}
        \tilde\mu_{2,2}-\tilde\mu_{2,1} \to \frac{1}{\fnum-2}, &&\tilde\mu_{2,0}-\tilde\mu_{2,1} \to \frac{1}{(\fnum-3)(\fnum-2)}, &&\tilde\mu_{2,1} \to -\frac{1}{(\fnum-1)(\fnum-2)}.
    \end{align*}
    Therefore, this proves the special case of Conjecture~\ref{conj_inverse} for $k=2$ as
    \begin{align*}
        &\tilde\mu_{2,1} \to -\frac{1}{(\fnum-1)(\fnum-2)} &&\tilde\mu_{2,2} \to \frac{1}{\fnum-1} &&\tilde\mu_{2,0} \to \frac{2}{(\fnum-1)(\fnum-2)(\fnum-3)}.
    \end{align*}
    With $\lambda(2,1)=-1/2$ it holds, $(\mathbf{J}\mathbf{X}_2^T\mathbf{W}_2)_{ST}=-\frac 1 2 t (\fnum-t)\mu_2(t)$, $(\mathbf{Q}\mathbf{X}_2^T\mathbf{W}_2)_{ST}=-\frac 1 2 (t-\vert T \cap S\vert)(\fnum-\vert T \cup S\vert)\mu_2(t)$ and $(\mathbf{I}\mathbf{X}_2^T\mathbf{W}_2)_{ST} = -\frac 1 2\mathbf{1}(\vert T \cap S \vert = 1)\mu_2(t)$.
    We can thus compute
    \begin{align*}
        &\left((\mathbf{X}_2^T \mathbf{W}_2 \mathbf{X}_2)^{-1} \mathbf{X}_2^T \mathbf{W}_2\right)_{ST} 
        \\
        &\overset{\mu_\infty \to \infty}{\longrightarrow} 
        \mu_2(t) \left(\underbrace{\frac {t (\fnum-t)} {2(\fnum-1)(\fnum-2)}}_{\text{from } \tilde\mu_{2,1} \mathbf{J}\mathbf{X}_2^T\mathbf{W}_2}
        -\underbrace{\frac{(t-\vert T \cap S \vert)(\fnum-t-2 + \vert T \cap S \vert)}{2(\fnum-3)(\fnum-2)}}_{\text{from } (\tilde\mu_{2,0}-\tilde\mu_{2,1})\mathbf{Q}\mathbf{X}_2^T\mathbf{W}_2} 
        - \underbrace{\frac{\mathbf{1}(\vert T \cap S \vert = 1)}{2(\fnum-2)}}_{\text{from } (\tilde\mu_{2,2}-\tilde\mu_{2,1})\mathbf{I}\mathbf{X}_2^T\mathbf{W}_2}\right)
        \\
        &= \frac{(\fnum-t-2)!(t-2)!}{(\fnum-4)!} \cdot
        \begin{cases}
            \frac {t (\fnum-t)} {2(\fnum-1)(\fnum-2)}- \frac{t (\fnum-t-2)}{2(\fnum-3)(\fnum-2)} = \frac{t((\fnum-t)(\fnum-1)-(\fnum-t-2)(\fnum-3))}{2(\fnum-1)(\fnum-2)(\fnum-3)} = \frac{t(t-1)}{(\fnum-1)(\fnum-2)(\fnum-3)}&\text{, if } T \cap S = \emptyset
            \\
            \frac{(\fnum-3)t(\fnum-t)-(\fnum-1)(t-1)(\fnum-t-1)-(\fnum-1)(\fnum-3)}{2(\fnum-1)(\fnum-2)(\fnum-3)} = -\frac{(t-1)(\fnum-t-1)}{(\fnum-1)(\fnum-2)(\fnum-3)} &\text{, if } \vert T \cap S \vert = 1
            \\
            \frac {t (\fnum-t)} {2(\fnum-1)(\fnum-2)}-\frac{(t-2)(\fnum-t)}{2(\fnum-3)(\fnum-2)} = \frac{(\fnum-t)((\fnum-3)t-(t-2)(\fnum-1))}{2(\fnum-1)(\fnum-2)(\fnum-3)} = \frac{(\fnum-t)(\fnum-t-1)}{(\fnum-1)(\fnum-2)(\fnum-3)}&\text{, if } \vert T \cap S \vert = 2
        \end{cases}
        \\
        &= (-1)^{\vert T \cap S\vert} \frac{(t-\vert T \cap S\vert)!(\fnum-t-2+\vert T \cap S \vert)!}{(\fnum-1)!} 
        = \frac{(-1)^{\vert T \cap S \vert }}{\fnum-1}\binom{\fnum-2}{t-\vert T\cap S\vert}^{-1},
    \end{align*}
    which yields the SII weights in \cref{eq_sii_weight} and concludes the proof for this case.   
        
    \paragraph{$T$ with infinite weight.}
    Without loss of generality, we consider $T=i$, since $(\mathbf{X}_2)_{iS}=(\mathbf{X}_2)_{(\fset\setminus i) S}$, and $(\mathbf{X}_2)_{iS}=\lambda(2,1)=-1/2$ for all $S=ij$ with $i \neq j \in \fset$, and zero otherwise.
    It holds $(\mathbf{J}\mathbf{X}_2^T\mathbf{W}_2)_{ST}=-\frac 1 2 (\fnum-1)\mu_\infty$, $(\mathbf{Q}\mathbf{X}_2^T\mathbf{W}_2)_{ST}=-\frac 1 2 \mathbf{1}(i\notin S)(\fnum-3)\mu_\infty$ and $(\mathbf{I}\mathbf{X}_2^T\mathbf{W}_2)_{ST} = -\frac 1 2\mathbf{1}(i \in S)\mu_\infty$.
    \begin{align*}
       \left((\mathbf{X}_2^T \mathbf{W}_2 \mathbf{X}_2)^{-1} \mathbf{X}_2^T \mathbf{W}_2\right)_{Si} 
        &= -\frac 1 2 \mu_\infty \left((\fnum-1)\tilde\mu_{2,1}+\mathbf{1}(i \notin S)(\fnum-3)(\tilde\mu_{2,0}-\tilde\mu_{2,1})+\mathbf{1}(i \in S)(\tilde\mu_{2,2}-\tilde\mu_{2,1}) \right)
       \\
       &=-\frac 1 2 \mu_\infty \cdot
       \begin{cases}
           (\fnum-1)\tilde\mu_{2,1} + (\fnum-3)(\tilde\mu_{2,0}-\tilde\mu_{2,1}) &\text{, if } i \notin S
           \\
           (\fnum-1)\tilde\mu_{2,1} + (\tilde\mu_{2,2}-\tilde\mu_{2,1} )&\text{, if } i \in S
       \end{cases}
    \end{align*}
    We now first explicitly compute $\tilde\mu_{2,1}$ with \cref{eq_mu1} and \cref{eq_mu01} as
    \begin{align}
        \tilde\mu_{2,1} 
        &\overset{(\ref{eq_mu1})}{=} -\frac{(\fnum-3)\mu_{2,1}+\binom{\fnum-3}{2}\mu_{2,0}}{2(\fnum-1)\mu_{2,1} + \binom{\fnum-2}{2}\mu_{2,0}} \cdot (\tilde\mu_{2,0}-\tilde\mu_{2,1}) -\frac{\mu_{2,1}}{2(\fnum-1)\mu_{2,1} + \binom{\fnum-2}{2}\mu_{2,0}} \cdot (\tilde\mu_{2,2}-\tilde\mu_{2,1}) \notag
        \\
        &\overset{(\ref{eq_mu01})}{=} \left(\frac{(\fnum-3)\mu_{2,1}+\binom{\fnum-3}{2}\mu_{2,0}}{2(\fnum-1)\mu_{2,1} + \binom{\fnum-2}{2}\mu_{2,0}} 
        \cdot
        \frac{\mu_{2,0}-\mu_{2,1}}{(\fnum-3)\mu_{2,1}-(\fnum-4)\mu_{2,0}} 
        -\frac{\mu_{2,1}}{2(\fnum-1)\mu_{2,1} + \binom{\fnum-2}{2}\mu_{2,0}}\right) \cdot (\tilde\mu_{2,2}-\tilde\mu_{2,1}) \notag
        \\
        &=\left(\frac{((\fnum-3)\mu_{2,1}+\binom{\fnum-3}{2}\mu_{2,0})
        \cdot (\mu_{2,0}-\mu_{2,1})
        - \mu_{2,1}((\fnum-3)\mu_{2,1}-(\fnum-4)\mu_{2,0})}
        {(2(\fnum-1)\mu_{2,1} + \binom{\fnum-2}{2}\mu_{2,0})
        \cdot((\fnum-3)\mu_{2,1}-(\fnum-4)\mu_{2,0})} 
        \right) \cdot (\tilde\mu_{2,2}-\tilde\mu_{2,1}) \notag
        \\
        &= \left(\frac{((\fnum-3)\mu_{2,1}+\binom{\fnum-3}{2}\mu_{2,0})
        \cdot (\mu_{2,0}-\mu_{2,1})
        - \mu_{2,1}((\fnum-3)\mu_{2,1}-(\fnum-4)\mu_{2,0})}
        {(2(\fnum-1)\mu_{2,1} + \binom{\fnum-2}{2}\mu_{2,0})
        \cdot((\fnum-3)\mu_{2,1}-(\fnum-4)\mu_{2,0})} 
        \right) \cdot (\tilde\mu_{2,2}-\tilde\mu_{2,1}) \notag
        \\
        &= \left(\frac{
        -2(\fnum-3)\mu_{2,1}^2
        +((\fnum-4)+(\fnum-3)-\binom{\fnum-3}{2})\mu_{2,1}\mu_{2,0}
        +\binom{\fnum-3}{2} \mu_{2,0}^2}
        {(2(\fnum-1)\mu_{2,1} + \binom{\fnum-2}{2}\mu_{2,0})
        \cdot((\fnum-3)\mu_{2,1}-(\fnum-4)\mu_{2,0})} 
        \right) \cdot (\tilde\mu_{2,2}-\tilde\mu_{2,1}) \notag
        \\
        &=:
        \frac{\gamma^\uparrow}{\gamma^{\downarrow}} \cdot (\tilde\mu_{2,2}-\tilde\mu_{2,1}) \label{eq_mu121}   
    \end{align}
    Utilizing the new notation, we obtain for $i \in S$
    \begin{align*}
        \left((\mathbf{X}_2^T \mathbf{W}_2 \mathbf{X}_2)^{-1} \mathbf{X}_2^T \mathbf{W}_2\right)_{Si} 
        &= -\frac 1 2 \mu_\infty((\fnum-1)\tilde\mu_{2,1} + (\tilde\mu_{2,2}-\tilde\mu_{2,1})) 
        \overset{(\ref{eq_mu121})}{=} -\frac 1 2 \mu_\infty((\fnum-1)\frac{\gamma^\uparrow}{\gamma^\downarrow}+1) (\tilde\mu_{2,2}-\tilde\mu_{2,1}) 
        \\
        &= -\frac 1 2 \mu_\infty\frac{(\fnum-1)\gamma^\uparrow+\gamma^\downarrow}{\gamma^\downarrow}(\tilde\mu_{2,2}-\tilde\mu_{2,1}).
    \end{align*}
    In the following, we use $\approx$ to indicate that terms grow similarly for $\mu_\infty \to \infty$.
    We further observe that the terms with $\mu_{2,1}^2$ vanish in $(\fnum-1)\gamma^\uparrow+\gamma^\downarrow$, and
    \begin{align*}
        (\fnum-1)\gamma^\uparrow+\gamma^\downarrow &\approx \mu_{2,1} \left((\fnum-1)((\fnum-4)+(\fnum-3)-\binom{\fnum-3}{2})\mu_{2,0} + \mu_{2,0}(\binom{\fnum-2}{2}(\fnum-3)-2(\fnum-1)(\fnum-4))\right)
        \\
        &= \mu_{2,1}\mu_{2,0}((\fnum-1) - (\fnum-1)\binom{\fnum-3}{2} +(\fnum-3)\binom{\fnum-2}{2})
        \\
        &= \mu_{2,1}\mu_{2,0}((\fnum-1) - (\fnum-3)(\binom{\fnum-3}{2}-\binom{\fnum-2}{2})-2\binom{\fnum-3}{2})
        \\
        &= \mu_{2,1}\mu_{2,0}((\fnum-1) - (\fnum-3)(\fnum-3)-(\fnum-3)(\fnum-2))
        \\
        &= 2\mu_{2,1}\mu_{2,0}(\fnum-2).
    \end{align*}
    Clearly, $\gamma^\downarrow \approx 2(\fnum-1)(\fnum-3)\mu_{2,1}^2$ are the dominating terms.
    Hence, with $\mu_{2,1} \approx \frac 1 2 \mu_\infty$, we have $\frac 1 2\mu_\infty\frac{\mu_{2,1}}{\mu_{2,1}^2} \to 1$ and with $\mu_{2,0}=(\fnum-3)$ we compare the coefficients of $\mu_{2,1}$ to obtain
   \begin{align*}
        \lim_{\mu_\infty \to \infty} \left((\mathbf{X}_2^T \mathbf{W}_2 \mathbf{X}_2)^{-1} \mathbf{X}_2^T \mathbf{W}_2\right)_{Si} 
        = -\underbrace{\frac{ 2(\fnum-3)(\fnum-2)}{2(\fnum-1)(\fnum-3)}}_{\text{from } \frac 1 2 \mu_\infty \frac{(\fnum-1)\gamma^\uparrow + \gamma^\downarrow}{\gamma^\downarrow}} \cdot \underbrace{\frac{1}{\fnum-2}}_{\text{from }\tilde\mu_{2,2}-\tilde\mu_{2,1}} 
        = - \frac 1 {\fnum-1} = \frac{(-1)^{\vert T \cap S \vert}}{\fnum-1}\binom{\fnum-2}{t-\vert T \cap S\vert}^{-1},
   \end{align*}
   which proves the convergence to the SII weight for the case where $T=i$ and $i \in S$.
   
    To prove the case for $T=i$ and $i \notin S$, we first represent $\tilde\mu_{2,1}$ with $\tilde\mu_{2,0}-\tilde\mu_{2,1}$. 
    Using the reverse relation in \cref{eq_mu01} and \cref{eq_mu1}, we obtain
    \begin{align}
       \tilde\mu_{2,1} 
        &\overset{(\ref{eq_mu1})}{=} -\frac{(\fnum-3)\mu_{2,1}+\binom{\fnum-3}{2}\mu_{2,0}}{2(\fnum-1)\mu_{2,1} + \binom{\fnum-2}{2}\mu_{2,0}} \cdot (\tilde\mu_{2,0}-\tilde\mu_{2,1}) -\frac{\mu_{2,1}}{2(\fnum-1)\mu_{2,1} + \binom{\fnum-2}{2}\mu_{2,0}} \cdot (\tilde\mu_{2,2}-\tilde\mu_{2,1}) \notag
        \\ 
        &\overset{(\ref{eq_mu01})}{=} \left(-\frac{(\fnum-3)\mu_{2,1}+\binom{\fnum-3}{2}\mu_{2,0}}{2(\fnum-1)\mu_{2,1} + \binom{\fnum-2}{2}\mu_{2,0}} +\frac{\mu_{2,1}}{2(\fnum-1)\mu_{2,1} + \binom{\fnum-2}{2}\mu_{2,0}} \cdot \frac{(\fnum-3)\mu_{2,1}-(\fnum-4)\mu_{2,0}}{\mu_{2,0}-\mu_{2,1}}  \cdot \right)(\tilde\mu_{2,0}-\tilde\mu_{2,1}) \notag
        \\
        &= \frac{-(\mu_{2,0}-\mu_{2,1})\cdot ((\fnum-3)\mu_{2,1}+\binom{\fnum-3}{2}\mu_{2,0})
        +\mu_{2,1} \cdot ((\fnum-3)\mu_{2,1}-(\fnum-4)\mu_{2,0})}
        {(2(\fnum-1)\mu_{2,1} + \binom{\fnum-2}{2}\mu_{2,0})\cdot (\mu_{2,0}-\mu_{2,1})}
        \cdot (\tilde\mu_{2,0}-\tilde\mu_{2,1}) \notag
        \\
        &= \frac{\mu_{2,1}^2 2(\fnum-3)
        -\mu_{2,1}\mu_{2,0}((\fnum-3)-\binom{\fnum-3}{2}+(\fnum-4))
        -\mu_{2,0}^2\binom{\fnum-3}{2}}
        {-\mu_{2,1}^2 2(\fnum-1)
        +\mu_{2,1}\mu_{2,0}(2(\fnum-1)-\binom{\fnum-2}{2})+\mu_{2,0}^2\binom{\fnum-2}{2}}
        \cdot (\tilde\mu_{2,0}-\tilde\mu_{2,1}) \notag
        \\
        &= \frac{\mu_{2,1}^2 2(\fnum-3)
        -\mu_{2,1}\mu_{2,0}((\fnum-3)-\binom{\fnum-3}{2}+(\fnum-4))
        -\mu_{2,0}^2\binom{\fnum-3}{2}}
        {-\mu_{2,1}^2 2(\fnum-1)
        +\mu_{2,1}\mu_{2,0}(2(\fnum-1)-\binom{\fnum-2}{2})+\mu_{2,0}^2\binom{\fnum-2}{2}}
        \cdot (\tilde\mu_{2,0}-\tilde\mu_{2,1}) \notag
        \\
        &=: \frac{\tau^\uparrow}{\tau^\downarrow} \cdot (\tilde\mu_{2,0}-\tilde\mu_{2,1}) \label{eq_mu101}.
    \end{align}
   Again, utilizing the new notation, we obtain for $i \notin S$
    \begin{align*}
        \left((\mathbf{X}_2^T \mathbf{W}_2 \mathbf{X}_2)^{-1} \mathbf{X}_2^T \mathbf{W}_2\right)_{Si} 
        &= -\frac 1 2 \mu_\infty((\fnum-1)\tilde\mu_{2,1} + (\fnum-3)(\tilde\mu_{2,0}-\tilde\mu_{2,1}))
        \\
        &\overset{(\ref{eq_mu101})}{=} -\frac 1 2 \mu_\infty((\fnum-1)\frac{\tau^\uparrow}{\tau^\downarrow}+\fnum-3) (\tilde\mu_{2,0}-\tilde\mu_{2,1}) 
        \\
        &= -\frac 1 2 \mu_\infty\frac{(\fnum-1)\tau^\uparrow+(\fnum-3)\tau^\downarrow}{\tau^\downarrow}(\tilde\mu_{2,0}-\tilde\mu_{2,1}).
    \end{align*}
    We observe again that the terms of $\mu_{1,2}^2$ cancel in $(\fnum-1)\tau^\uparrow+(\fnum-3)\tau^\downarrow$ and thus the coefficients of $\mu_{2,1}$ are the dominating terms for $\mu_\infty \to \infty$.
    Therefore,
    \begin{align*}
        (\fnum-1)\tau^\uparrow+(\fnum-3)\tau^\downarrow &\approx \mu_{2,1}\mu_{2,0} (-(\fnum-1)((\fnum-3)-\binom{\fnum-3}{2}+(\fnum-4))+(\fnum-3)(2(\fnum-1)-\binom{\fnum-2}{2}))
        \\
        &= \mu_{2,1}\mu_{2,0}(\underbrace{(\fnum-1)(\fnum-3)-(\fnum-1)(\fnum-4)}_{=\fnum-1} + \underbrace{\binom{\fnum-3}{2}((\fnum-1) - (\fnum-3)) - (\fnum-3)(\fnum-3))}_{=(\fnum-3)(\fnum-4)-(\fnum-3)(\fnum-3)=-(\fnum-3)}
        \\
        &=2\mu_{2,1}\mu_{2,0}.
    \end{align*}
    Clearly the denominating terms in $\tau^\downarrow$ are $\tau^\downarrow \approx -\mu_{2,1}^2 2(\fnum-1)$. 
    Hence, with $\mu_{2,1} \approx \frac 1 2 \mu_\infty$, we have $\frac 1 2\mu_\infty\frac{\mu_{2,1}}{\mu_{2,1}^2} \to 1$ and with $\mu_{2,0}=(\fnum-3)$ we compare the coefficients of $\mu_{2,1}$ to obtain
   \begin{align*}
        \lim_{\mu_\infty \to \infty} \left((\mathbf{X}_2^T \mathbf{W}_2 \mathbf{X}_2)^{-1} \mathbf{X}_2^T \mathbf{W}_2\right)_{Si} 
        &= -\frac{ 2(\fnum-3)}{-2(\fnum-1)} \cdot \underbrace{\frac{1}{(\fnum-2)(\fnum-3)}}_{\text{from }\tilde\mu_{2,0}-\tilde\mu_{2,1}} 
        = \frac 1 {(\fnum-1)(\fnum-2)} = \frac{(-1)^{\vert T \cap S \vert}}{\fnum-1}\binom{\fnum-2}{t-\vert T \cap S\vert}^{-1},
   \end{align*}  
   which proves the convergence to the SII weight for $T=i$ and $i \notin S$, and finishes the proof.
\end{proof}

\clearpage
\section{Algorithms and Ground Truth Calculations}
\label{appx_sec_algos}
In this section, we give further details on the implemented algorithms, the default parameters of KernelSHAP-IQ and the computation of GT SII values for the synthetic SOUM.

\subsection{KernelSHAP-IQ Default Parameters}
We use the following default configurations in our experiments:
\begin{itemize}
    \item \textbf{Sampling weights:} We use $p^*(T) \propto \mu_1(t)$, similar to KernelSHAP \cite{Covert.2021}, i.e.\ 
    \begin{equation*}
        q(t) := p^*(\vert T \vert = t) \propto \binom{\fnum}{t}\mu_1(t) \propto \frac{1}{t(\fnum-t)}.
    \end{equation*}
    We further use similar sampling weights for the baseline methods, if applicable, i.e.\ SHAP-IQ \cite{Fumagalli.2023} and SVARM-IQ \cite{Kolpaczki.2024b}.
    \item \textbf{Border-Trick:} We use the \emph{border-trick} \cite{Fumagalli.2023,Lundberg.2017} to split the sampling procedure in a deterministic and a sampling part, cf. \cref{appx_sec_sampling}.
    We also apply this method to the baseline algorithms, were it is applicable, i.e.\ SHAP-IQ \cite{Fumagalli.2023} and SVARM-IQ \cite{Kolpaczki.2024b}.
    \item \textbf{Infinite Weight:} We set $\mu_\infty = 10^6$, in line with KernelSHAP \cite{Lundberg.2017}.
\end{itemize}
\subsection{Inconsistent KernelSHAP-IQ}\label{appx_sec_inconsistent}
The inconsistent KernelSHAP-IQ procedure is similar to KernelSHAP-IQ.
However, there is no iterative computation involved and the weights do not change.
We outline the pseudocode in \cref{appx_alg_inconsistent}.

\begin{algorithm}[htb]
    \caption{Inconsistent KernelSHAP-IQ}
    \label{appx_alg_inconsistent}
    \begin{algorithmic}[1]
    \REQUIRE order $k$, sampling weights $q$, budget $b$.
    \STATE $\{T_i\}_{i=1,\dots,b}, \{w_{T}\}_{T=T_1,\dots,T_b} \gets \textsc{\texttt{Sample}}(q,b)$
    \STATE $\mathbf{y}_1 \gets [\nu(T_1),\dots,\nu(T_b)]^T$
    \FOR{$T \in \{T_i\}$ and $1 \leq \vert S \vert \leq k$}
        \STATE $(\hat{\mathbf{X}}_{\leq k})_{TS} \gets \lambda(\vert S \vert,\vert T \cap S \vert)$ \COMMENT{Bernoulli weighting}
         \STATE $(\hat{\mathbf{W}}^*_{\leq k})_{TT} \gets \mu_{1}(t) \cdot w_{T}$ \COMMENT{weight adjustment}
    \ENDFOR
    \STATE $\hat{\phi}_{\leq k} \gets \textsc{\texttt{SolveWLS}}(\hat{\mathbf{X}}_{\leq k},\hat{\mathbf{y}}_{\leq k},\hat{\mathbf{W}}^*_{\leq k})$
    \STATE $\hat{\phi}_1,\dots,\hat{\phi}_k \gets \textsc{\texttt{Unstack}}(\hat{\phi}_{\leq k})$
    \STATE $\hat\Phi_k \gets \textsc{\texttt{AggregateSII}}$$(\hat\phi_1,\dots,\hat\phi_k)$ \COMMENT{compute $k$-SII} 
    \STATE \textbf{return} $k$-SII estimates $\hat\Phi_k$, SII estimates $\hat{\phi}_{\leq k}$
    \end{algorithmic}
\end{algorithm}

\subsection{Sampling Algorithm}\label{appx_sec_sampling}
In this section, we describe our sampling approach.
We make use of the \emph{border-trick} \cite{Fumagalli.2023}, which computes the low- and high-cardinality subsets explicitly without sampling, if the expected number of subsets is higher than the number of subsets of that size.
The sampling procedure is split in a \emph{deterministic} ($t<q_0$ and $t>\fnum-q_0)$ and a \emph{sampling} part ($q_0 \leq t \leq \fnum-q_0$).
The method takes a (symmetric) sampling weight vector $q\geq 0$ and budget $b>0$ and returns $b$ distinct subsets with \emph{subset weights} $w_T$ for each subset adjusted for the sampling distribution to readily apply the KernelSHAP-IQ weights at a later step.

Consider first the initial probability distribution $p^*(T) \propto q(t)$.
It is clear that this distribution is explicitly given as
\begin{equation*}
    p^*(T) := \frac{q(t)}{\binom{\fnum}{t} \cdot \sum_{\ell=0}^{\fnum} q(\ell)}, \text{ where }   p^*(\vert T \vert = t) = \frac{q(t)}{\sum_{\ell=0}^{\fnum-0} q(\ell)}.
\end{equation*}
The expected number of subsets of a given size $t$ is thereby given as $p(\vert T \vert = t) \cdot b$, where we consider the subset size $t$ in the deterministic part, if the expected number of subsets exceeds the total number of subsets of that size $\binom{\fnum}{t}$.
The procedure that splits the subset sizes in the deterministic and sampling part based on this notion is outlined in \cref{appx_alg_sampling_order}, where $q_0$ is iteratively increased and finally returned.
Given the sampling order $q_0$, we can re-define the sampling probabilities as
\begin{equation*}
    p^*_{q_0}(T) := \frac{q(t)}{\binom{\fnum}{t} \cdot \sum_{\ell=q_0}^{\fnum-q_0} q(\ell)}, \text{ where }   p^*_{q_0}(\vert T \vert = t) = \frac{q(t)}{\sum_{\ell=q_0}^{\fnum-q_0} q(\ell)},
\end{equation*}
where $p^*_{q_0}$ is now a probability distribution over $\mathcal T_{q_0}$, i.e.\ subsets of size $q_0 \leq t \leq \fnum-q_0$.
Recall that our goal is to approximate a sum over squared losses $\tilde\nu(T)$ weighted by the KernelSHAP-IQ weights $\mu(t)$.
We can re-write this sum as
\begin{equation*}
    \sum_{T \subseteq \fset} \mu(t) \tilde\nu(T) = \sum_{T \subseteq \fset}^{T \notin \mathcal T_{q_0}} \mu(t) \tilde\nu(T) + \mathbb{E}_{T \sim p_{q_0}^*}\left[\frac{\mu(t)\tilde\nu(T)}{p_{q_0}^*(T)}\right] \approx \sum_{T \subseteq \fset}^{T \notin \mathcal T_{q_0}} \mu(t) \tilde\nu(T) + \frac{1}{n_{\text{samples}}} \sum_{\ell=1}^{n_{\text{samples}}} \frac{\mu(t_\ell)\tilde\nu(T_\ell)}{p_{q_0}^*(T_\ell)},
\end{equation*}
where the $n_{\text{samples}}$ Monte Carlo samples $T_1,\dots,T_{n_{\text{samples}}}$ are drawn according to $p^*_{q_0}$.
In practice, note it is easy to sample from this distribution by sampling a subset size according to $p_{q_0}^*(\vert T \vert = t)$ and then sample a subset $T$ of that size uniformly with probability $\binom{\fnum}{t}^{-1}$.
Clearly, the sampling procedure should therefore return the sampling probabilities
\begin{align*}
    &\textbf{Deterministic part: } w_T \equiv 1 &&\textbf{Sampling part: } w_T = \frac{1}{n_{\text{samples}} \cdot p^*_{q_0}(T)}.
\end{align*}
If a subset is sampled multiple times, then clearly the budget should be decreased, as the game only needs to be evaluated once.
As a consequence, we simply increase the weight proportionally and leave the budget unchanged.
The full algorithm is outlined in \cref{appx_alg_sampling}.

\begin{algorithm}[htb]
    \caption{Sampling}
    \label{appx_alg_sampling}
    \begin{algorithmic}[1]
    \REQUIRE Budget $b>0$, sampling weights $\{q(t)\}_{t=0,\dots,\fnum}$ 
    \STATE $q_0 \gets \textsc{\texttt{SamplingOrder}}(b,q)$
    \STATE $i \gets 1$
    \FOR[deterministic part]{$T \subseteq \fset$ with $T \notin \mathcal T_{q_0}$}
        \STATE $T_i \gets T$
        \STATE $w_T  \gets 1$ \COMMENT{no weighting adjustment}
        \STATE $b \gets b-1$ \COMMENT{reduce budget}
    \ENDFOR
    \STATE $q_{\text{sampling}}(t) \gets q(t)/\sum_{\ell =q_0}^{\fnum-q_0}q(\ell)$ for $q_0 \leq t \leq \fnum-q_0$ \COMMENT{sampling weight normalization, i.e.\ $p^*_{q_0}(\vert T \vert = t)$}
    \STATE $n_\text{samples} \gets b$ \COMMENT{store total number of subsets sampled}
    \WHILE[sampling part]{$b>0$}
    \STATE $t \gets $ Sample according to $q_{\text{sampling}}$
    \STATE $T \gets $ Sample uniformly of size $t$
    \STATE $p(T) \gets q_{\text{sampling}}(t) \cdot \binom{\fnum}{t}^{-1}$ \COMMENT{probability to draw this subset}
    \IF[Subset is new]{$T \notin \{T_i\}$}
    \STATE $w_T \gets (n_{\text{samples}} \cdot p(T))^{-1} $
    \STATE $b \gets b -1$ \COMMENT{reduce budget}
    \ELSE[Subset is not new]
    \STATE $w_T \gets w_T + (n_{\text{samples}} \cdot p(T))^{-1} $ \COMMENT{increase weighting for each occurrence of $T$}
    \ENDIF
    \ENDWHILE
    \STATE \textbf{return} $\{T_i\}_{i=1,\dots,b}, \{w_{T}\}_{T=T_1,\dots,T_b}$
    \end{algorithmic}
\end{algorithm}

\begin{algorithm}[htb]
    \caption{Compute Sampling Order}
    \label{appx_alg_sampling_order}
    \begin{algorithmic}[1]
    \REQUIRE Budget $b>0$, sampling weights $\{q(t)\}_{t=0,\dots,\fnum}$ 
        \STATE initialize $q_0 = 0$
        \FOR{$t = 0,\dots,\text{\textsc{\texttt{Floor}}}(\fnum/2)$}
            \STATE $q(t) \gets q(t)/\sum_{\ell = q_0}^{\fnum-q_0}q(\ell)$ \COMMENT{weight normalization}
            \IF[compare expected number of subsets with total number of subsets]{$b \cdot q(t) \geq \binom{\fnum}{t}$ and $b \cdot q(\fnum-t) \geq \binom{\fnum}{t}$}
                \STATE $q_0 \gets q_0 + 1$ \COMMENT{increase $q_0$}
                \STATE $b \gets b- 2 \binom{\fnum}{t}$
            \ELSE
            \STATE \textbf{break}
            \ENDIF
        \ENDFOR
    \STATE \textbf{return} $q_0$
    \end{algorithmic}
\end{algorithm}

\subsection{Solving WLS}
The WLS problem can be solved analytically as described in \cref{appx_alg_solve_wls}.
\begin{algorithm}[t]
    \caption{SolveWLS}
    \label{appx_alg_solve_wls}
    \begin{algorithmic}[1]
    \REQUIRE Data $\mathbf{X}$, weights $\mathbf{W}$, response $\mathbf{y}$
    \STATE $\mathbf{A} \gets (\mathbf{X}^T \mathbf{W} \mathbf{X})^{-1}$ \COMMENT{precision matrix}
    \STATE $\phi \gets \mathbf{A} \mathbf{X}^T\mathbf{W} \mathbf{y}$
    \STATE \textbf{return} $\phi$
    \end{algorithmic}
\end{algorithm}

\subsection{Subset Weights SII}
We briefly describe the weighting for SII.
According to \citet{Fumagalli.2023}, the SII is represented as
\begin{equation*}
    \phi^{\text{SII}}(S) = \sum_{T\subseteq \fset} (-1)^{s-\vert T \cap S\vert}m(t-\vert T \cap S \vert), \text{ where } m(t):=\frac{(\fnum-t-s)!t!}{(\fnum-s+1)!}.
\end{equation*}
The weights are therefore assigned to every combinations $T,S$ as outlined  in \cref{appx_alg_sii_weights}.

\begin{algorithm}[h]
    \caption{SIIWeight}
    \label{appx_alg_sii_weights}
    \begin{algorithmic}[1]
    \REQUIRE Subset $T \subseteq \fset$, interaction $S \subseteq \fset$
    \STATE $t,s,r \gets \vert T \vert, \vert S \vert, \vert T \cap S \vert$
    \STATE $\omega \gets (-1)^{s-r} \frac{(\fnum-t-s+r)!(t-r)!}{(\fnum-s+1)!}$
    \STATE \textbf{return} $\omega$
    \end{algorithmic}
\end{algorithm}

\subsection{Aggregate SII to $k$-SII}\label{appx_sec_ksii_weights}
This section describes the $k$-SII weights, i.e.\ the aggregation of SII to $k$-SII.
It has been shown in Appendix~A.2 in \cite{Bord.2023} that $k$-SII can be explicitly represented as
\begin{equation*}
    \Phi_k(S) = \sum_{S \subseteq \tilde S \subseteq \fset}^{\vert \tilde S \vert \leq k} B_{\vert \tilde S \vert- \vert S \vert} \phi^{\text{SII}}(\tilde S) =  \sum_{\tilde S \subseteq \fset}^{\vert \tilde S \vert \leq k} \mathbf{1}_{S \subseteq \tilde S } B_{\vert \tilde S \vert- \vert S \vert} \phi^{\text{SII}}(\tilde S)  \text{ for } 1 \leq \vert S \vert \leq k.
\end{equation*}
Hence, we can aggregate SII to $k$-SII as outlined in \cref{appx_alg_ksii_weights}.

\begin{algorithm}[!htb]
    \caption{AggregateSII}
    \label{appx_alg_ksii_weights}
    \begin{algorithmic}[1]
    \REQUIRE SII estimates $\hat\phi_1,\dots,\hat\phi_k$    
    \STATE $(\mathbf{Z}_k)_{S\tilde S} \gets \mathbf{1}_{S \subseteq \tilde S} B_{\vert \tilde S \vert - \vert S \vert}$ for $1\leq \vert S \vert , \vert \tilde S \vert \leq k$
    \STATE $\hat\Phi_k \gets \mathbf{Z}_k \cdot [\hat\phi_1,\dots,\hat\phi_k]^T$
    \STATE \textbf{return} $\hat\Phi_k$

    \end{algorithmic}
\end{algorithm}

\subsection{Baseline Methods: SHAP-IQ, Permutation Sampling and SVARM-IQ}
The pseudocode and implementations can be found in the corresponding paper.
Permutation sampling \cite{Tsai.2022,Fumagalli.2023} is described for SII as an extension to permutation sampling of the SV \cite{Castro.2009}.
SHAP-IQ \cite{Fumagalli.2023} is implemented using similar sampling weights as KernelSHAP-IQ and extends Unbiased KernelSHAP \cite{Covert.2021} to interactions.
SVARM-IQ \cite{Kolpaczki.2024b} extends SVARM \cite{Kolpaczki.2024a} to interactions.

\subsection{Analytic Solution for SII of SOUMs}\label{appx_sii_gt}
\begin{proposition}
    For a unanimity game $\nu_R(T) := \mathbf{1}_{R \subseteq T}$ with $R\subseteq \fset$, the SII for any $S \subseteq \fset$ can be explicitly computed as
\begin{equation*}
    \phi^{\text{SII}}_{\nu_R}(S) =  \frac{\mathbf{1}_{S \subseteq R}}{r-s+1}.
\end{equation*}
Hence for a SOUM, defined as $\nu^{\text{SOUM}}(T) := \sum_{m=1}^M a_m \nu_{R_m}(T)$ with subsets $R_1,\dots,R_M \subseteq \fset$ and coefficients $a_1,\dots,a_M \in \mathbb{R}$, it follows
\begin{equation*}
    \phi^{\text{SII}}_{\nu^{\text{SOUM}}}(S) = \sum_{m=1}^M a_m \frac{\mathbf{1}_{S \subseteq R_m}}{r_m-s+1}.
\end{equation*}
\end{proposition}
\begin{proof}
    Note that due to linearity the second result follows immediately from the SII of the unanimity game.
    We thus aim to compute the SII for a unanimity game $\nu_R$ with subset $R \subseteq \fset$.
    In the following, we give two different proofs.
    \paragraph{Proof via Möbius transform and conversion formula.}
    It is well-known in cooperative game theory, cf. p.54 in \cite{Grabisch.2016}, that the Möbius transform of a unanimity game is
    \begin{equation*}
        a_{\nu_R}(S) = \sum_{T \subseteq S } (-1)^{s-t} \nu_R(T) = \mathbf{1}_{R=S}. 
    \end{equation*}
    Furthermore, the result follows immediately from the conversion formula, cf. Table 3 in \cite{Grabisch.2000}, as
    \begin{equation*}
        \phi^{\text{SII}}_{\nu_R}(S) = \sum_{T \supseteq S} \frac{1}{t-s+1} a_{\nu_R}(T) 
        =\sum_{T \subseteq \fset} \frac{\mathbf{1}_{S\subseteq T} }{t-s+1} a_{\nu_R}(T)
        =\mathbf{1}_{R=T} \frac{\mathbf{1}_{S\subseteq T}}{t-s+1}
        = \frac{\mathbf{1}_{S \subseteq R}}{r-s+1}.
    \end{equation*}

    \paragraph{Proof via computation.}
    We can also give an alternative analytical proof.    
    We compute the discrete derivatives of $\nu_R$ for $S\subseteq \fset$ and $T \subseteq \fset\setminus S$ as
    \begin{align*}
        \Delta_S(T) &= \sum_{L \subseteq S}(-1)^{s-\ell} \nu_R(T \cup L) = \sum_{L \subseteq S}(-1)^{s-\ell} \mathbf{1}_{R \subseteq T \cup L} = \mathbf{1}_{R\setminus S \subseteq T} \sum_{L \subseteq S}(-1)^{s-\ell}  \mathbf{1}_{R \cap S \subseteq L}
        \\
        &= \mathbf{1}_{R\setminus S \subseteq T} \sum_{\ell=\vert R \cap S\vert}^{s} (-1)^{s-\ell} \binom{s -\vert R \cap S \vert}{\ell-\vert R \cap S \vert} 
        = \mathbf{1}_{R\setminus S \subseteq T} \underbrace{\sum_{\ell=0}^{s-\vert R \cap S \vert} (-1)^{s-\ell-\vert R \cap S \vert} \binom{s -\vert R \cap S \vert}{\ell}}_{=0, \text{ except for } \vert R \cap S \vert = s} 
        =         \mathbf{1}_{R\setminus S \subseteq T} \mathbf{1}_{S \subseteq R}.
    \end{align*}
    Hence, with $q := \vert R \setminus S\vert$, the SII is computed as
    \begin{align*}
        \phi^{\text{SII}}(S) 
        &= \sum_{T \subseteq \fset\setminus S} \frac{(\fnum-t-s)!t!}{(\fnum-s+1)!} \Delta_S(T) 
        &&=\mathbf{1}_{S\subseteq R} \sum_{T \subseteq \fset\setminus S} \frac{(\fnum-t-s)!t!}{(\fnum-s+1)!} \mathbf{1}_{R\setminus S \subseteq T} 
        \\
        &=\mathbf{1}_{S\subseteq R} \sum_{t=q}^{\fnum-s} \frac{(\fnum-t-s)!t!}{(\fnum-s+1)!} \binom{\fnum-s-q}{t-q}
        &&= \mathbf{1}_{S\subseteq R} \sum_{t=0}^{\fnum-s-q} \frac{(t+q)!}{(\fnum-s+1)!} \frac{(\fnum-s-q)!}{t!} 
        \\
        &=\mathbf{1}_{S\subseteq R} \frac{(\fnum-s-q)!}{(\fnum-s+1)!}\sum_{t=0}^{\fnum-s-q} \frac{(t+q)!}{q!}
        &&=\mathbf{1}_{S\subseteq R} \frac{(\fnum-s-q)!}{(\fnum-s+1)!} q!\sum_{t=0}^{\fnum-s-q} \binom{t+q}{q}
        \\
        &=\mathbf{1}_{S\subseteq R} \frac{(\fnum-s-q)!}{(\fnum-s+1)!} q!\sum_{t=0}^{\fnum-s-q} \binom{t+q}{q}
        &&= \mathbf{1}_{S\subseteq R} \frac{(\fnum-s-q)!}{(\fnum-s+1)!} q! \binom{\fnum-s+1}{\fnum-s-q}
        \\
        &= \mathbf{1}_{S\subseteq R} \frac{1}{q+1} 
        &&= \frac{ \mathbf{1}_{S\subseteq R}}{r-s+1},
    \end{align*}
    where we used the hockey-stick identity in the second last row to compute the sum and $q=r-s$, if $S \subseteq R$.
 \end{proof}

\subsection{Intuition about the KernelSHAP-IQ Weights}\label{sec_appendix_intuition}

The intuition behind $\mathbf{W}_k$ is based on common factors in two representations of SII: One being the solution of the WLS problem from \cref{eq_WLS_solution}
\begin{equation*}
    \phi_k^{\text{SII}} = (\mathbf{X}^T_k \mathbf{W}_k \mathbf{X}_k)^{-1} \mathbf{X}_k^T \mathbf{W}_k \cdot \mathbf{y}_k,
\end{equation*} 
and the other being the representation of \cref{conj_sii} and by \citet{Fumagalli.2023}
\begin{align*}
    \phi_k^{\text{SII}} = \mathbf{Q}_k \cdot \mathbf{y}_k.
\end{align*}

In detail, if we consider the final weight of a single subset $T \subseteq N$ with $k \leq \vert T \vert \leq n-k$ and all interactions $S\subseteq N$ of order $\vert S \vert=k$, then the common factors for all $S$ are according to the $S$-row and $T$-column of $\mathbf{Q}_k$, i.e. $(-1)^{k-\vert T \cap S\vert} m_k(t-\vert T \cap S\vert)$, determined by the common factors in $m_k(t-0),\dots,m_k(t-k)$. By definition of $m_k(t) = \frac{(n-k-t)!t!}{(n-k+1)!}$, we have $m_k(t-\ell) = \frac{(n-k-t+\ell)!(t-\ell)!}{(n-k+1)!}$ and thus the common factors in $\mathbf{Q}_k$ for all $S$-rows given a $T$-column are $\frac{(n-k-t)!(t-k)!}{(n-k+1)!} \propto \binom{n-2k}{t-k}^{-1}$. On the other hand, by \cref{eq_WLS_solution}, the matrix $(\mathbf{X}^T_k \mathbf{W}_k \mathbf{X}_k)^{-1}$ is independent of the subset $T$, and hence $\mu_k(t)$, as it contains one row and one column for each interaction $S$ of order $k$. This implies 
\begin{equation*}
   ((\mathbf{X}^T_k \mathbf{W}_k \mathbf{X}_k)^{-1} \mathbf{X}^T_k \mathbf{W}_k)_{S T} \propto \mu_k(t) 
\end{equation*}
for all interactions $S$ of order $k$. Combining both proportionality results, we conclude that $\mu_k(t) \propto \binom{n-2k}{t-k}^{-1}$ is a suitable candidate. Note that this reasoning works for subsets with $k \leq \vert T \vert \leq n-k$.

\clearpage
\section{Experimental Setup and Reproducibility}
\label{sec_appendix_experimental_setup}

This section contains additional information about the setup and reproducibility of our empirical evaluation.
\cref{sec_appendix_task_description} contains additional information about the models, datasets and explanation tasks used for evaluation as summarized in \cref{tab_setup}. \cref{sec_appendix_computational_cost} contains additional information about the environmental impact and computational effort of the empirical evaluation.
All experiments can be reproduced via the technical supplement available at \repository which will be made fully publicly available.

\subsection{Model, Datasets and Task Descriptions}
\label{sec_appendix_task_description}

This section contains detailed information about the models and datasets used for benchmarking the approximation quality of KernelSHAP-IQ and the available baseline algorithms.
To increase the comparability, we conduct our empirical evaluation of the approximation quality on settings and tasks presented in the literature \cite{Kolpaczki.2024b,Fumagalli.2023,Muschalik.2024,Sundararajan.2020,Tsai.2022}. 
For further questions regarding the technical details, we further refer to the technical supplement at \repository which will be made fully publicly available.

\subsubsection{Sum Of Unanimity Models (SOUM)}
The sum of unanimity models (SOUMs) are used in different empirical evaluations of Shapley-based interaction estimators \cite{Fumagalli.2023,Kolpaczki.2024b,Tsai.2022}. 
This synthetic model class can be used to create explanation tasks with differing complexity while allowing to compute GT interaction scores analytically even for higher player counts.
For a given player set $\fset := \{1,\dots,\fnum\}$ with $\fnum$ many players, we sample $M=50$ interactions $R_1,\ldots,R_M \subseteq \fset$ from $\mathcal{P}(\fset)$.
Next, we draw for each interaction subset $R_m$ a coefficient $a_m \in [0,1]$ uniformly at random.
The value function is then defined as $\nu(T) = \sum_{m=1}^M a_m \cdot  \nu_{R_m}(T)$ for all coalitions $T \subseteq \fset$ with the unanimity game $\nu_{R_m}(T) := \mathbf{1}_{T \supseteq R_m}$ Similar to \citet{Tsai.2022}, we limit the highest interaction size $|R_m| \leq 4$. 
Further, we randomly set two features to be non-informative (dummy player) which are never part of any interaction subset.

\subsubsection{Language Model (LM)}
Language models (LMs) are widely investigated with Shapley interactions \cite{Fumagalli.2023,Kolpaczki.2024b,Tsai.2022,Sundararajan.2020} highlighting that explanations based on interactions are more expressive than word-level explanations.
Similar to the related work, we also investigate KernelSHAP-IQ's approximation quality in a LM scenario.
The LM used in our experiments is a pre-trained \texttt{DistilBert} \cite{Sanh.2019} sentiment analysis model.
The LM used is available at \url{https://huggingface.co/lvwerra/distilbert-imdb} via the \texttt{transformers} API \cite{Wolf.2020}.
The transformer was pre-trained on the IMDB movie review dataset \cite{Maas.2011,Lhoest.2021}.
Provided a tokenized text input, the model predicts the sentiment of the input with a sentiment score ranging from negative $-1$ to positive $1$.
For explanation purposes, we remove features (tokens) from the tokenized representation of the input.

\subsubsection{Convolutional Neural Network (CNN)}
Similar to \citet{Fumagalli.2023,Kolpaczki.2024b}, we evaluate the approximation quality of KernelSHAP-IQ and the available baselines on a convolutional neural network (CNN) image-input explanation task. 
The CNN is a \texttt{ResNet18} \cite{He.2018} fitted on the ImageNet \cite{Deng.2009} dataset.
The CNN is available at \url{https://pytorch.org/vision/main/models/generated/torchvision.models.resnet18.html} via \texttt{pytorch} \cite{Paszke.2017}.
The task is to explain the predicted label's probability for random ImageNet images.
Individual pixels are grouped together as superpixels with SLIC \cite{Achanta.2012}. 
Missing features (superpixels) are set to the gray as a method of mean-imputation.

\subsubsection{Vision Transformer (ViT)}
The vision transformer (ViT) setting presented in \citet{Kolpaczki.2024b} is similar to the CNN as the task is to explain random ImageNet images in terms of the predicted class probability.
The specific ViT model operates on ``words'' of 32x32 pixels image patches.
The ViT model is available at \url{https://huggingface.co/google/vit-base-patch32-384} via the \texttt{transformers} API \cite{Wolf.2020}.
To create a player set of 16 players (computation of GT SII values already requires $2^{16} = 65\,536$ coalitions) the smaller 32x32 image patches are grouped together into 96x96 pixels super-patches (3x3 original sized patches make up one larger patch).
Similar to the LM, absent players (patches) are removed on the tokenized representation of the input.
The value of a coalition is the ViT's predicted class probability for the class which has the highest probability provided the grand coalition (the unmodified image with no patches removed). 

\subsubsection{Bike Rental (BR)}
Based on \citet{Muschalik.2024}, the bike rental (BR) setting is based on the \emph{bike} regression dataset \cite{FanaeeT.2014}, where the goal is to predict the amount of rented bikes based on features like weather conditions or time of day. 
The dataset is retrieved from \texttt{openml} \cite{Feurer.2020} with \emph{42712} as the dataset identifier.
We encode categorical features with ordinal values and scale numeric features. 
We further logarithmize the target variable to base ten.
Based on this dataset, we train an XGBoost model \cite{Chen.2016} and explain randomly sampled local instances by removing numerical and categorical features through mean and mode imputation, respectively.

\subsubsection{California Housing (CH)}
Like \citet{Muschalik.2024}, we retrieve the California housing (CH) dataset \cite{Kelley.1997} through \texttt{scikit-learn} library \cite{Pedregosa.2011}.
The CH regression dataset contains information about property prices (the continuous target variable) and corresponding property attributes including the location in terms of latitude and longitude.
We standardize all features and, likewise to the BR dataset, logarithmize the target variable to base ten.
We train a small NN with \texttt{pytorch} on the regression task and explain randomly sampled local instances the model by removing missing features through mean imputation.
For the illustration in \cref{fig_intro_illustration}, we fit a gradient boosted tree with \texttt{scikit-learn}.

\subsubsection{Adult Census (AC)}
The adult census (AC) classification dataset \cite{Kohavi.1996} contains socio-demographic features of individuals paired with their income levels.
Similar to \citet{Muschalik.2024}, we retrieve the AC dataset via \texttt{openml} and \emph{1590} as the dataset identifier.
We transform the dataset by imputing missing numerical attributes with median values and then apply standard scaling.
Categorical features are encoded with ordinal values.
We train a random forest (RF) classifier from \texttt{scikit-learn} on this dataset and explain randomly sampled local instances by removing missing features through mean and mode imputation for numerical and categorical features, respectively.

\subsection{Computational Effort}
\label{sec_appendix_computational_cost}

The computational effort required for evaluating KernelSHAP-IQ in comparison to all available baselines is modest.
The main computational burden lies in the pre-computation of GT values and queries to underlying black box models.
The largest effort for pre-computation of GT values stems from the ViT with $d = 16$ players requiring $2^{16} = 65\,536$ queries to the ViT model for each input image. 
To alleviate this computational burden, we pre-compute the worth of each coalition for \emph{each} instance of \emph{each} benchmark task, except the synthetic SOUMs.
Hence, for each run of the different estimators the pre-computed coalition values can be looked up from a file.
The code for recreating or extending these look-up files is part of the technical supplement at \repository.
All benchmarks are performed on a single a Dell XPS 15 9510 Laptop with an Intel i7-11800H clocking at 2.30GHz.
The experiments consumed around $500$ CPU hours.

\clearpage
\section{Additional Empirical Results}
\label{sec_appendix_additional_results}

\subsection{Runtime Analysis}
Next to the experiments regarding approximation quality, we also conducted a small runtime analysis of KernelSHAP-IQ and all available baseline algorithms.
The results are summarized in \cref{fig_appendix_runtime}.
We run KernelSHAP-IQ, SHAP-IQ, SVARM-IQ, and permutation sampling on the LM and an example sentence consisting of 14 words.
We let the approximation algorithms estimate the SII values of order $l=2$.
The main computational cost results from the model evaluations (accesses to $\nu$), which is bounded by the model’s inference time.
From \cref{fig_appendix_runtime} it is apparent that with increasing the amount of allowed model accesses, the runtime scales linearly for all approximation algorithms and all estimators perform equally.
Further, we evaluate the runtime of the different approximation methods jointly with their performance.
\cref{app_tab_runtime} shows the mean (averaged over 10 runs) runtime of each approximator to reach certain MSE levels.
For \emph{CH}, KernelSHAP-IQ and Inconistent KernelSHAP-IQ are most efficient in terms of model calls to reach the desired MSE error levels.
The runtime is similar for the other baselines.
While Inconsistent KernelSHAP-IQ fails to reach the desired error levels for \emph{LM}, KernelSHAP-IQ performs equally well with SVARM-IQ for \emph{LM}.
Permutation sampling generally has a fast runtime, but does not achieve good estimation qualities in sensible time.

\begin{figure}[htb]
    \centering
    \includegraphics[width=0.45\textwidth]{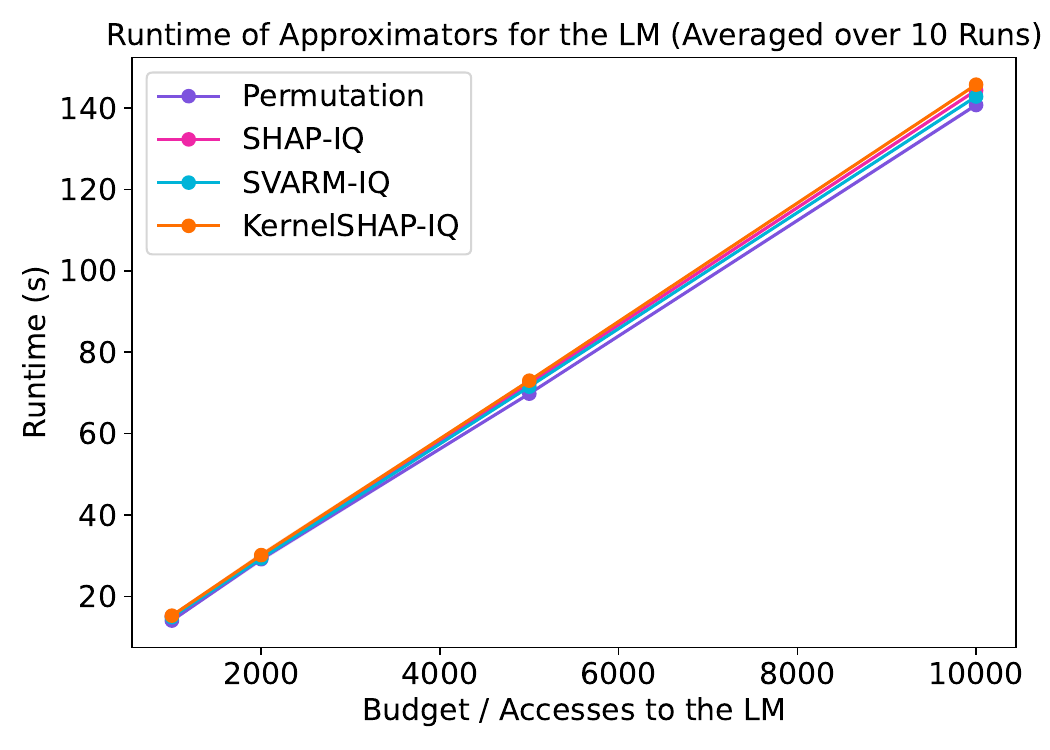}
    \caption{Runtime analysis of KernelSHAP-IQ and baseline algorithms for calculating $l=2$ SII scores for an example sentence with $\fnum=14$ words and the LM. For each approximator we evaluate 10 independent runs. The shaded bands corresponds to the SEM.}
    \label{fig_appendix_runtime}
\end{figure}

\begin{table}[htb]
\centering
\caption{Mean runtime of each approximator to reach certain MSE error levels in terms of model calls and elapsed time in seconds. The runtime is averaged over 10 independent runs for two benchmarks. For \emph{LM}, Inconsistent KernelSHAP-IQ never reaches the MSE error levels and Permutation requires more than $2^{14}$ evaluations.}
\label{app_tab_runtime}
\resizebox{\columnwidth}{!}{%
\begin{tabular}{@{}llllllll@{}}
\toprule
\textbf{Benchmark} & \textbf{Error} & \textbf{Metric} & \textbf{KernelSHAP-IQ} & \textbf{Inc. KernelSHAP-IQ} & \textbf{SVARM-IQ} & \textbf{SHAP-IQ} & \textbf{Permutation} \\ \midrule
\multirow{4}{*}{\emph{CH}} & \multirow{2}{*}{\textit{MSE at $\text{2e}^{\text{-3}}$}} & \textit{Model Calls} & 75 & 50 & 130 & 180 & 195 \\
 &  & \textit{Time (s)} & 0.023 & 0.012 & 0.035 & 0.033 & 0.024 \\
 & \multirow{2}{*}{\textit{MSE at $\text{1e}^{\text{-3}}$}} & \textit{Model Calls} & 85 & 70 & 170 & 240 & 600 \\ 
 &  & \textit{Time (s)} & 0.026 & 0.017 & 0.043 & 0.037 & 0.071 \\ \midrule
\multirow{4}{*}{\emph{LM}} & \multirow{2}{*}{\textit{MSE at $\text{1e}^{\text{-3}}$}} & \textit{Model Calls} & 2800 & - & 2800 & 5500 & $>$16384 \\
 &  & \textit{Time (s)} & 35.5 & - & 35.6 & 95.9 & - \\
 & \multirow{2}{*}{\textit{MSE at $\text{5e}^{\text{-4}}$}} & \textit{Model Calls} & 4000 & - & 4000 & 7200 & $>$16384 \\
 &  & \textit{Time (s)} & 70.2 & - & 71.6 & 126.1 & - \\ \bottomrule
\end{tabular}%
}
\end{table}

\subsection{Validations of Higher-Order Conjecture}
In the following, we empirically validate \cref{conj_inverse} and \cref{conj_sii}.
We let the numbers of players be $\fnum=2,\dots,11$ and the order of interactions $k=1,\dots,\lfloor \fnum/2 \rfloor$, since the conjectures hold for $\fnum\geq 2k$.

\paragraph{Validation of \cref{conj_inverse}}
We let $\mu_\infty = 10^7$, generate the matrices $\mathbf{X}_k$ and $\mathbf{W}_k$ and compute the inverse $\mathbf{A}_k$ using standard \emph{numpy} functions.
We then compare the results with the proposed inverse from \cref{conj_inverse}.
Lastly, we compute the MSE of all elements and assert that this error is less than $10^{-10}$.

\paragraph{Validation of \cref{conj_sii}}
To validate \cref{conj_sii}, we randomly generate $10$ instances of SOUMs containing $1000$ randomly generated interactions, where the size of each interaction is uniformly distributed.
We compute the GT SIIs for these games and compare them with the empirically computed scores via \cref{conj_sii}.
Lastly, we average the MSE over all SOUM instances and assert again that this error is below $10^{-10}$.

\subsection{Additional Approximation Results}
This section contains additional experimental results and evaluations. 
\cref{fig_appendix_ch,fig_appendix_bike} contain additional plots to the experiments conducted on the CH and BR regression tasks, respectively.
KernelSHAP-IQ and inconsistent KernelSHAP-IQ substantially outperform SVARM-IQ, SHAP-IQ, and permutation sampling, while inconsistent KernelSHAP-IQ does not converge to the GT.
\cref{fig_appendix_cnn,fig_appendix_vit} show the estimation qualities for the image-related tasks.
On the CNN, KernelSHAP-IQ outperforms all existing baselines and inconsistent KernelSHAP-IQ differs greatly from the GT SII values.
Interestingly, on the ViT, SVARM-IQ outperforms KernelSHAP-IQ and inconsistent KernelSHAP-IQ for estimating second order SII scores while both kernel-based estimators retrieve better first and third order SII values.
\cref{fig_appendix_ac} shows additional estimation results on the AC classification dataset.
On AC, KernelSHAP-IQ, again, achieves state-of-the-art approximation results.
Lastly, \cref{fig_appendix_lm} shows a detailed view of the experiment conducted on the LM highlight the estimation quality of each order individually.
The results on the LM show that KernelSHAP-IQ and SVARM-IQ both retrieve equally good estimates for order 2 while KernelSHAP-IQ slightly outperforms SVARM-IQ on order 1 and 3.
Lastly, \cref{fig_appendix_soum} shows the results on the synthetic SOUMs with higher number of players.
In both settings inconsistent KernelSHAP and KernelSHAP clearly outperform SVARM-IQ, SHAP-IQ, and permutation sampling. 
Notably, in these higher player settings the inconsistent version's drawback of not converging to the GT is yet to materialize.

\begin{figure}[htb]
    \centering
    \includegraphics[width=0.75\textwidth]{figures/legend_horizontal_all_orders.pdf}
    \\
    \begin{minipage}[c]{0.07\textwidth}
    \centering SII\\$l=2$
    \end{minipage}
    \begin{minipage}[c]{0.32\textwidth}
    \includegraphics[width=\textwidth]{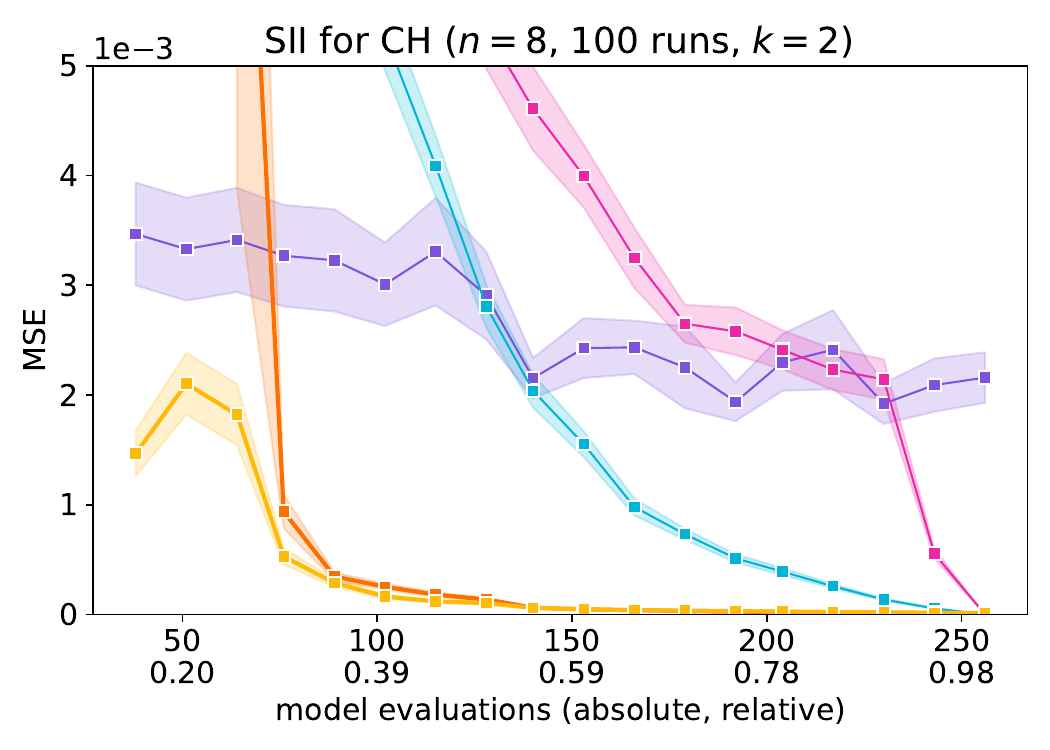}
    \end{minipage}
    \begin{minipage}[c]{0.32\textwidth}
    \includegraphics[width=\textwidth]{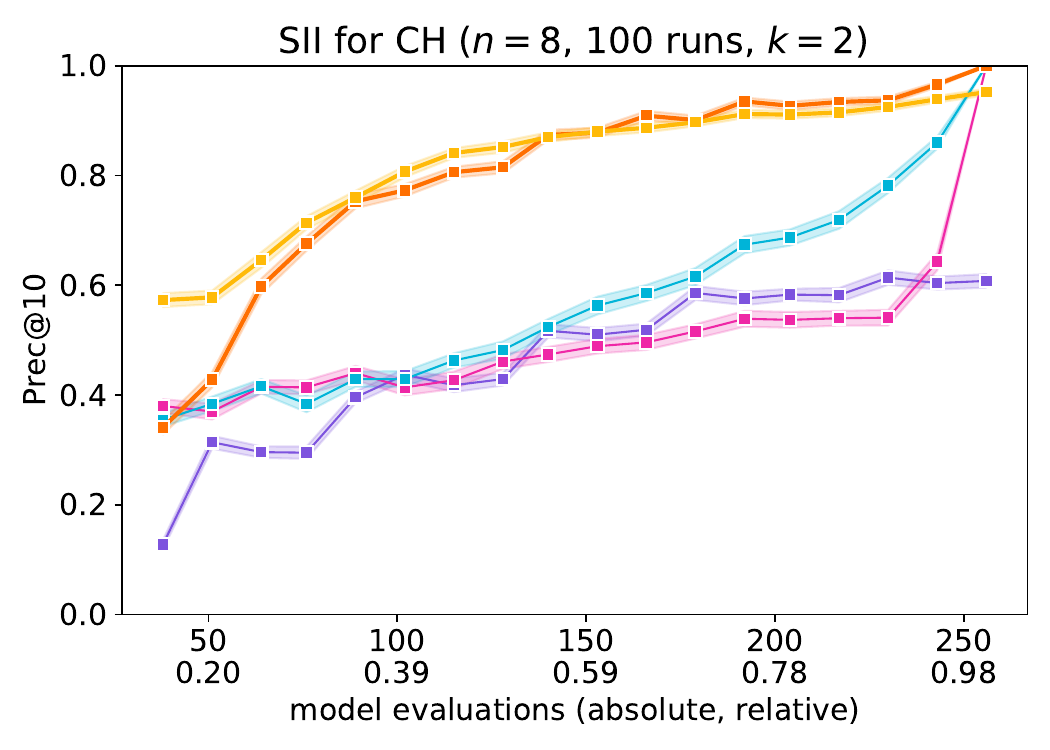}
    \end{minipage}
    \\
    \begin{minipage}[c]{0.07\textwidth}
    \centering$2$-SII\\$l=2$
    \end{minipage}
    \begin{minipage}[c]{0.32\textwidth}
    \includegraphics[width=\textwidth]{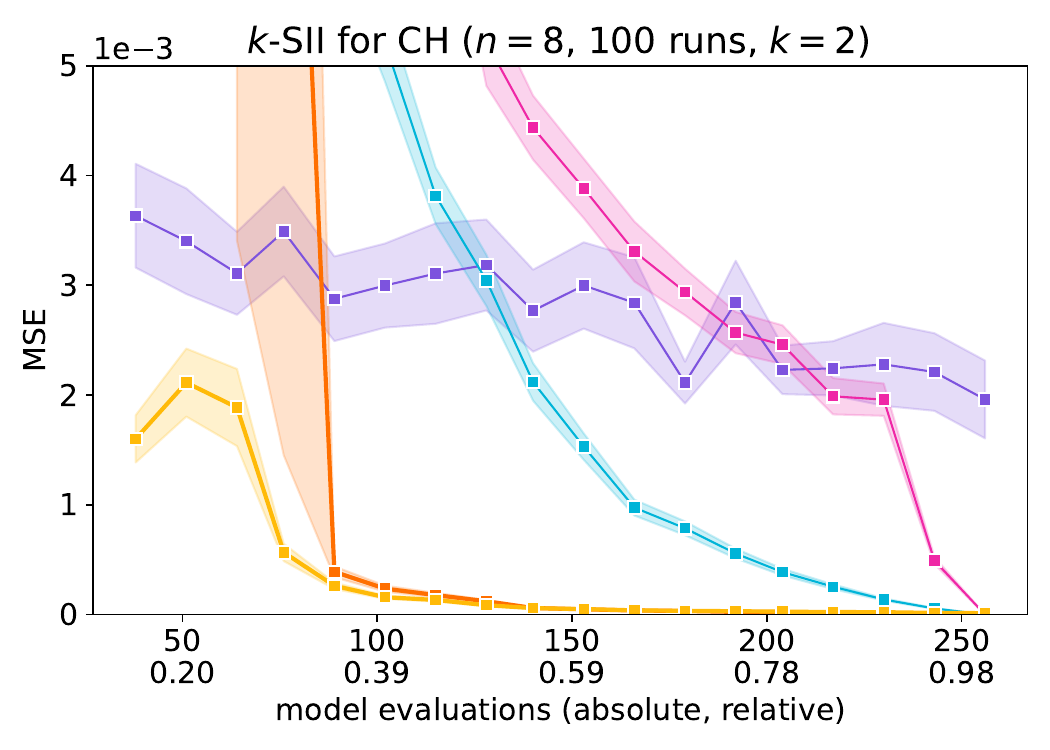}
    \end{minipage}
    \begin{minipage}[c]{0.32\textwidth}
    \includegraphics[width=\textwidth]{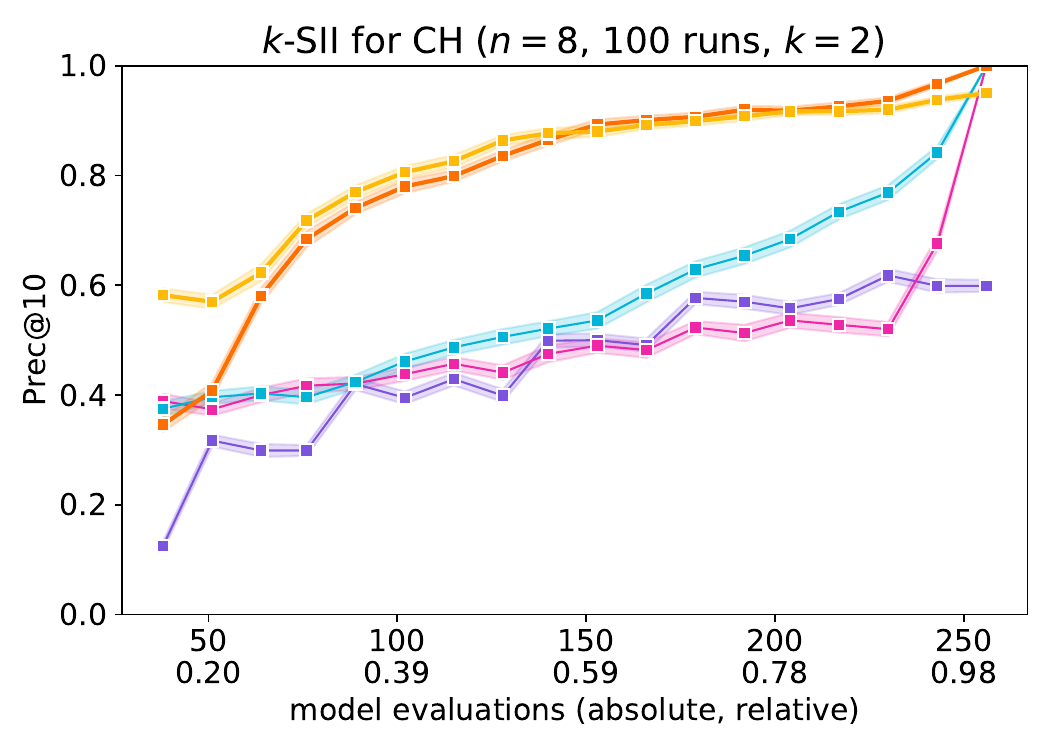}
    \end{minipage}
    \caption{SII (top) and $k$-SII (bottom) approximation quality in terms of MSE (left) and Prec@10 (right) of the CH dataset.}
    \label{fig_appendix_ch}
\end{figure}

\begin{figure}[htb]
    \centering
    \includegraphics[width=0.75\textwidth]{figures/legend_horizontal_all_orders.pdf}
    \\
    \begin{minipage}[c]{0.07\textwidth}
    $l=2$
    \end{minipage}
    \begin{minipage}[c]{0.32\textwidth}
    \includegraphics[width=\textwidth]{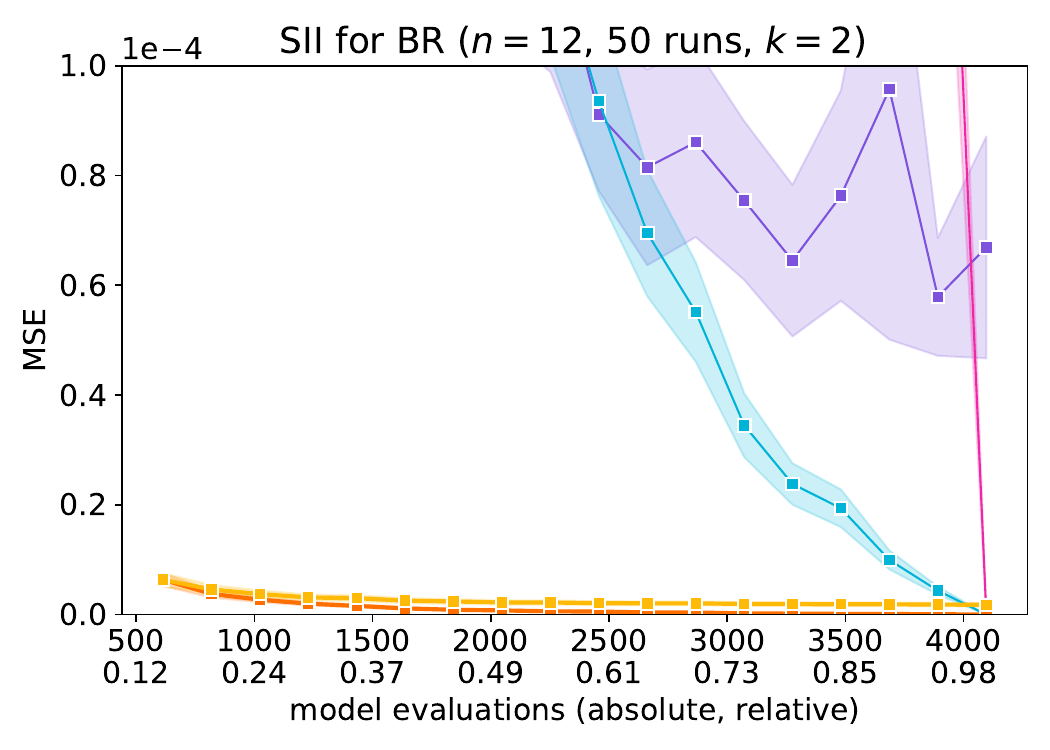}
    \end{minipage}
    \begin{minipage}[c]{0.32\textwidth}
    \includegraphics[width=\textwidth]{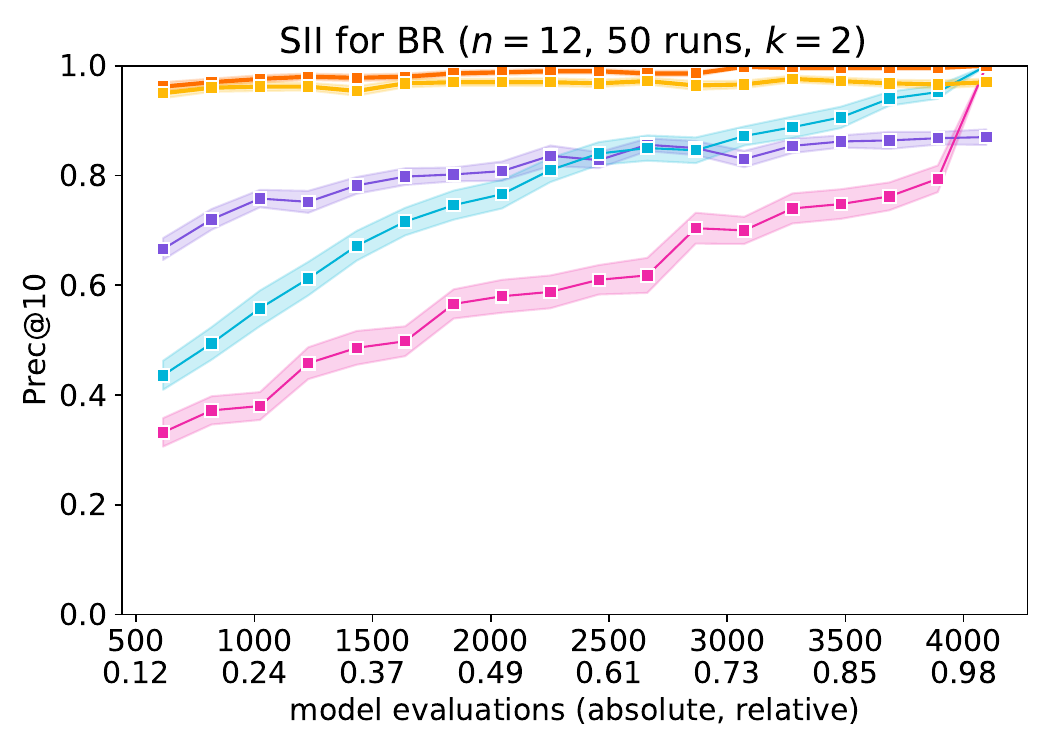}
    \end{minipage}
    \\
    \begin{minipage}[c]{0.07\textwidth}
    $l=3$
    \end{minipage}
    \begin{minipage}[c]{0.32\textwidth}
    \includegraphics[width=\textwidth]{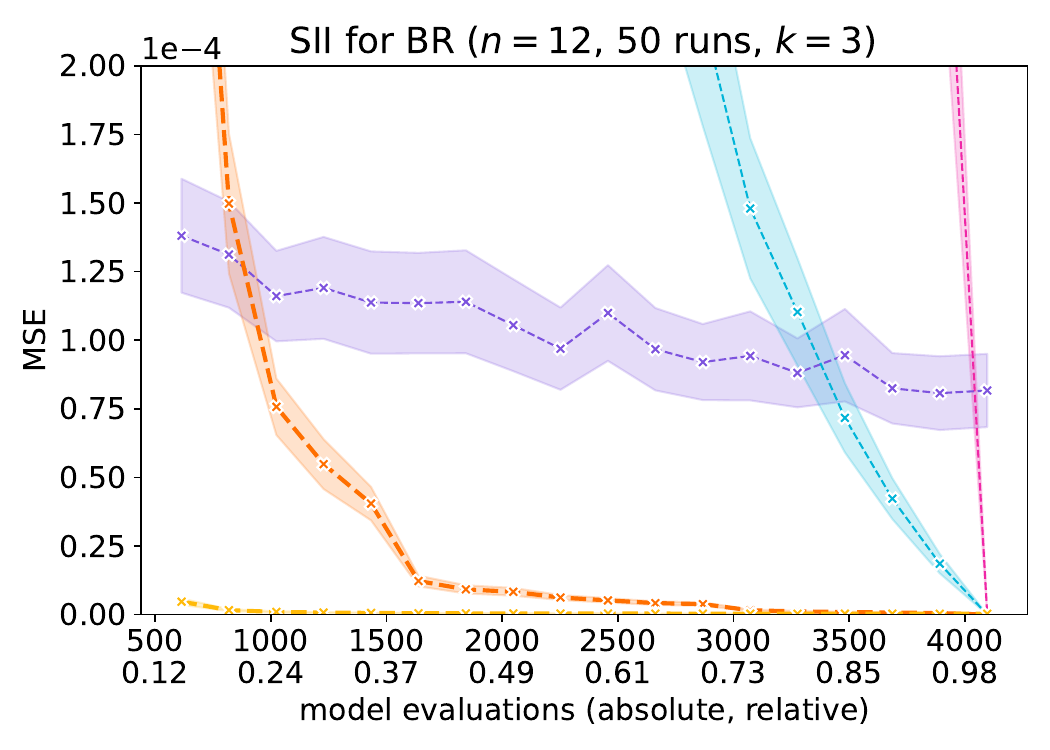}
    \end{minipage}
    \begin{minipage}[c]{0.32\textwidth}
    \includegraphics[width=\textwidth]{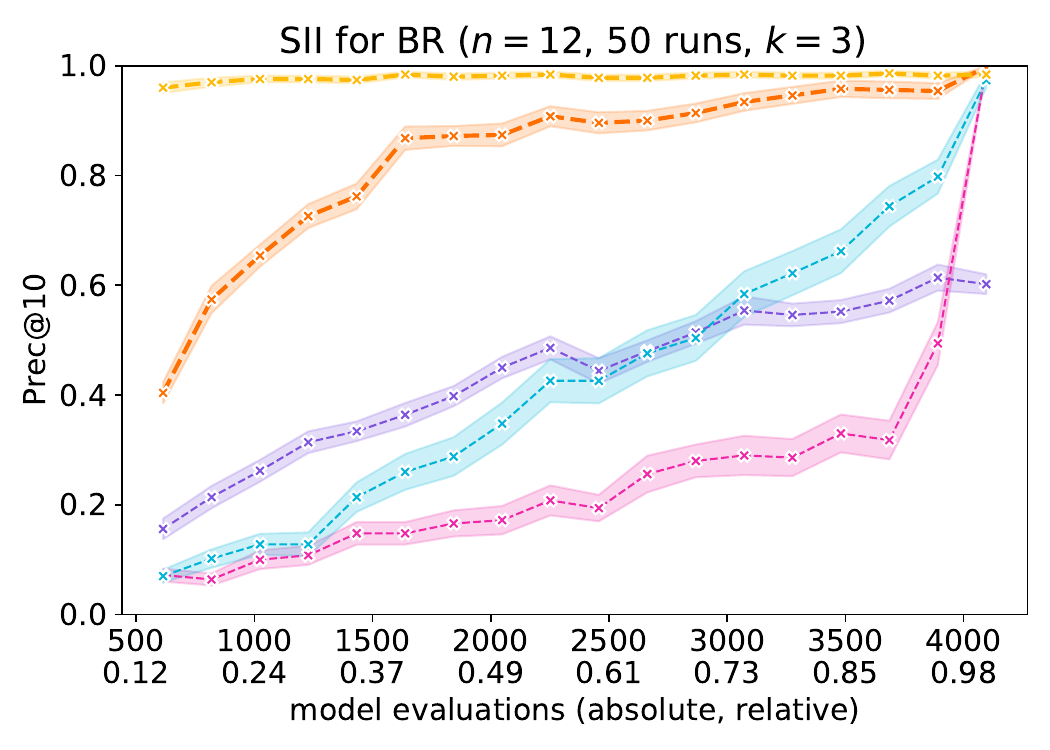}
    \end{minipage}
    \caption{Approximation quality in terms of MSE (left) and Prec@10 (right) of the BR dataset.}
    \label{fig_appendix_bike}
\end{figure}

\begin{figure}[htb]
    \centering
    \includegraphics[width=0.75\textwidth]{figures/legend_horizontal_all_orders.pdf}
    \\
    \begin{minipage}[c]{0.07\textwidth}
    $l=2$
    \end{minipage}
    \begin{minipage}[c]{0.32\textwidth}
    \includegraphics[width=\textwidth]{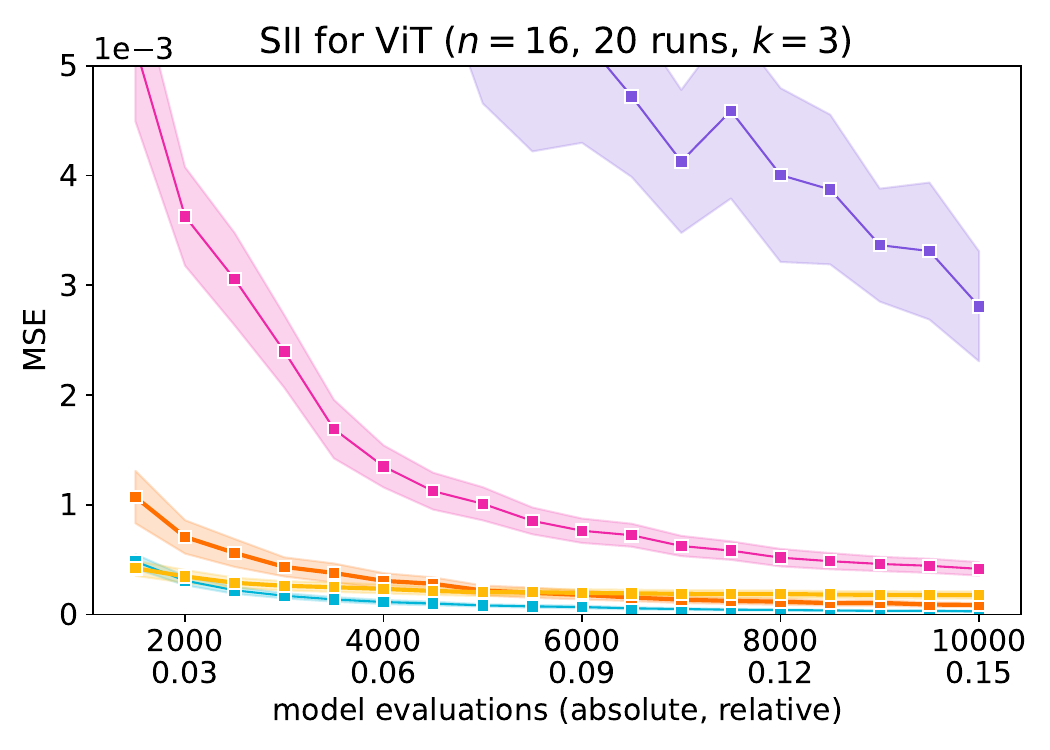}
    \end{minipage}
    \begin{minipage}[c]{0.32\textwidth}
    \includegraphics[width=\textwidth]{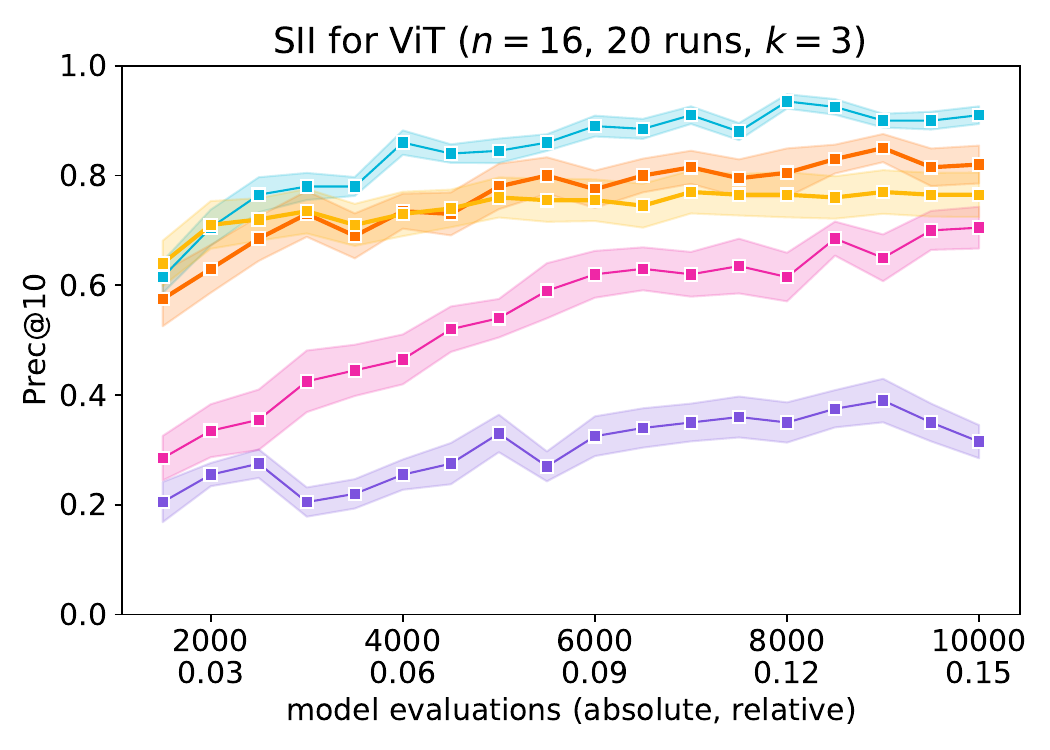}
    \end{minipage}
    \\
    \begin{minipage}[c]{0.07\textwidth}
    $l=3$
    \end{minipage}
    \begin{minipage}[c]{0.32\textwidth}
    \includegraphics[width=\textwidth]{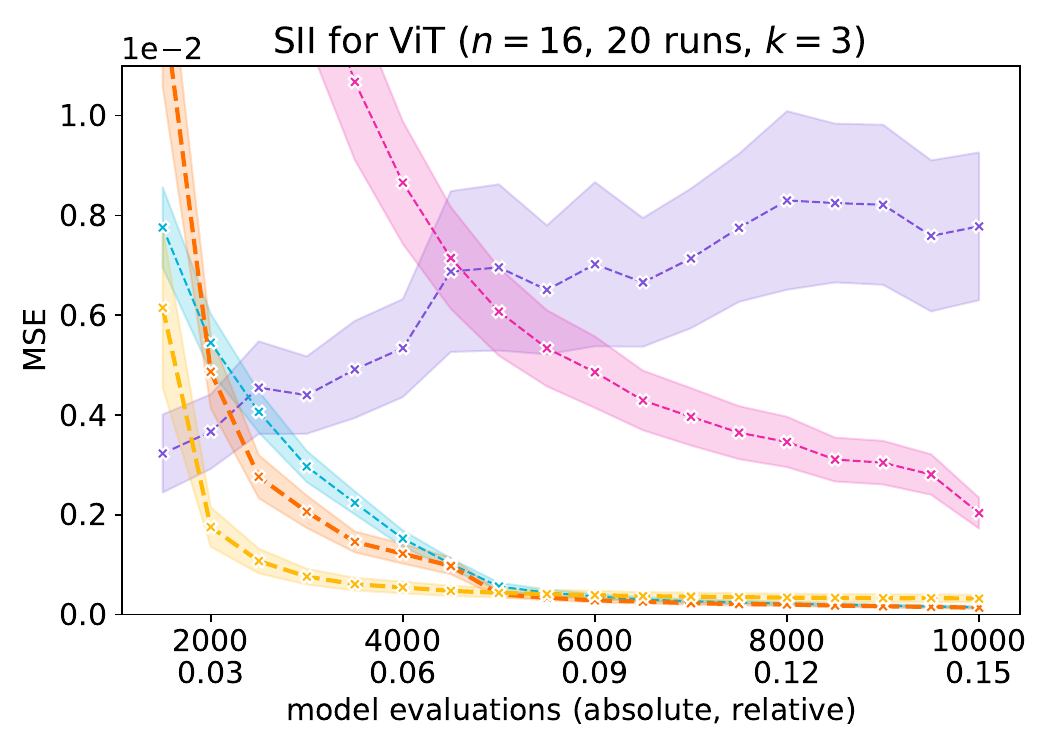}
    \end{minipage}
    \begin{minipage}[c]{0.32\textwidth}
    \includegraphics[width=\textwidth]{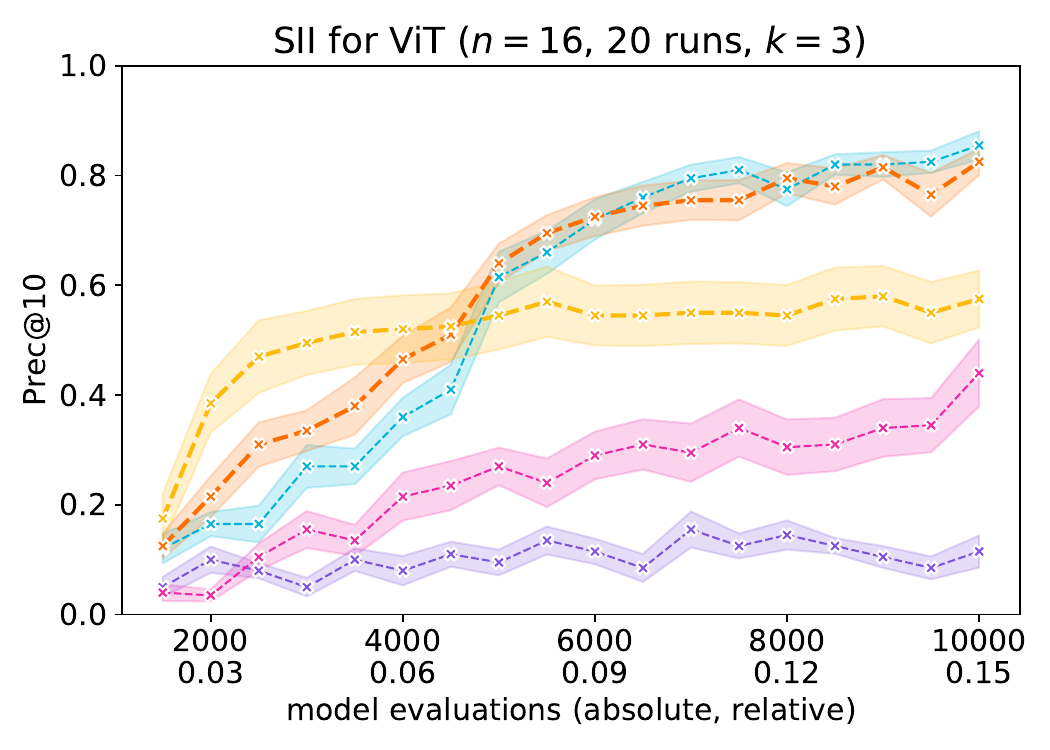}
    \end{minipage}
    \caption{Approximation quality in terms of MSE (left) and Prec@10 (right) of the ViT.}
    \label{fig_appendix_vit}
\end{figure}

\begin{figure}[htb]
    \centering
    \includegraphics[width=0.75\textwidth]{figures/legend_horizontal_all_orders.pdf}
    \\
    \begin{minipage}[c]{0.07\textwidth}
    $l=2$
    \end{minipage}
    \begin{minipage}[c]{0.32\textwidth}
    \includegraphics[width=\textwidth]{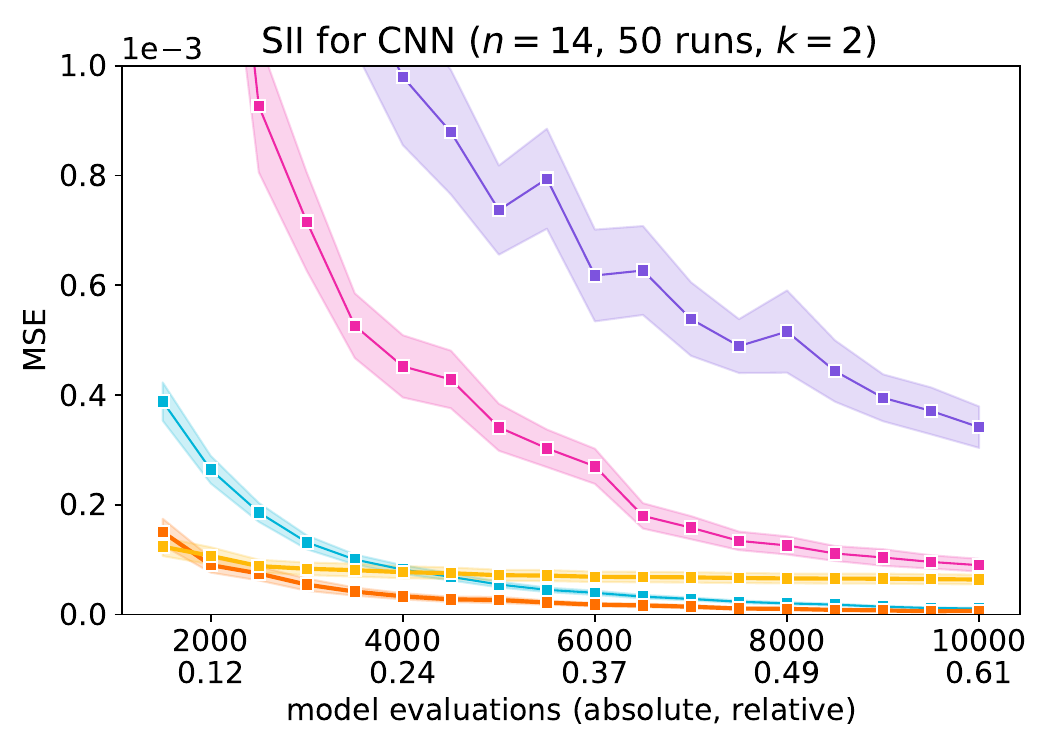}
    \end{minipage}
    \begin{minipage}[c]{0.32\textwidth}
    \includegraphics[width=\textwidth]{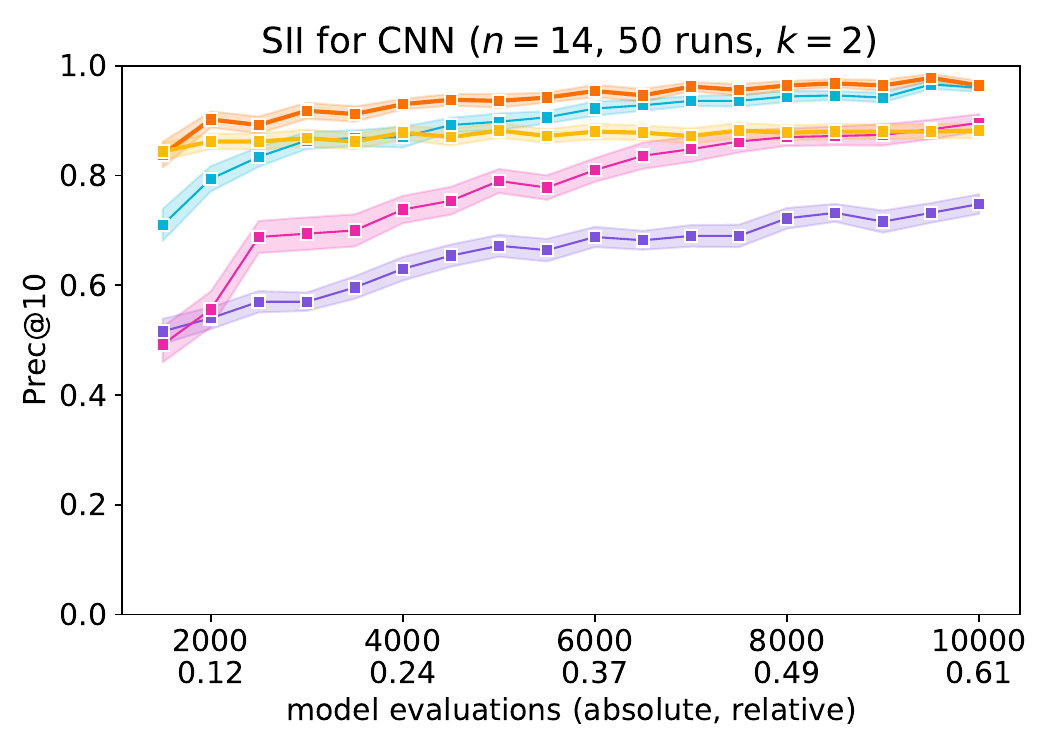}
    \end{minipage}
    \\
    \begin{minipage}[c]{0.07\textwidth}
    $l=3$
    \end{minipage}
    \begin{minipage}[c]{0.32\textwidth}
    \includegraphics[width=\textwidth]{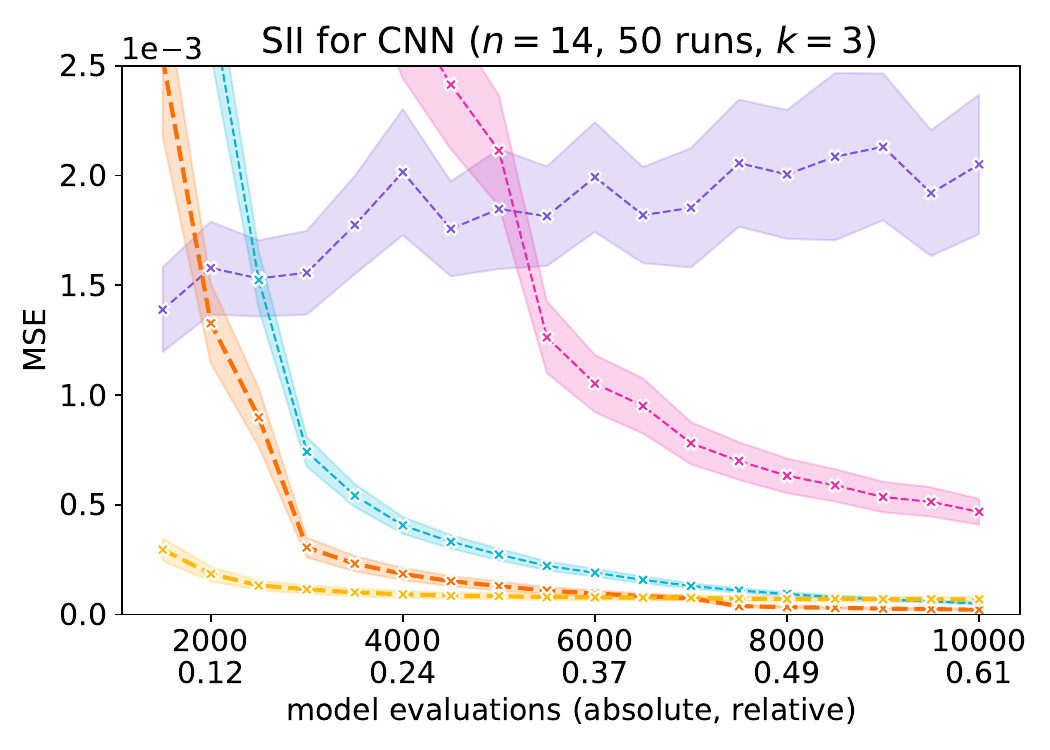}
    \end{minipage}
    \begin{minipage}[c]{0.32\textwidth}
    \includegraphics[width=\textwidth]{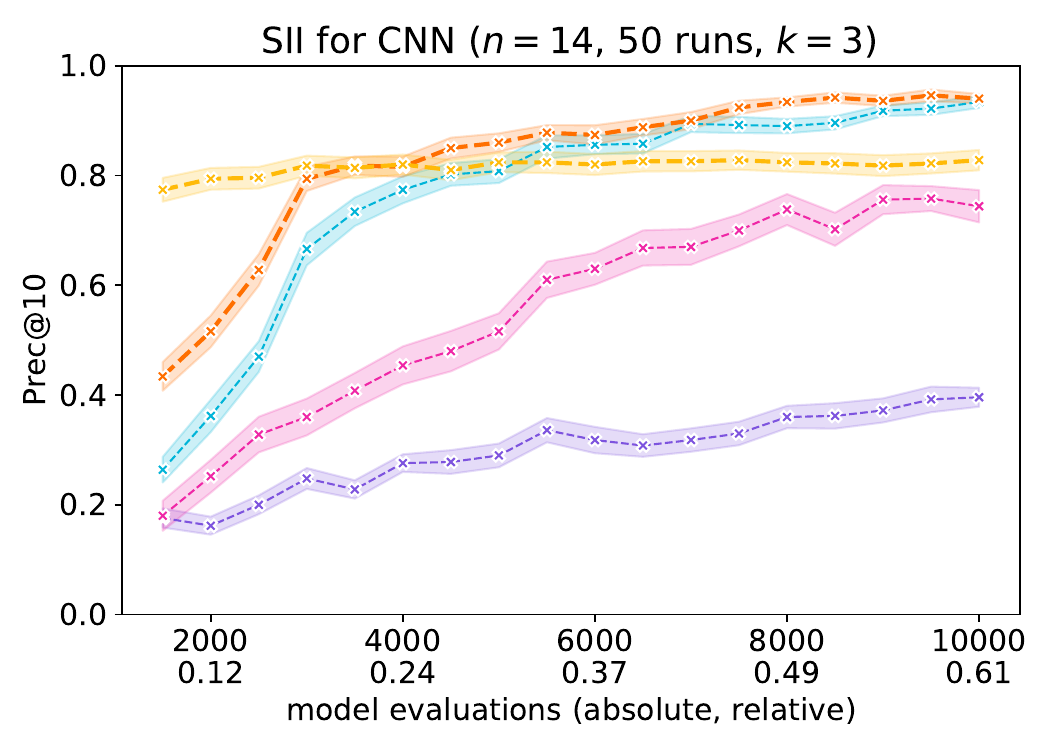}
    \end{minipage}
    \caption{Approximation quality in terms of MSE (left) and Prec@10 (right) of the CNN.}
    \label{fig_appendix_cnn}
\end{figure}

\begin{figure}[htb]
    \centering
    \includegraphics[width=0.75\textwidth]{figures/legend_horizontal_all_orders.pdf}
    \\
    \begin{minipage}[c]{0.07\textwidth}
    $l=2$
    \end{minipage}
    \begin{minipage}[c]{0.32\textwidth}
    \includegraphics[width=\textwidth]{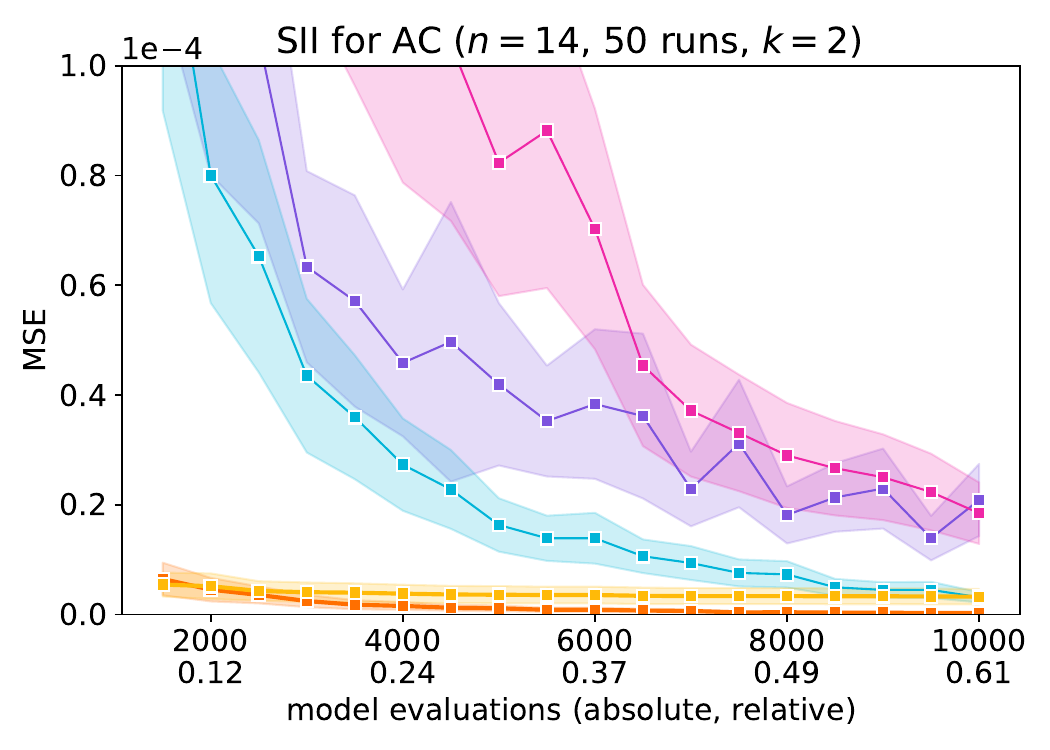}
    \end{minipage}
    \begin{minipage}[c]{0.32\textwidth}
    \includegraphics[width=\textwidth]{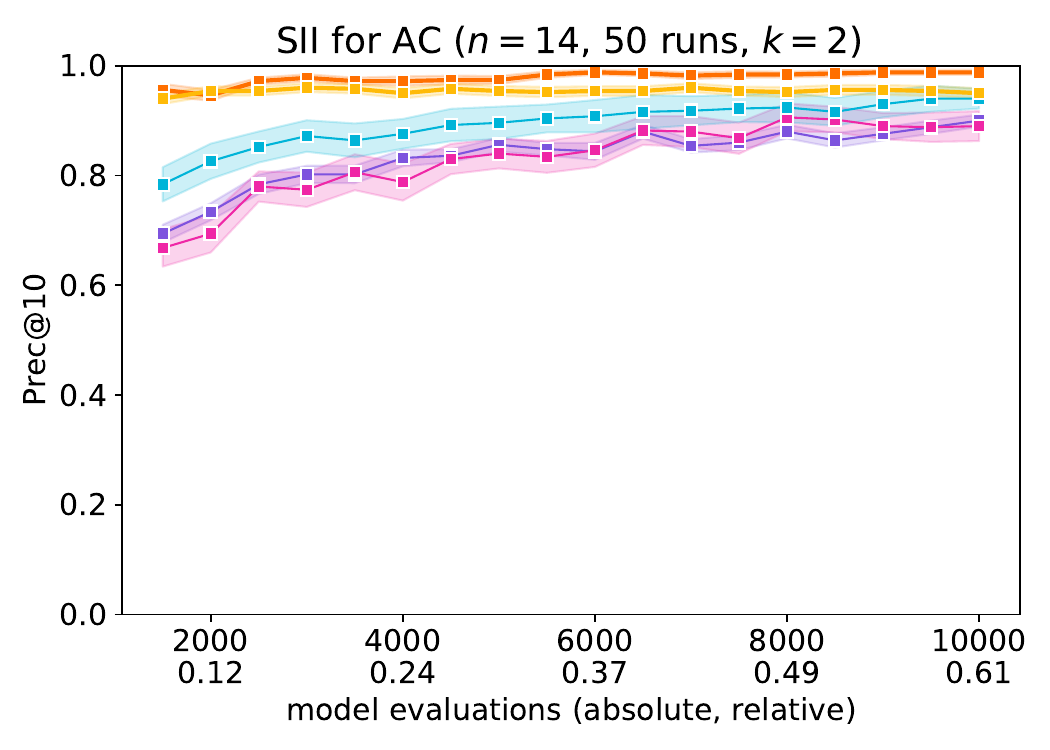}
    \end{minipage}
    \\
    \begin{minipage}[c]{0.07\textwidth}
    $l=3$
    \end{minipage}
    \begin{minipage}[c]{0.32\textwidth}
    \includegraphics[width=\textwidth]{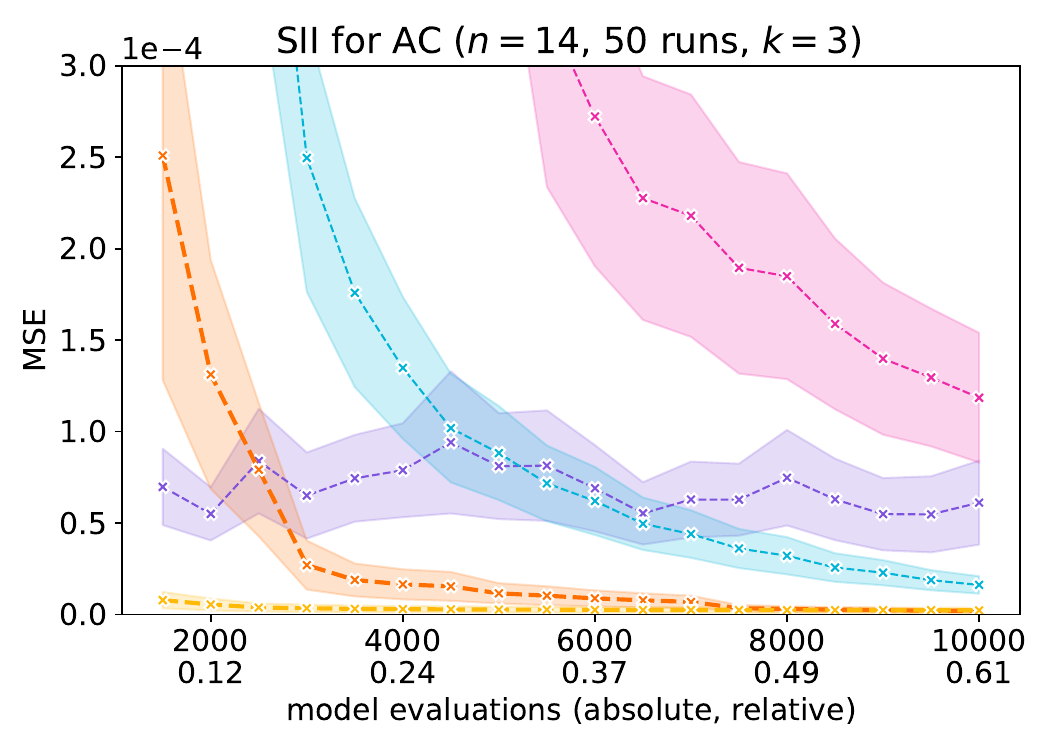}
    \end{minipage}
    \begin{minipage}[c]{0.32\textwidth}
    \includegraphics[width=\textwidth]{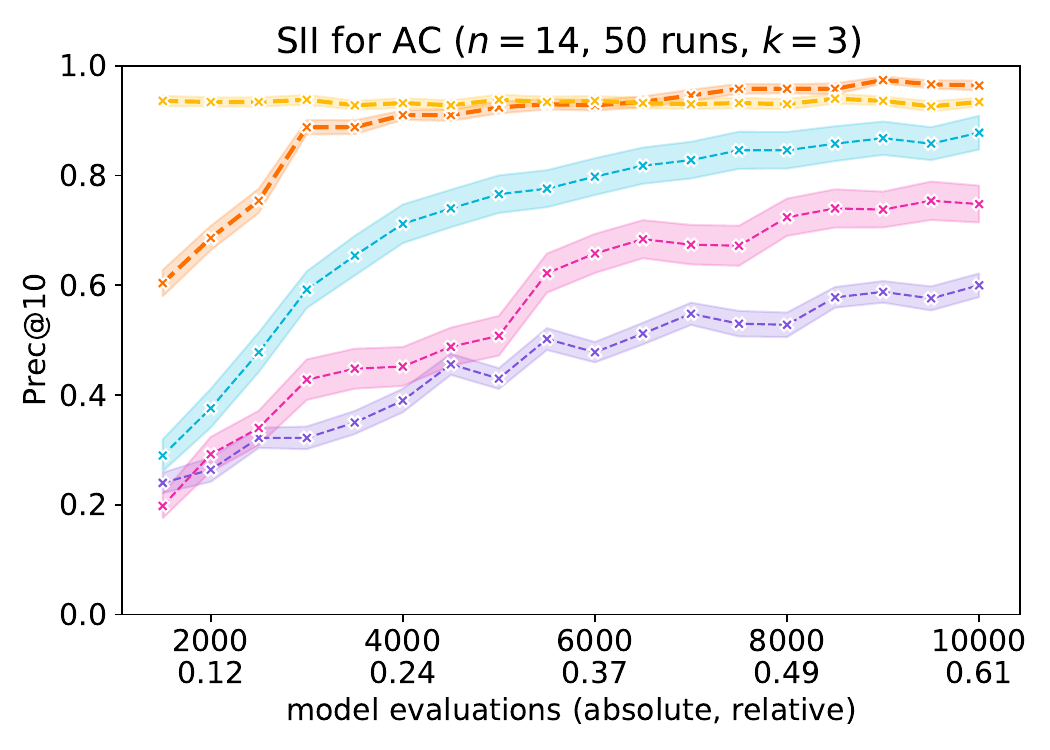}
    \end{minipage}
    \caption{Approximation quality in terms of MSE (left) and Prec@10 (right) of the AC dataset.}
    \label{fig_appendix_ac}
\end{figure}

\begin{figure}[htb]
    \centering
    \includegraphics[width=0.75\textwidth]{figures/legend_horizontal_all_orders.pdf}
    \\
    \begin{minipage}[c]{0.07\textwidth}
    $\fnum=20$
    \end{minipage}
    \begin{minipage}[c]{0.32\textwidth}
    \includegraphics[width=\textwidth]{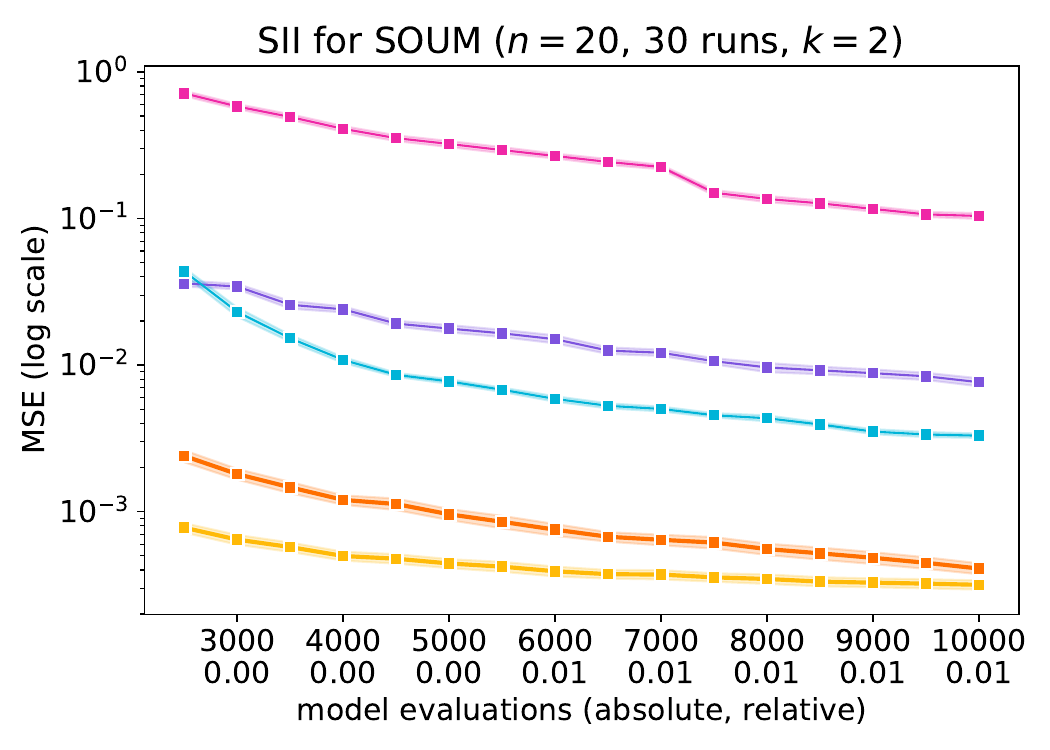}
    \end{minipage}
    \begin{minipage}[c]{0.32\textwidth}
    \includegraphics[width=\textwidth]{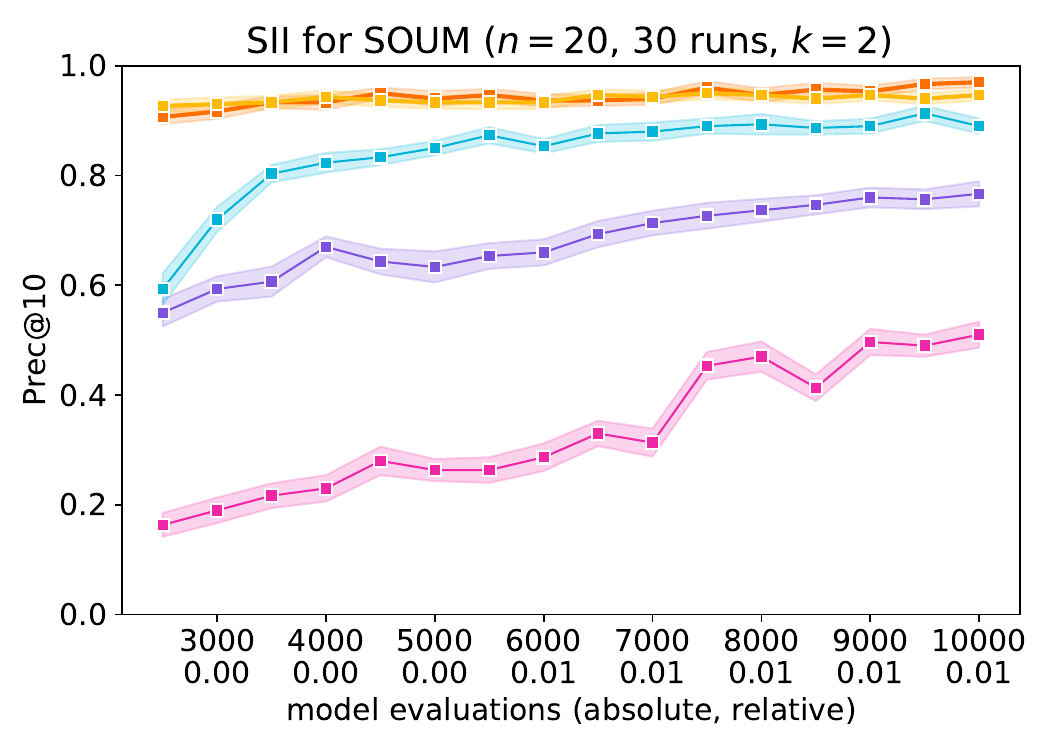}
    \end{minipage}
    \\
    \begin{minipage}[c]{0.07\textwidth}
    $\fnum=40$
    \end{minipage}
    \begin{minipage}[c]{0.32\textwidth}
    \includegraphics[width=\textwidth]{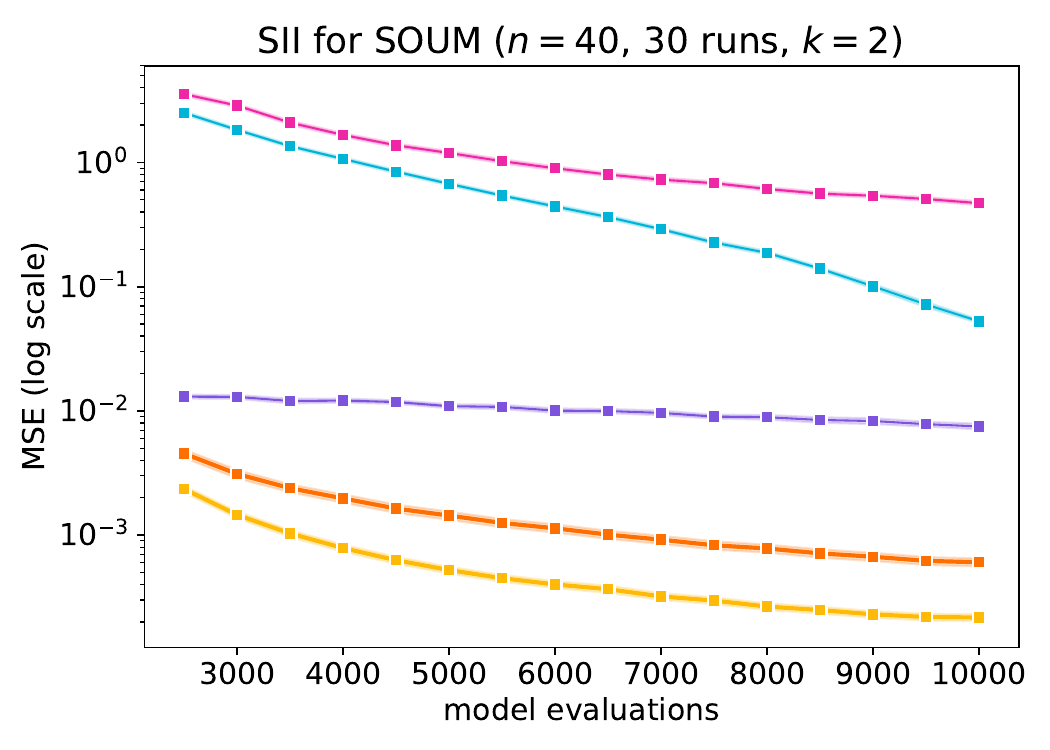}
    \end{minipage}
    \begin{minipage}[c]{0.32\textwidth}
    \includegraphics[width=\textwidth]{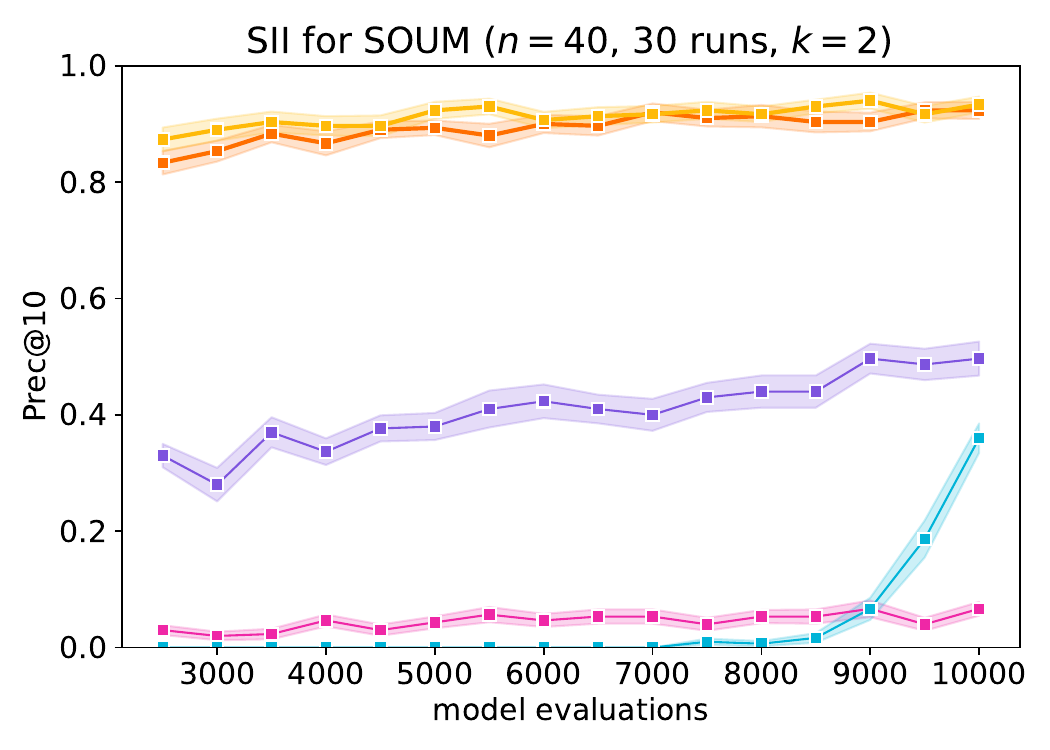}
    \end{minipage}
    \caption{Approximation quality in terms of MSE (left) and Prec@10 (right) for the SOUMs with $\fnum=20$ (top) and $\fnum=40$ (bottom).}
    \label{fig_appendix_soum}
\end{figure}

\begin{figure}[htb]
    \centering
    \includegraphics[width=0.75\textwidth]{figures/legend_horizontal_all_orders.pdf}
    \\
    \begin{minipage}[c]{0.13\textwidth}
    {\raggedright
    \textbf{MSE and\\Prec@10 for\\all orders:}
    }
    \end{minipage}
    \hfill
    \begin{minipage}[c]{0.42\textwidth}
    \includegraphics[width=\textwidth]{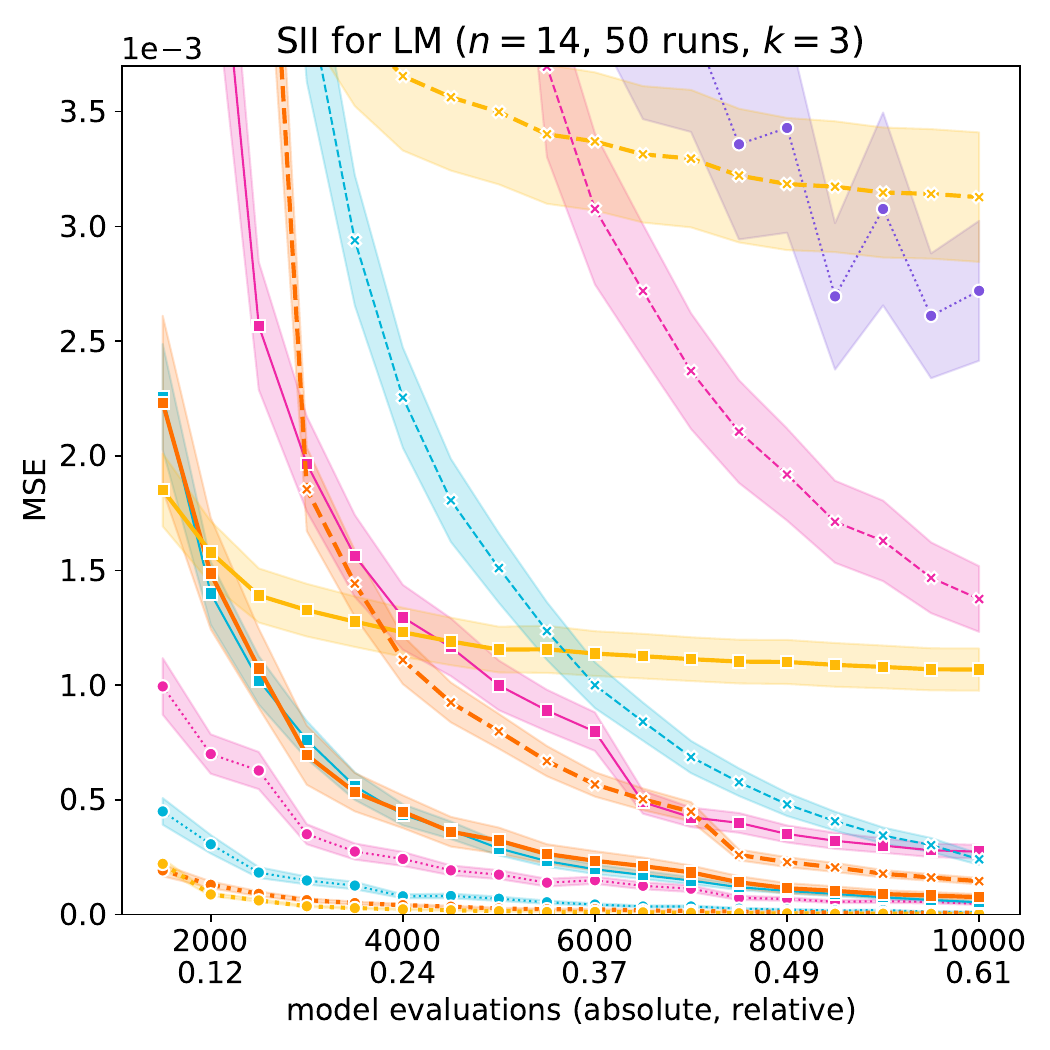}
    \end{minipage}
    \begin{minipage}[c]{0.42\textwidth}
    \includegraphics[width=\textwidth]{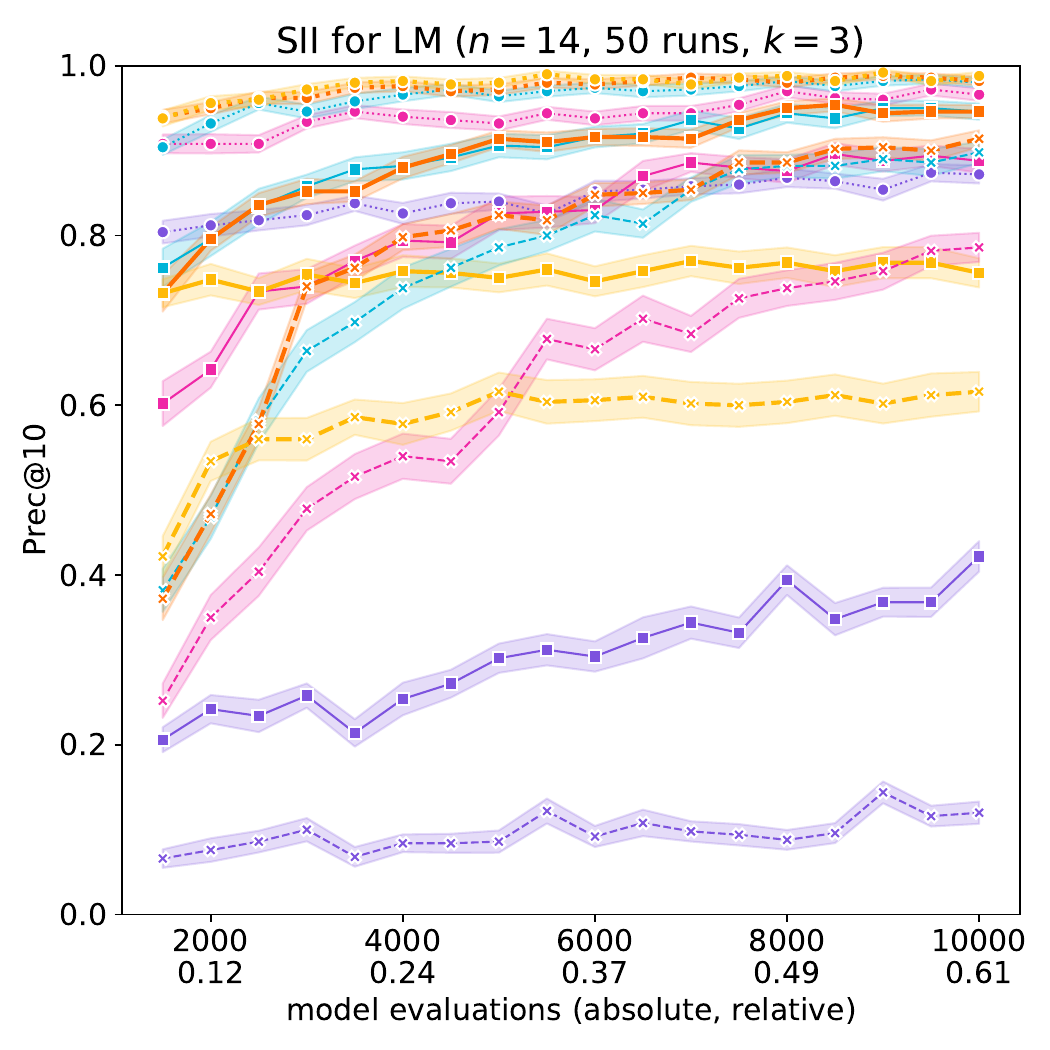}
    \end{minipage}
    \\
    \begin{minipage}[c]{0.08\textwidth}
    \vspace{-2em}
    \textbf{MSE:}\\[13em]
    \textbf{Prec@10:}
    \end{minipage}
    \hfill
    \begin{minipage}[c]{0.29\textwidth}
    \centering
    \includegraphics[width=\textwidth]{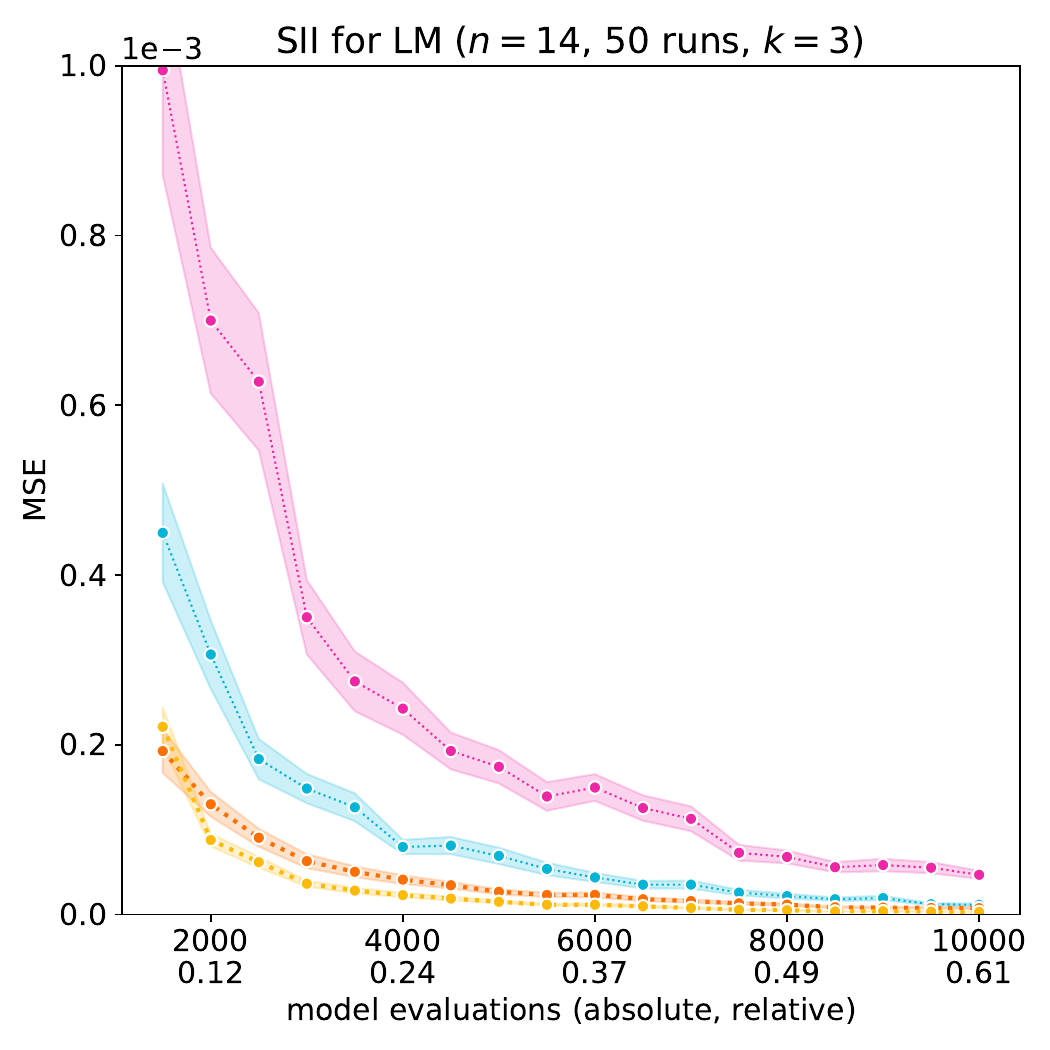}\\
    \includegraphics[width=\textwidth]{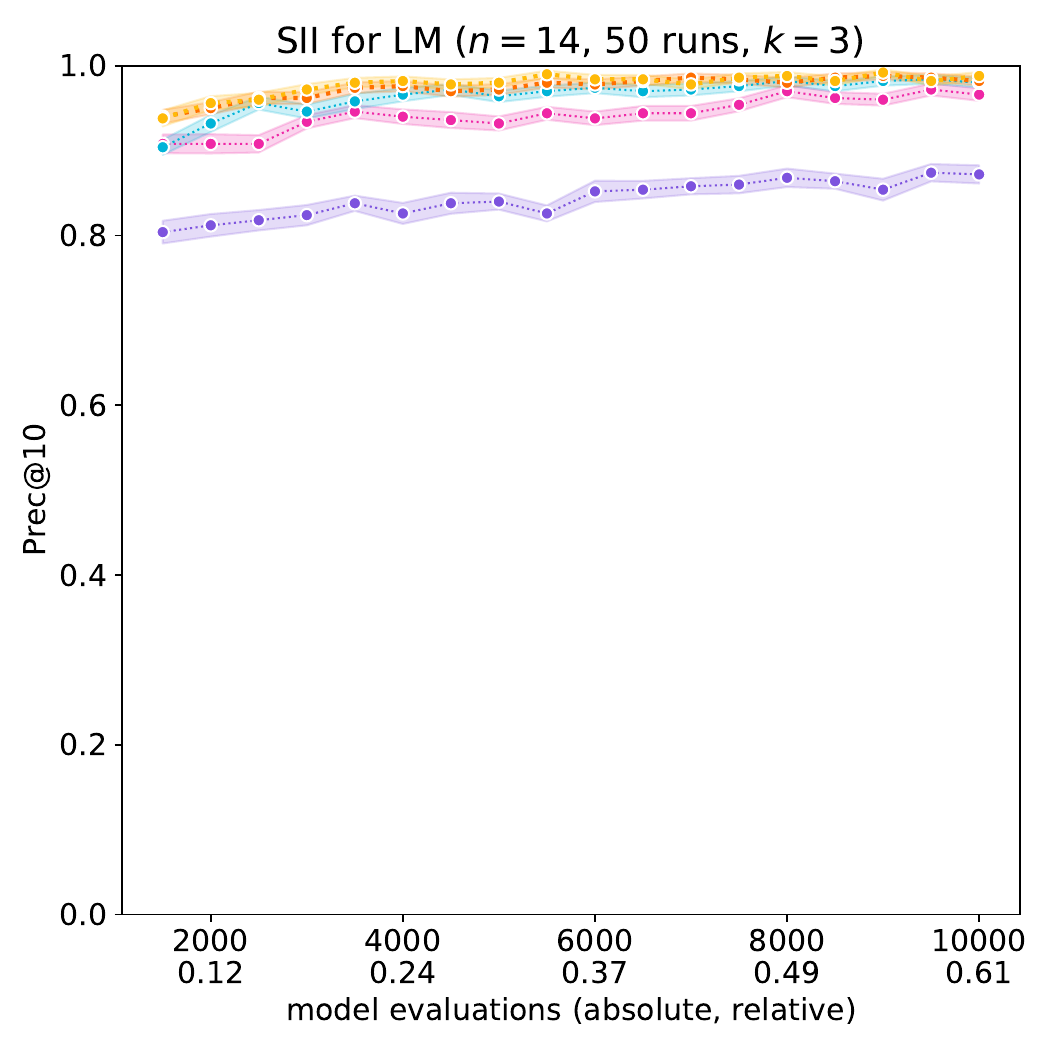}\\$l = 1$
    \end{minipage}
    \begin{minipage}[c]{0.29\textwidth}
    \centering
    \includegraphics[width=\textwidth]{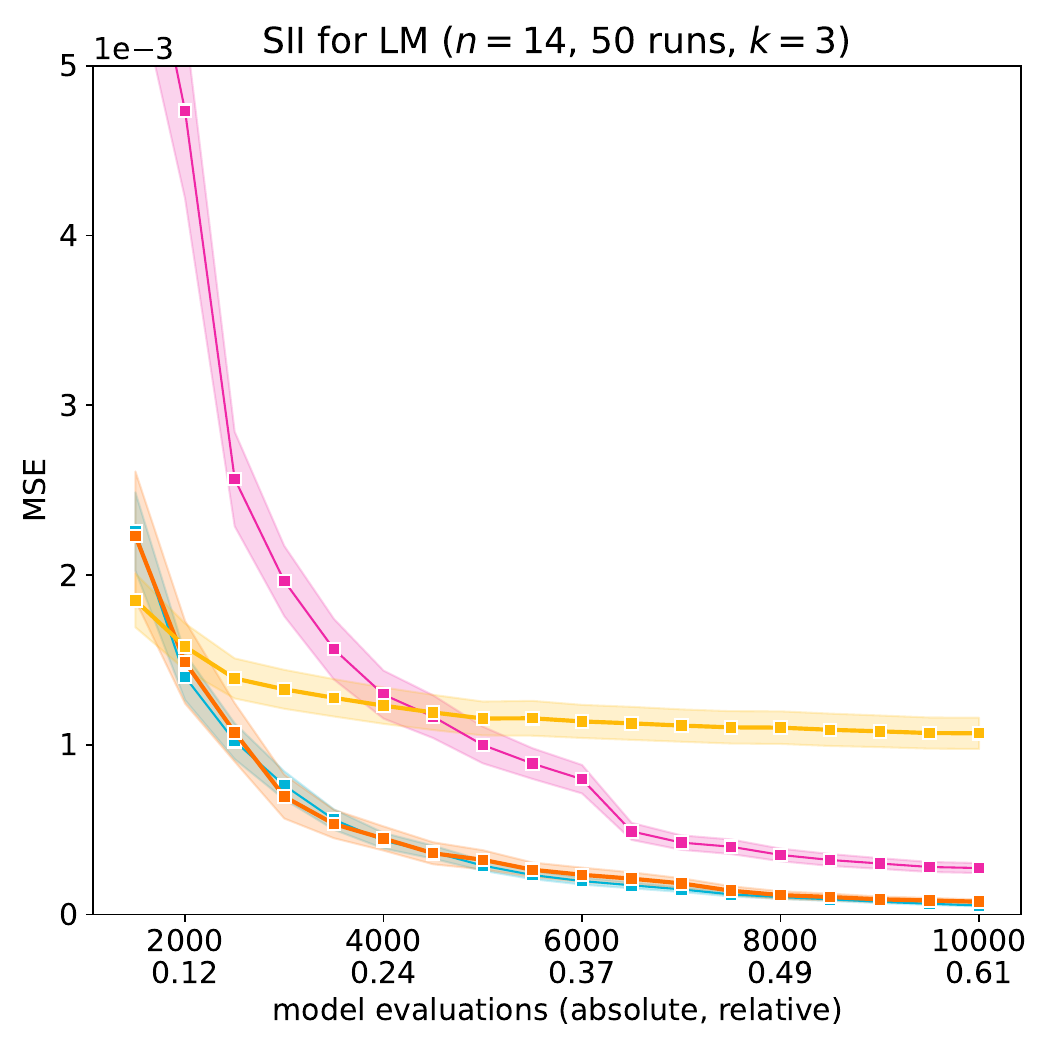}\\
    \includegraphics[width=\textwidth]{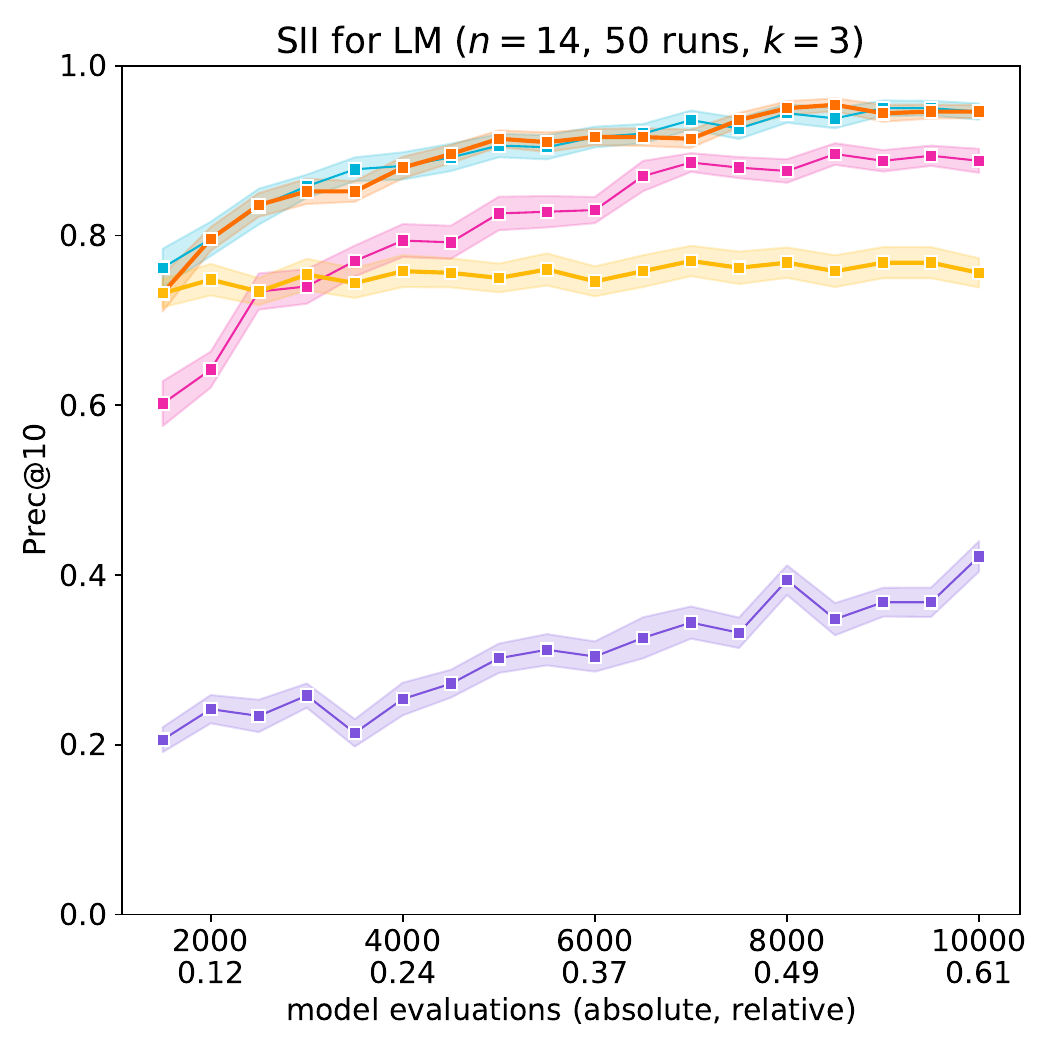}\\$l = 2$
    \end{minipage}
    \begin{minipage}[c]{0.29\textwidth}
    \centering
    \includegraphics[width=\textwidth]{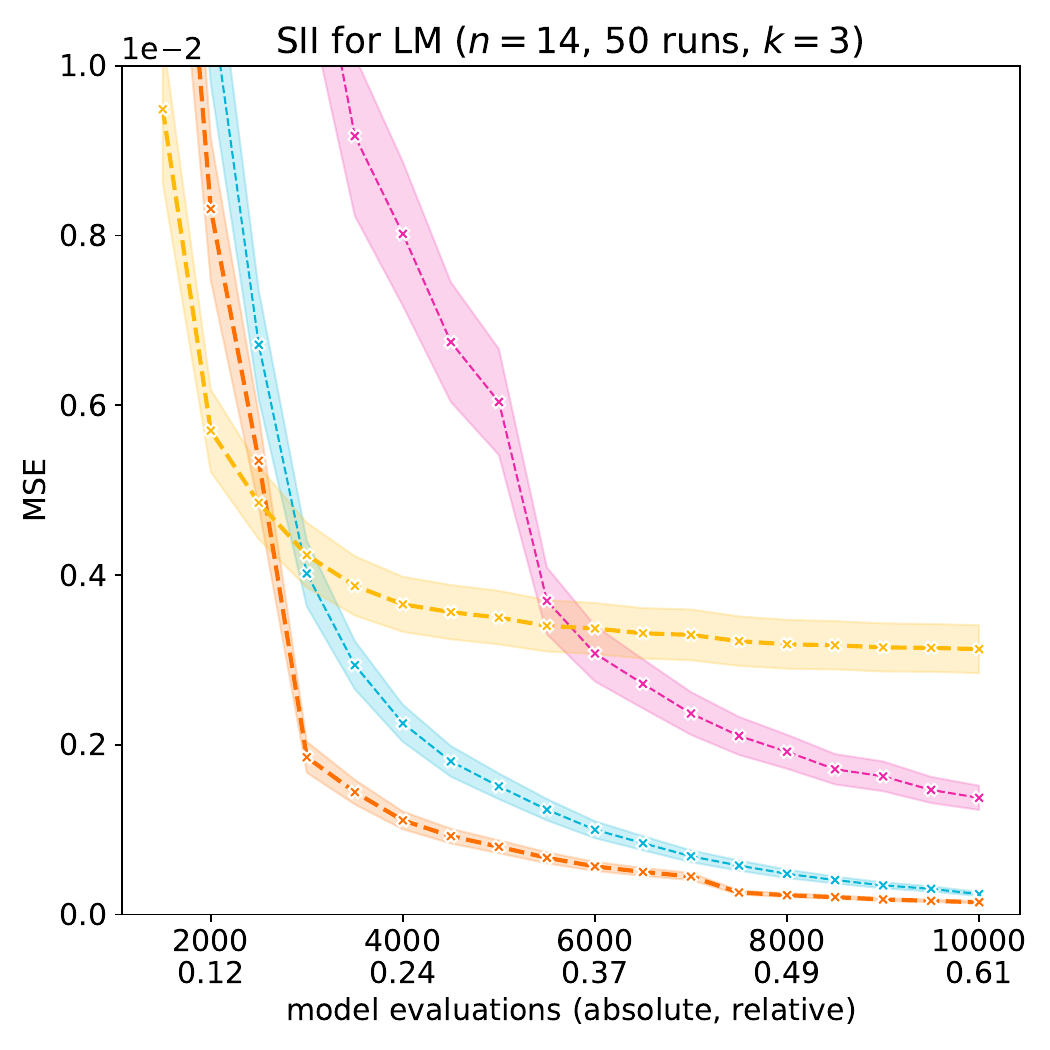}\\
    \includegraphics[width=\textwidth]{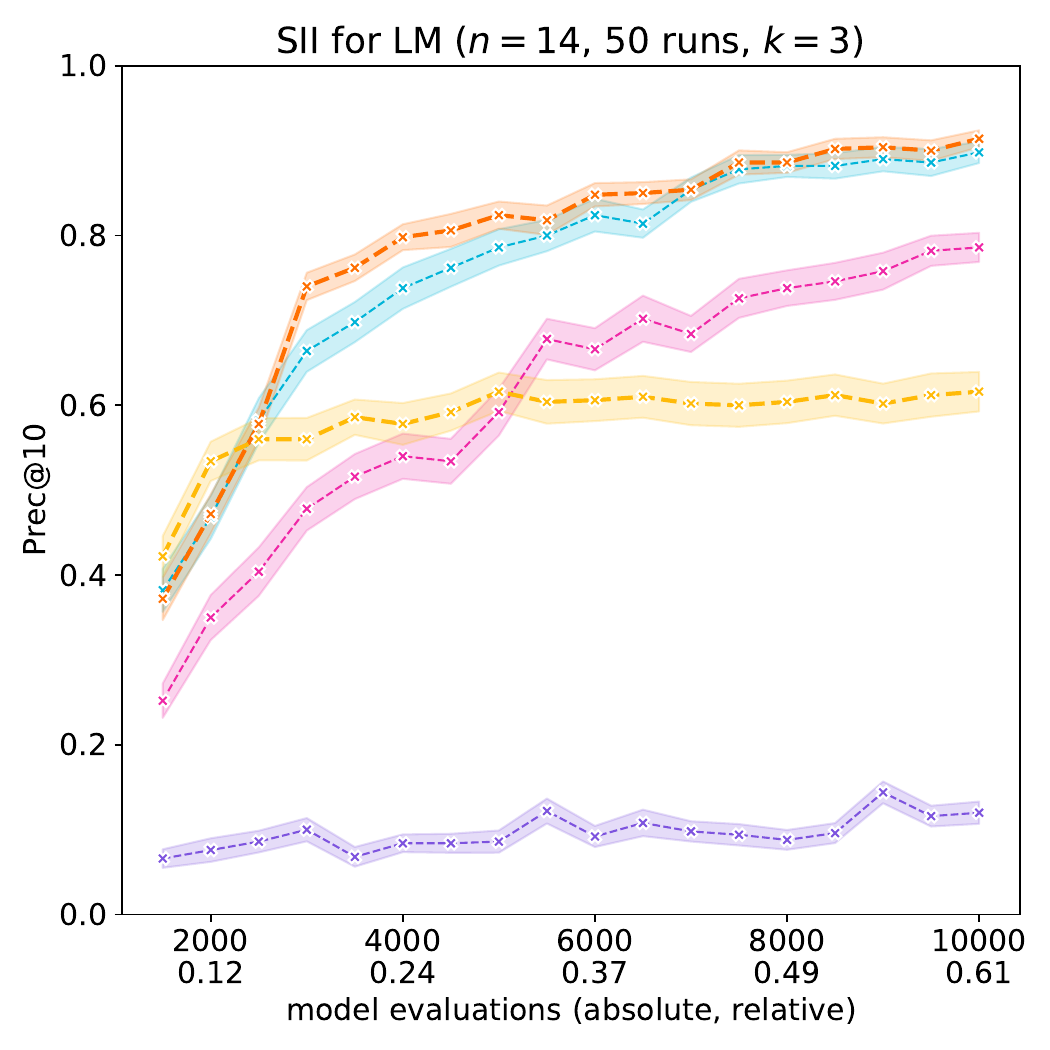}\\$l = 3$
    \end{minipage}
    \caption{Approximation quality of KernelSHAP-IQ compared to baseline techniques on the LM.}
    \label{fig_appendix_lm}
\end{figure}

\end{document}